\documentclass{article}

\usepackage[preprint]{neurips_2026}

\usepackage[utf8]{inputenc} 
\usepackage[T1]{fontenc}    
\usepackage{hyperref}       
\usepackage{url}            
\usepackage{booktabs}       
\usepackage{amsfonts}       
\usepackage{nicefrac}       
\usepackage{microtype}      
\usepackage{xcolor}         

\usepackage{graphicx} 
\usepackage{amsmath,amssymb,amsthm,amsfonts, algorithmicx, algpseudocode, algorithm}
\usepackage{geometry} 
\usepackage{bm} 
\usepackage{mathrsfs}
\usepackage{textcomp}
\usepackage{tcolorbox}
\usepackage{float,xcolor}
\usepackage{indentfirst,color}
\usepackage{fancyhdr}
\usepackage{epstopdf}
\usepackage{hyperref}
\usepackage{amsopn}
\usepackage{amstext}
\usepackage{amscd}
\usepackage{latexsym} 
\usepackage{natbib}
\usepackage{subcaption}
\usepackage{tikz}
\usepackage{enumitem}

\usetikzlibrary{arrows.meta, positioning, calc}

\newtheorem{thm}{Theorem}[section]

\newtheorem{prop}{Proposition}[section]

\theoremstyle{definition}

\newtheorem{lemma}{Lemma}[section]

\newtheorem{asp}{Assumption}

\title{Transformers Provably Implement In-Context Reinforcement Learning with Policy Improvement}

\author{Haodong Liang, Lifeng Lai\\
Department of Electrical and Computer Engineering\\
University of California, Davis\\
\texttt{\{hdliang,lflai\}@ucdavis.edu} 
}

\begin{document}

\maketitle

\begin{abstract}
    We investigate the ability of transformers to perform in-context reinforcement learning (ICRL), where a model must infer and execute learning algorithms from trajectory data without parameter updates. We show that a linear self-attention transformer block can provably implement policy-improvement methods, including semi-gradient SARSA and actor-critic, via explicit parameter constructions. Beyond existence, we design a teacher-mimicking training procedure, analyze its gradient-flow dynamics, and establish the first convergence guarantee in the ICRL literature: under suitable richness conditions on the training MDP distribution, gradient flow converges locally and exponentially to an optimal parameter manifold corresponding to the desired RL update. Empirically, training transformers on randomly generated tabular MDPs confirms these predictions: the learned models recover the parameter structure of our explicit constructions and, when deployed on unseen MDPs, deliver strong in-context control performance. Together, these results illuminate how transformer architectures internalize and execute classical reinforcement learning algorithms in context, bridging mechanistic understanding and training dynamics in ICRL.\footnote{Code is available at \url{https://github.com/harris-liang/Transformer-in-context-RL}.}
\end{abstract}

\section{Introduction}
In-context learning (ICL) refers to a model's ability to learn and perform a new task by conditioning on examples provided in its input, without any parameter updates \citep{brown2020language,garg2023transformers,dong2024surveyincontextlearning}. The transformer architecture \citep{vaswani} has empirically demonstrated remarkable in-context learning capabilities across a wide range of tasks, among which Large Language Models (LLMs) are the most prominent example \citep{brown2020language,chowdhery2023palm,achiam2023gpt4}. 

This phenomenon has motivated substantial interest in understanding the mechanisms that enable transformers to learn in-context. A growing body of theoretical work has sought to understand this by asking: \emph{which algorithms can a transformer execute in context, and why can it implement them?}

For supervised learning problems, this question now has a fairly complete answer: trained transformers can be shown to implement variants of gradient descent and closed-form least squares on the in-context examples, with explicit constructions and, in special cases, end-to-end training-dynamics guarantees \citep{garg2023transformers,vonoswald2023transformers,akyurek2023what,ahn2023transformers,mahankali2024one,bai2023transformers,zhang2024trained,liang2025transformers}.

Beyond the supervised learning scenario, recent works have begun to extend this lens to in-context reinforcement learning (ICRL)\citep{laskin2023algorithm,lin2024transformers,moeini2025surveyincontextreinforcementlearning}, where an agent is given a trajectory in its context and must internally execute a policy evaluation or policy improvement step — core operations in the reinforcement learning framework \citep{sutton2018reinforcement}. A central question in this line of work is:
\begin{center}
\emph{Can a transformer block implement -- and be trained to implement \\ standard RL algorithms such as temporal-difference (TD) methods, SARSA, or actor-critic?} 
\end{center}
While these works, including \citet{laskin2023algorithm}, \citet{lee2023supervised}, \citet{brooks2023large}, and \citet{raparthy2024generalization}, primarily focus on empirical experiments and demonstrate strong performance in various ICRL settings, they lack theoretical guarantees on the underlying mechanisms. Toward a theoretical answer, \citet{TCLTD} take a first step by showing that a linear-attention block admits a closed-form structure that implements policy evaluation via TD learning. Two questions, however, remain open. First, their analysis is restricted to policy evaluation, leaving open whether a transformer can also implement \emph{policy-improvement} algorithms -- a component that is no less central to RL. Second, existing analyses establish only the existence of an \emph{invariant} set: once the parameters lie in the set, the expected update keeps them there. They do not show whether the training actually drives the parameters into this set in the first place, nor identify the conditions under which it does.

In this work, we close both gaps: we move from policy evaluation to \emph{policy improvement} by showing that a single linear-attention block can implement classical on-policy control updates -- SARSA and actor-critic -- and we move from invariance to \emph{convergence} by identifying sufficient conditions under which gradient training provably drives the transformer parameters into the theoretical manifold that exactly realizes the target update.

Our existence theorem is similar in form to \citet{TCLTD} in that both exhibit a closed-form weight configuration $(\mathbf{P}^\star,\mathbf{V}^\star)$ for a single linear-attention block, but there are substantial differences. In \citet{TCLTD}, their transformer outputs only the value estimate at the query sample, so the construction is tailored to a fixed policy evaluation problem and cannot be reused to drive further interaction. Our construction instead directly outputs the updated parameters themselves -- the critic parameter $\mathbf{w}$ for SARSA, and both the critic $\mathbf{w}$ and the actor $\boldsymbol{\lambda}$ for actor-critic -- which can then be utilized to update policy, sample new trajectories, and close the in-context policy-improvement loop.

On the training-dynamics side, the optimum is no longer a single point but a continuous manifold of equivalent constructions, so global strong-convexity arguments do not apply directly; we instead establish a \emph{local} Polyak--\L{}ojasiewicz inequality on a tubular neighborhood of the manifold which, together with explicit richness conditions on the training-MDP distribution, yields exponential convergence of the gradient flow to a single point on the manifold.

Our contributions are summarized as follows:
\begin{itemize}
    \item \textbf{Existence Proof.} We construct explicit parameters $\boldsymbol{\theta}^\star=(\mathbf{P}^\star,\mathbf{V}^\star)$ for a single-head linear self-attention block such that $\textsf{TF}_{\boldsymbol{\theta}^\star}$ exactly realizes the batch SARSA semi-gradient update (Theorem~\ref{thm: semi-sarsa tf}) and the batch actor-critic update (Theorem~\ref{thm: policy gradient tf}).
    \item \textbf{Local convergence.} Under explicit richness assumptions on the training-MDP distribution (Assumption~\ref{asp:training-MDP assumptions}), we prove that by mimicking the teacher updates, the gradient flow on the population loss converges locally to a single point on the manifold (Theorem~\ref{thm:local-convergence-manifold}), with an analogous guarantee established for the actor-critic setting (Theorem~\ref{thm: ac local convergence} in Appendix~\ref{app: ac convergence}). To our knowledge, these are the first convergence guarantees in the ICRL literature.
    \item \textbf{Empirical validation.} We design ICRL training algorithms (Algorithm \ref{alg: training transformer}, Algorithm \ref{alg: training transformer - actor critic}) for both the SARSA and actor-critic settings, and show that the resulting trained transformers recover the theoretical block structure of the attention parameters $(\mathbf{P}^\star,\mathbf{V}^\star)$ (Section~\ref{subsec: training}). Furthermore, when deployed on unseen MDPs, the trained transformers closely track their analytical teachers and drive policies toward oracle-level returns (Section~\ref{subsec: evaluation}).
\end{itemize}
\textbf{Additional related work:} We provide a detailed discussion of more related work in Appendix~\ref{app: related works}.

\section{Preliminaries}\label{sec: background}
\subsection{Transformer architecture}
Following prior mechanistic analyses of in-context learning~\citep{vonoswald2023transformers,mahankali2024one,TCLTD}, we adopt linear attention for theoretical tractability. We consider a one-layer, single-head transformer block with parameters $\boldsymbol{\theta}=\{(\mathbf{P},\mathbf{V})\}\subseteq\mathbb{R}^{D\times D}$. Given input $\mathbf{H}\in\mathbb{R}^{D\times n}$, the output of the attention layer is given by
\begin{align}\label{def:hLin}
    \mathbf{H}_{\text{out}} := \mathbf{H} + \frac{1}{n}(\mathbf{V}\mathbf{H})(\mathbf{H}^\top\mathbf{P}\mathbf{H}) \in\mathbb{R}^{D\times n},
\end{align}
where $\mathbf{P}:=\mathbf{Q}^\top\mathbf{K}$ absorbs the query and key projections. The transformer output is given by $\textsf{TF}_{\boldsymbol{\theta}}(\mathbf{H}) := f_{\text{read}}(\mathbf{H}_{\text{out}})$ for a fixed readout $f_{\text{read}}:\mathbb{R}^{D\times n}\to\mathbb{R}^{D_{\text{out}}}$.

\subsection{Markov decision process}
A Markov decision process (MDP) is a tuple $(\mathcal{S},\mathcal{A},P,R,\gamma)$ with state space $\mathcal{S}$, action space $\mathcal{A}$, transition kernel $P:\mathcal{S}\times\mathcal{A}\to\mathcal{P}(\mathcal{S})$, reward $R:\mathcal{S}\times\mathcal{A}\to\mathbb{R}$, and discount $\gamma\in[0,1)$, where $\mathcal{P}(\mathcal{X})$ denotes the set of probability distributions over $\mathcal{X}$. The goal is to find a policy $\pi:\mathcal{S}\to\mathcal{P}(\mathcal{A})$ maximizing the (expected) return $G_t:=\sum_{k=0}^\infty\gamma^k r_{t+1+k}$. The state- and action-value functions under $\pi$ are
\begin{align*}
    v_\pi(s) := \mathbb{E}_\pi[G_t\mid s_t=s],\qquad q_\pi(s,a) := \mathbb{E}_\pi[G_t\mid s_t=s,a_t=a].
\end{align*}
\subsection{SARSA}
SARSA is an on-policy value-based algorithm. Given a feature map $\boldsymbol{\phi}:\mathcal{S}\times\mathcal{A}\to\mathbb{R}^d$ and parameter $\mathbf{w}\in\mathbb{R}^d$, we consider linear function approximation $q_{\mathbf{w}}(s,a)=\mathbf{w}^\top\boldsymbol{\phi}(s,a)$, and the batch semi-gradient SARSA update over an $n$-step trajectory is
\begin{align}\label{eq:semi-sarsa linear batch}
    \mathbf{w}_{\text{SARSA}} = \mathbf{w}+\frac{\alpha}{n}\sum_{i=0}^{n-1}\bigl[r_{i+1}+\gamma\mathbf{w}^{\top}\boldsymbol{\phi}(s_{i+1},a_{i+1})-\mathbf{w}^{\top}\boldsymbol{\phi}(s_i,a_i)\bigr]\boldsymbol{\phi}(s_i,a_i),
\end{align}
where $\alpha>0$ is the step size and the bracketed quantity is the temporal-difference (TD) error.

\subsection{Actor-Critic}
Actor-Critic combines a parametric policy (the actor) and a parametric value function (the critic). Given a policy feature map $\boldsymbol{\phi}_\pi:\mathcal{S}\times\mathcal{A}\to\mathbb{R}^m$, a value feature map $\boldsymbol{\phi}_V:\mathcal{S}\to\mathbb{R}^d$, and parameters $\boldsymbol{\lambda}\in\mathbb{R}^m$, $\mathbf{w}\in\mathbb{R}^d$, 
with softmax policy $\pi_{\boldsymbol{\lambda}}(a\mid s)\propto\exp(\boldsymbol{\lambda}^\top\boldsymbol{\phi}_\pi(s,a))$ and linear critic $v_{\mathbf{w}}(s)=\mathbf{w}^\top\boldsymbol{\phi}_V(s)$, the batch update over an $n$-step trajectory is
\begin{equation}\label{eq: actor-critic batch}
\begin{aligned}
    \mathbf{w}_{\text{AC}} &= \mathbf{w}+\frac{\beta}{n}\sum_{i=0}^{n-1}\bigl[r_{i+1}+\gamma\mathbf{w}^{\top}\boldsymbol{\phi}_V(s_{i+1})-\mathbf{w}^{\top}\boldsymbol{\phi}_V(s_i)\bigr]\boldsymbol{\phi}_V(s_i),\\
    \boldsymbol{\lambda}_{\text{AC}} &= \boldsymbol{\lambda}+\frac{\alpha}{n}\sum_{i=0}^{n-1}\gamma^i\bigl[r_{i+1}+\gamma\mathbf{w}^{\top}\boldsymbol{\phi}_V(s_{i+1})-\mathbf{w}^{\top}\boldsymbol{\phi}_V(s_i)\bigr]\mathbf{g}_{\boldsymbol{\lambda}}(s_i,a_i),
\end{aligned}
\end{equation}
where $\alpha,\beta>0$ are step sizes, and $\mathbf{g}_{\boldsymbol{\lambda}}(s,a):=\boldsymbol{\phi}_\pi(s,a)-\sum_b\boldsymbol{\phi}_\pi(s,b)\pi_{\boldsymbol{\lambda}}(b\mid s)$ is the score function.
\section{Main Results}\label{sec: main results}
We demonstrate that a single-head transformer block can learn to perform in-context policy improvement using classical on-policy RL algorithms. Figure \ref{fig:icrl_pipeline} illustrates the ICRL pipeline. Given a formulated input prompt $\mathbf{H}(z)$ constructed from the training sample $z$, we show that there exists $\boldsymbol{\theta}^\star$ such that the transformer block $\textsf{TF}_{\boldsymbol{\theta}^\star}$ directly outputs the updated parameter $\mathbf{w}$, which is then used to update the value function or policy. In Section \ref{subsec: ic-sarsa} we establish the existence of such implementations of SARSA, and in Section \ref{subsec: training-dynamics} we analyze the training dynamics of the proposed training procedure and identify sufficient conditions for convergence to the optimum parameter manifold. We then extend these results to actor-critic in Section \ref{subsec: ic-actor-critic}.

\begin{figure}[t]
\centering
\begin{tikzpicture}[scale=0.65, transform shape,
    x=1cm,y=1cm,
    >=Stealth,
    line width=0.9pt,
    font=\normalsize,
    block/.style={draw=black!80, rounded corners=2pt, minimum height=1.7cm, minimum width=2.1cm, align=center},
    bigblock/.style={draw=blue!70!black, rounded corners=4pt},
    blackarrow/.style={->, draw=black!80, line width=1pt},
    bluearrow/.style={->, draw=blue!70!black, line width=1pt},
    purplearrow/.style={->, draw=violet!80!black, line width=1pt}
]

\node[bigblock, minimum width=7.8cm, minimum height=3.1cm] (tf) at (-0.9,0) {};
\node[blue!70!black, font=\large] at (-0.8,2.0) {$\mathrm{TF}_{\boldsymbol{\theta}}$};

\node[block, minimum width=2.1cm, minimum height=1.7cm] (attn) at (-3.0,0) {$\textsf{Attn}(\mathbf{P},\mathbf{V})$};
\node[block, minimum width=2.1cm, minimum height=1.7cm] (mid) at (1.1,0) {$f_{\text{read}}$};
\node[block, minimum width=2.7cm, minimum height=1.3cm] (env) at (-0.8,-3.0) {Agent};
\node[black!80] at (-4.5,0.3) {$\mathbf{H}$};
\node[black!80] at (-1.0,0.3) {$\mathbf{H}_{\text{out}}$};
\node[black!80] at (2.5,0.3) {$\mathbf{w}$};

\draw[blackarrow]
(env.west) -- ++(-4,0)      
           -- ++(0,3)       
           -- (attn.west);\draw[blackarrow] (attn.east) -- (mid.west);
\draw[blackarrow]
(mid.east) -- ++(2,0)      
           -- ++(0,-3)       
           -- (env.east);
\end{tikzpicture}
\caption{A simplified pipeline of transformer's ICRL implementation. The agent samples a trajectory under the current policy, which is fed into the transformer block. The transformer outputs an updated parameter $\mathbf{w}$, which is then used to update the value function and policy.}\label{fig:icrl_pipeline}
\end{figure}

\subsection{Transformers can implement SARSA}\label{subsec: ic-sarsa}
Let $z=(\mathcal{M},\boldsymbol{\phi},\mathbf{w},\tau)$ denote a training sample, where $\mathcal{M}$ is the training MDP task, $\boldsymbol{\phi}:\mathcal{S}\times\mathcal{A}\to\mathbb{R}^d$ is the feature map, $\mathbf{w}\in\mathbb{R}^d$ is the parameter of the behavior policy $\pi$, and $\tau=(s_0,a_0,r_1,\ldots,s_n,a_n)$ is the resulting $n$-step trajectory in $\mathcal{M}$. For notational convenience, we define the following quantities for $i=0,\ldots,n-1$:
\begin{gather*}
    \boldsymbol{\phi}_i:=\boldsymbol{\phi}(s_i,a_i)\in\mathbb{R}^d,\quad
    \boldsymbol{\phi}_i^+:=\boldsymbol{\phi}(s_{i+1},a_{i+1})\in\mathbb{R}^d,\quad
    \delta_i:=r_{i+1}+\gamma \mathbf{w}^\top \boldsymbol{\phi}_i^+-\mathbf{w}^\top \boldsymbol{\phi}_i\in\mathbb{R},\\
    \mathbf{x}_i:=\begin{bmatrix}
    \boldsymbol{\phi}_i\\
    \gamma \boldsymbol{\phi}_i^+\\
    r_{i+1}
    \end{bmatrix}\in\mathbb{R}^{2d+1},\quad \tilde{\mathbf{w}}:=\begin{bmatrix}1\\ \mathbf{w}\end{bmatrix}\in\mathbb{R}^{d+1}.
\end{gather*} 
We also define the following trajectory statistics:
\begin{gather*}
    \hat{\boldsymbol{\Sigma}}(z):=\frac{1}{n}\sum_{i=0}^{n-1}\mathbf{x}_i\mathbf{x}_i^\top,\quad
    \mathbf{R}(z):=\frac{1}{n}\sum_{i=0}^{n-1}\boldsymbol{\phi}_i\mathbf{x}_i^\top,\quad 
    \mathbf{b}(z):=\frac{1}{n}\sum_{i=0}^{n-1}\mathbf{x}_i\,\delta_i.
\end{gather*}
Consider the training prompt as follows:
\begin{align}\label{eq:input-matrix}
    \mathbf{H}(z)=\begin{bmatrix}
    \mathbf{x}_0 & \cdots & \mathbf{x}_{n-1} & \mathbf{0}\\
    \mathbf{0} & \cdots & \mathbf{0} & \tilde{\mathbf{w}}\\
    \end{bmatrix}\in\mathbb{R}^{D\times (n+1)},
\end{align}
where $D=3d+2$. We now establish the existence of a transformer block that implements the batch semi-gradient SARSA update \eqref{eq:semi-sarsa linear batch}.
\begin{thm}[Implement semi-gradient SARSA with a transformer block]\label{thm: semi-sarsa tf}
    Let the output projection function be $f_{\text{read}}(\cdot)=\cdot_{-d:,-1}$, which extracts the last $d$ elements of the last column. Given the input prompt \eqref{eq:input-matrix}, for any fixed $\alpha>0$, there exists $\boldsymbol{\theta}^\star=\{(c\mathbf{P}^\star,c^{-1}\mathbf{V}^\star)\}\subseteq\mathbb{R}^{D\times D}$ with $c\neq 0$ and
    \begin{align}\label{eq:QKV}
        \mathbf{P}^\star&=\begin{bmatrix}
            \mathbf{*}_{(2d+1)\times(2d+1)} & \begin{bmatrix}
            \mathbf{0}_{d\times 1} & -\mathbf{I}_{d}\\
            \mathbf{0}_{d\times 1} & \mathbf{I}_{d}\\
            1 & \mathbf{0}_{1\times d}
            \end{bmatrix}\\
            \mathbf{*}_{(d+1)\times(2d+1)} & \mathbf{0}_{(d+1)\times (d+1)}
        \end{bmatrix},\quad
         \mathbf{V}^\star=\begin{bmatrix}
            \mathbf{*}_{(2d+1)\times(2d+1)} & \mathbf{*}_{(2d+1)\times (d+1)}\\
            \begin{bmatrix}
            \mathbf{*}_{1\times d} & \mathbf{*}_{1\times d} & *\\
            \alpha\mathbf{I}_d & \mathbf{0}_{d\times d} & \mathbf{0}_{d\times 1}
            \end{bmatrix} & \begin{bmatrix}
            \mathbf{*}_{1\times (d+1)}\\
            \mathbf{0}_{d\times (d+1)}
            \end{bmatrix}
        \end{bmatrix},
    \end{align}
    where the asterisks denote arbitrary entries that do not affect the transformer output. With these parameters, $\textsf{TF}_{\boldsymbol{\theta}^\star}(\mathbf{H}(z))$ directly implements the batch semi-gradient SARSA update \eqref{eq:semi-sarsa linear batch}.
\end{thm}
    The proof is deferred to Appendix \ref{proof of thm: semi-sarsa tf}. Theorem \ref{thm: semi-sarsa tf} guarantees the existence of a parameter configuration $\boldsymbol{\theta}^\star$ that realizes the semi-gradient SARSA update, but does not prescribe how to obtain it. To bridge this gap, we need to develop a training algorithm.

\subsection{Algorithm and training dynamics analysis}\label{subsec: training-dynamics}
We now design an ICRL algorithm to train the transformer to implement SARSA, as outlined in Algorithm \ref{alg: training transformer}. This procedure is inspired by \cite{laskin2023algorithm}.  At each iteration, trajectory data are sampled from the MDP, the corresponding SARSA update is computed as a target, and the transformer parameters are updated to minimize the discrepancy between the transformer's prediction and the target update. The empirical loss at each step approximates the population loss defined below:
\begin{align}\label{eq:loss}
    \mathcal{L}(\boldsymbol{\theta})&=\mathbb{E}_z\!\left[\frac{1}{2}\bigl\|\textsf{TF}_{\boldsymbol{\theta}}\bigl(\mathbf{H}(z)\bigr)-\mathbf{w}_{\text{SARSA}}(z)\bigr\|^2\right],
\end{align}
where $\mathbf{w}_{\text{SARSA}}(z)$ is the target parameter given by the batch semi-gradient SARSA update~\eqref{eq:semi-sarsa linear batch}. We now analyze the training dynamics under this loss objective.

We consider gradient flow, which captures the behavior of gradient descent with infinitesimal step size, and is described by the following ordinary differential equation (ODE):
\begin{align}\label{eq: gradient flow}
    \frac{d\boldsymbol{\theta}(t)}{dt}=-\nabla \mathcal{L}(\boldsymbol{\theta}(t)).
\end{align}
We first identify the invariant parameter groups during training.

\begin{algorithm}[t]
    \begin{algorithmic}[1]
        \caption{IC-SARSA Transformer Training}\label{alg: training transformer}
        \State \textbf{Input:} Transformer parameter $\boldsymbol{\theta}^{(0)}=(\boldsymbol{P}^{(0)}, \boldsymbol{V}^{(0)})$, trajectory window length $n$, number of frames per MDP $T$, number of MDPs $K$, $\epsilon$-greedy parameter $\epsilon$, learning rate $\eta$, and SARSA step size $\alpha$
        \For{$k=0,\ldots,K-1$}
        \State Sample a MDP $\mathcal{M}^{(k)}=(\mathcal{S},\mathcal{A},P^{(k)},R^{(k)}, \gamma^{(k)})$, and generate corresponding feature map $\phi^{(k)}:\mathcal{S}\times\mathcal{A}\to\mathbb{R}^d$. Initialize parameter $\mathbf{w}^{(0)}$ and state $s_0^{(0)}$ randomly
        \For{$t=0,\ldots,T-1$}
        \State Estimate state-action value function $Q_{\mathbf{w}^{(t)}}^{(k)}(s,a)=\mathbf{w}^{(t)\top}\boldsymbol{\phi}^{(k)}(s,a),\, \forall (s,a)\in\mathcal{S}\times\mathcal{A}$
        \State Sample a n-step trajectory $(s_0^{(t)},a_0^{(t)},r_1^{(t)},s_1^{(t)},a_1^{(t)},\ldots,r_{n}^{(t)},s_{n}^{(t)},a_{n}^{(t)})$ from the MDP $\mathcal{M}^{(k)}$, according to $\epsilon$-greedy policy derived from $Q_{\mathbf{w}^{(t)}}^{(k)}$
        \State Formulate $\mathbf{H}^{(t)}$ as in \eqref{eq:input-matrix} with the sampled trajectory and current parameter $\mathbf{w}^{(t)}$
        \State Compute $\mathbf{w}^{(t+1)}=\text{TF}_{\boldsymbol{\theta}^{(t)}}(\mathbf{H}^{(t)}):=\mathbf{w}^{(t)}+\frac{1}{n}\left[\mathbf{V}^{(t)}\mathbf{H}^{(t)}\mathbf{H}^{(t)\top}\mathbf{P}^{(t)}\mathbf{H}^{(t)}\right]_{-d:,-1}$
        \State Compute $\mathbf{w}_{\text{SARSA}}^{(t+1)}$ according to \eqref{eq:semi-sarsa linear batch}
        \State Compute loss $\hat{L}^{(t)}=\frac{1}{2}\|\mathbf{w}^{(t+1)}-\mathbf{w}_{\text{SARSA}}^{(t+1)}\|^2$
        \State Update $\boldsymbol{\theta}^{(t+1)}\leftarrow \boldsymbol{\theta}^{(t)}-\eta\nabla_{\boldsymbol{\theta}}\hat{L}^{(t)}$
        \State Update $s_0^{(t+1)}\leftarrow s_n^{(t)}$, $\mathbf{w}^{(t+1)}\leftarrow \mathbf{w}_{\text{SARSA}}^{(t+1)}$
        \EndFor
        \State Update $\boldsymbol{\theta}^{(0)}\leftarrow \boldsymbol{\theta}^{(T)}$
        \EndFor
    \State \textbf{Return} Trained transformer parameter $\boldsymbol{\theta}$
    \end{algorithmic}
\end{algorithm}
\begin{prop}[Invariant parameter groups]\label{prop:invariant-set}
Partition the attention parameters into four sub-blocks:
\begin{align}\label{eq:block partition}
    \mathbf{P}=\begin{bmatrix}
    \mathbf{P}_{11} & \mathbf{P}_{12}\\
    \mathbf{P}_{21} & \mathbf{P}_{22}
    \end{bmatrix},\qquad
    \mathbf{V}=\begin{bmatrix}      
    \mathbf{V}_{11} & \mathbf{V}_{12}\\
    \mathbf{V}_{21} & \mathbf{V}_{22}
    \end{bmatrix},
\end{align}
where $\mathbf{P}_{11}, \mathbf{V}_{11}\in\mathbb{R}^{(2d+1)\times(2d+1)}$, $\mathbf{P}_{12}, \mathbf{V}_{12}\in\mathbb{R}^{(2d+1)\times(d+1)}$, $\mathbf{P}_{21}, \mathbf{V}_{21}\in\mathbb{R}^{(d+1)\times(2d+1)}$, and $\mathbf{P}_{22}, \mathbf{V}_{22}\in\mathbb{R}^{(d+1)\times(d+1)}$. Let $\bar{\mathbf{V}}_{21}$ and $\bar{\mathbf{V}}_{22}$ denote the last $d$ rows of $\mathbf{V}_{21}$ and $\mathbf{V}_{22}$, respectively.
Under gradient flow \eqref{eq: gradient flow}, the following parameters are invariant:
\begin{enumerate}[noitemsep]
    \item[(i)] The blocks $\mathbf{P}_{11},\mathbf{P}_{21},\mathbf{V}_{11},\mathbf{V}_{12}$ and the first rows of $\mathbf{V}_{21},\mathbf{V}_{22}$ do not affect the transformer output and remain unchanged throughout training.
    \item[(ii)] If $\mathbf{P}_{22}^{(0)}=\mathbf{0}$ and $\bar{\mathbf{V}}_{22}^{(0)}=\mathbf{0}$, then $\mathbf{P}_{22}^{(t)}=\mathbf{0}$ and $\bar{\mathbf{V}}_{22}^{(t)}=\mathbf{0}$ for all $t\geq 0$.
\end{enumerate}
\end{prop}
The proof is deferred to Appendix \ref{proof of prop:invariant-set}. Throughout the remainder of this section, we restrict our analysis to initializations satisfying $\mathbf{P}_{22}^{(0)}=\mathbf{0}$ and $\bar{\mathbf{V}}_{22}^{(0)}=\mathbf{0}$. By Proposition~\ref{prop:invariant-set}(ii), this constraint is preserved by gradient flow, and combined with~(i) it reduces the training dynamics on $(\mathbf{P},\mathbf{V})$ to dynamics on the effective parameter $\boldsymbol{\theta}_{\text{eff}}:=(\mathbf{P}_{12},\bar{\mathbf{V}}_{21})\in\boldsymbol{\Theta}_{\text{eff}}$, where $\boldsymbol{\Theta}_{\text{eff}}=\mathbb{R}^{(2d+1)\times(d+1)}\times\mathbb{R}^{d\times(2d+1)}$ is the effective parameter space. From \eqref{eq:QKV}, the optimum structures of these sub-blocks are
\begin{align*}
    \mathbf{P}_{12}^\star:=\begin{bmatrix}
        \mathbf{0}_{d\times 1} & -\mathbf{I}_{d}\\
        \mathbf{0}_{d\times 1} & \mathbf{I}_{d}\\
        1 & \mathbf{0}_{1\times d}
    \end{bmatrix}\in\mathbb{R}^{(2d+1)\times(d+1)},\qquad
    \bar{\mathbf{V}}_{21}^\star:=\begin{bmatrix}
        \alpha\mathbf{I}_d & \mathbf{0}_{d\times d} & \mathbf{0}_{d\times 1}
    \end{bmatrix}\in\mathbb{R}^{d\times(2d+1)},
\end{align*} 
subject to the scaling freedom $(\mathbf{P}_{12}^\star,\bar{\mathbf{V}}_{21}^\star)\mapsto(c\mathbf{P}_{12}^\star,c^{-1}\bar{\mathbf{V}}_{21}^\star)$ for any $c\neq 0$. Without loss of generality we restrict to $c>0$ (the sign-flipped branch is handled by symmetry), and to a compact sub-interval $I=[c_-,c_+]\subset(0,\infty)$ bounded away from the degenerate limits $c\to 0,\infty$. On this interval, we define the
\emph{optimal parameter manifold}
\begin{align}\label{def:opt-manifold}
    \mathfrak{M}_I:=\bigl\{\bar{\boldsymbol{\theta}}_{\text{eff}}(c):=(c\mathbf{P}_{12}^\star,\,c^{-1}\bar{\mathbf{V}}_{21}^\star)\,:\,c\in I\bigr\}
    \subset\boldsymbol{\Theta}_{\text{eff}},
\end{align}
and the normal space at $\bar{\boldsymbol{\theta}}_{\text{eff}}(c)$ as
\begin{align}\label{eq: normal hyperplane}
    N_c\mathfrak{M}_I:=\left\{(\mathbf{U},\mathbf{W})\in\boldsymbol{\Theta}_{\text{eff}}:\langle\mathbf{U},\mathbf{P}_{12}^\star\rangle_F - c^{-2}\langle\mathbf{W},\bar{\mathbf{V}}_{21}^\star\rangle_F=0\right\}.
\end{align}
We now state the richness assumptions on the training-MDP distribution that ensure local convergence to the optimal parameter manifold \eqref{def:opt-manifold}.
\begin{asp}[Richness of training MDPs]\label{asp:training-MDP assumptions}
Let $\Pi$ denote the class of behavior policies used to generate training trajectories. There exist positive constants $B_\phi,B_r,B_{\tilde{w}},\kappa_{\tilde{w}},\kappa_R,\kappa_q$ and $\rho\in[0,1)$ such that, for every $\pi\in\Pi$:
\begin{enumerate}[noitemsep]
    \item[(1)] \textbf{Boundedness.} Almost surely, $\|\boldsymbol{\phi}(s,a)\|\leq B_\phi$, $|r|\leq B_r$, and $\|\tilde{\mathbf{w}}\|\leq B_{\tilde{w}}$.
    \item[(2)] \textbf{Parameter excitation.} $\mathbb{E}[\tilde{\mathbf{w}}\tilde{\mathbf{w}}^\top]\succeq\kappa_{\tilde{w}}\mathbf{I}_{d+1}$.
    \item[(3)] \textbf{Bellman-regressor excitation.} $\mathbb{E}[\mathbf{R}(z)^\top \mathbf{R}(z)\mid\mathbf{w}]\succeq\kappa_R\mathbf{I}_{2d+1}$ for every $\mathbf{w}\in\mathbb{R}^d$.
    \item[(4)] \textbf{TD-target excitation.} $\mathbb{E}[\mathbf{b}(z)\mathbf{b}(z)^\top]\succeq\kappa_q\mathbf{I}_{2d+1}$.
    \item[(5)] \textbf{No destructive cancellation.} For every $c\in I$ and $(\mathbf{U},\mathbf{W})\in N_c\mathfrak{M}_I$,
    \begin{align*}
        \mathbb{E}_z\langle\mathbf{R}(z)\mathbf{U}\tilde{\mathbf{w}},\,\mathbf{W}\mathbf{b}(z)\rangle
        \geq-\rho\sqrt{\mathbb{E}_z\|\mathbf{R}(z)\mathbf{U}\tilde{\mathbf{w}}\|^2\cdot\mathbb{E}_z\|\mathbf{W}\mathbf{b}(z)\|^2}.
    \end{align*}
\end{enumerate}
At a high level, these are mild conditions on the training distribution. Conditions~(1)--(2) can be enforced by construction: normalizing features and clipping rewards gives boundedness, and drawing $\mathbf{w}$ from an isotropic distribution gives parameter excitation. Conditions~(3)--(5) are structural properties of the trajectory distribution that hold generically under a random feature map $\boldsymbol{\phi}$, random rewards, and $\epsilon$-greedy exploration ($\epsilon>0$). In Appendix~\ref{app: assumption verification} we discuss these conditions further and make the genericity claim quantitative: we develop a Monte Carlo procedure that estimates $(\kappa_{\tilde{w}},\kappa_R,\kappa_q,\rho)$ directly from the training prompts and use it to check the conditions numerically for the MDP family of our experiments (Section~\ref{sec: experiments}).
\end{asp}
\begin{thm}[Local convergence to the optimal parameter manifold]\label{thm:local-convergence-manifold}
Under Assumption \ref{asp:training-MDP assumptions}, let
\begin{gather*}
    B_\Sigma:=2B_\phi^2+B_r^2,\quad C_Q:=\frac{1}{2} B_\Sigma B_{\tilde{w}},\\
    m_0:=(1-\rho)
    \min\left\{c_+^{-2}\alpha^2\kappa_R\kappa_{\tilde{w}},c_-^2\kappa_q\right\}, \quad M_0:=B_\Sigma^2B_{\tilde{w}}^2\left(c_-^{-2}\alpha^2+2c_+^2\right).
\end{gather*}
Define the tube
\begin{align}\label{def:tube}
    \mathcal{T}_r(I):=\left\{\bar{\boldsymbol{\theta}}_{\text{eff}}(c)+\Delta: c\in I,\,\Delta\in N_c\mathfrak{M}_I,\,\|\Delta\|<r\right\},
\end{align}
where $0<r<\frac{\sqrt{m_0}}{3C_Q}$.
Fix a compact subinterval $I_0\subset \mathrm{int}(I)$. Let
\begin{align}\label{eq: def eta}
    \eta_{I_0,r}:=\mathrm{dist}\big(\mathfrak{M}_{I_0},\boldsymbol{\Theta}_{\text{eff}}\setminus \mathcal{T}_r(I)\big),
\end{align}
Suppose the initial parameter $\boldsymbol{\theta}_{\text{eff}}(0)$ satisfies the following condition: 
\begin{align}\label{eq:initialization condition}
    \mathrm{dist}(\boldsymbol{\theta}_{\text{eff}}(0),\mathfrak{M}_{I_0})<\min\left(r,\frac{\eta_{I_0,r}}{1+\frac{K_r\left(\sqrt{M_0}+C_Q r\right)}{\sqrt{2}\mu_r}}\right),
\end{align}
where $K_r:=\sqrt{2}B_\Sigma B_{\tilde{w}}\left(\left(c_+\sqrt{2d+1}+r\right)^2+\left(c_-^{-1}\alpha\sqrt{d}+r\right)^2\right)^{1/2}$. Consider the gradient flow \eqref{eq: gradient flow} with initial point $\boldsymbol{\theta}_{\text{eff}}(0)$. 
Then the gradient flow solution exists for all $t\geq 0$, remains in $\mathcal{T}_r(I)$, and satisfies
\begin{align*}
    \mathcal{L}(\boldsymbol{\theta}_{\text{eff}}(t)) \leq
    e^{-2\mu_r t}\big(\mathcal{L}(\boldsymbol{\theta}_{\text{eff}}(0))\big),
\end{align*}
where $\mu_r:=\frac{\left(m_0-3C_Q\sqrt{m_0}r\right)^2}{\left(\sqrt{M_0}+C_Qr\right)^2}$. 
Furthermore, $\boldsymbol{\theta}_{\text{eff}}(t)$ converges to a point $\boldsymbol{\theta}_{\text{eff}}(\infty)\in\mathfrak{M}_I$, and
\begin{align*}
    \|\boldsymbol{\theta}_{\text{eff}}(t)-\boldsymbol{\theta}_{\text{eff}}(\infty)\|\leq\frac{K_r}{\mu_r}\sqrt{\mathcal{L}(\boldsymbol{\theta}_{\text{eff}}(0))}\,e^{-\mu_r t}.
\end{align*}
\end{thm} 
Theorem \ref{thm:local-convergence-manifold} establishes that, when initialized sufficiently close to the optimal manifold $\mathfrak{M}_I$, gradient flow provably converges to a point on that manifold at an exponential rate governed by $\mu_r$. The key insight is that under Assumption \ref{asp:training-MDP assumptions} on the training MDPs, the loss function exhibits local contractivity: the manifold itself is a level set, but perturbations perpendicular to it decay exponentially. To briefly summarize our proof strategy, 
we first define the tube \eqref{def:tube},  which specifies the region around $\mathfrak{M}_I$ where a local Polyak--{\L}ojasiewicz (PL) inequality holds (Lemma \ref{lem:local-pl-ic-sarsa}). Next, the buffer distance \eqref{eq: def eta} quantifies the slackness between the inner manifold patch $\mathfrak{M}_{I_0}$ and the boundary of the tube, providing the margin a trajectory can travel before risking exit. The initialization condition \eqref{eq:initialization condition} then bounds the cumulative gradient-flow displacement by this margin $\eta_{I_0,r}$, which ensures that the trajectory never leaves the tube and the PL inequality remains active for all $t\geq 0$. The complete proof is deferred to Appendix \ref{proof of thm:local-convergence-manifold}.

\subsection{Transformers can implement Actor-Critic}\label{subsec: ic-actor-critic}
Having established both existence and convergence for the SARSA-based ICRL implementation, we now ask whether an analogous construction extends to the actor-critic method \eqref{eq: actor-critic batch}, in which the transformer must jointly update a value function and a parameterized policy. To accommodate this, we construct the input prompt as follows:
\begin{align}\label{eq: ac prompt}
    \mathbf{H}=\begin{bmatrix}
        \boldsymbol{\phi}_V(s_0)&\cdots&\boldsymbol{\phi}_V(s_{n-1})& \mathbf{0}\\
        \gamma\boldsymbol{\phi}_V(s_1)&\cdots& \gamma\boldsymbol{\phi}_V(s_{n})& \mathbf{0}\\
        r_1 & \cdots & r_n & 0\\
        \gamma^0\mathbf{g}_{\boldsymbol{\lambda}}(s_0, a_0) & \cdots &  \gamma^{n-1}\mathbf{g}_{\boldsymbol{\lambda}}(s_{n-1},a_{n-1}) &\mathbf{0}\\    
        0 & \cdots & 0 & 1\\
        \mathbf{0} & \cdots & \mathbf{0} & \boldsymbol{\lambda}\\
        \mathbf{0} & \cdots & \mathbf{0} & \mathbf{w}
    \end{bmatrix}\in\mathbb{R}^{D\times (n+1)},
\end{align}
where $D= 3d+2m+2$. Compared to the SARSA prompt \eqref{eq:input-matrix}, the construction in \eqref{eq: ac prompt} additionally encodes the discounted score function terms $\gamma^{i}\mathbf{g}_{\boldsymbol{\lambda}}(s_i,a_i)$ and the current actor parameter $\boldsymbol{\lambda}$. The following theorem shows that a single transformer block can still exactly implement the batch actor-critic update.
\begin{thm}[Implementing Actor-Critic with a transformer block]\label{thm: policy gradient tf}
    Let the output projection function be $f_{\text{read}}(\cdot)=\cdot_{-(d+m):,-1}$, which extracts the last $(d+m)$ entries of the final column. Given the input prompt \eqref{eq: ac prompt}, for any $\alpha, \beta>0$ there exist parameters $\boldsymbol{\theta}^\star=\{(c\mathbf{P}^\star,c^{-1}\mathbf{V}^\star)\}\subseteq\mathbb{R}^{D\times D}$ with $c \neq 0$ and
\begin{equation}\label{eq:QKV, actor-critic}
    \begin{gathered}
      \mathbf{P}^\star=\begin{bmatrix}
            \mathbf{*}_{(2d+m+1)\times(2d+m+1)} & \begin{bmatrix}
        \mathbf{0}_{d\times 1} & \mathbf{0}_{d\times m} & -\mathbf{I}_{d\times d}\\
        \mathbf{0}_{d\times 1} & \mathbf{0}_{d\times m} & \mathbf{I}_{d\times d}\\
        1 & \mathbf{0}_{1\times m} & \mathbf{0}_{1\times d}\\
        \mathbf{0}_{m\times 1} & \mathbf{0}_{m\times m} & \mathbf{0}_{m\times d}
        \end{bmatrix}\\
            \mathbf{*}_{(d+m+1)\times(2d+m+1)} & \mathbf{0}_{(d+m+1)\times (d+m+1)}
        \end{bmatrix},\\
     \mathbf{V}^\star=\begin{bmatrix}
            \mathbf{*}_{(2d+m+1)\times(2d+m+1)} & \mathbf{*}_{(2d+m+1)\times (d+m+1)}\\
            \begin{bmatrix}
        \mathbf{*}_{1\times d} & \mathbf{*}_{1\times d} & \mathbf{*}_{1\times 1} & \mathbf{*}_{1\times m}\\
        \mathbf{0}_{m\times d} & \mathbf{0}_{m\times d} & \mathbf{0}_{m\times 1} & \alpha \mathbf{I}_{m\times m}\\
        \beta \mathbf{I}_{d\times d} & \mathbf{0}_{d\times d} & \mathbf{0}_{d\times 1} & \mathbf{0}_{d\times m}
        \end{bmatrix} & \begin{bmatrix}
            \mathbf{*}_{1\times (d+m+1)}\\
            \mathbf{0}_{m\times (d+m+1)}\\
            \mathbf{0}_{d\times (d+m+1)}
            \end{bmatrix}
        \end{bmatrix},
    \end{gathered}
\end{equation}
such that the transformer $\textsf{TF}_{\boldsymbol{\theta}^\star}$ directly implements the batch actor-critic update \eqref{eq: actor-critic batch}.
\end{thm}
The proof of Theorem \ref{thm: policy gradient tf} is provided in Appendix \ref{app: proof of thm: policy gradient tf}. Similar to SARSA, we design a training procedure to learn the transformer parameters $\boldsymbol{\theta}$ in Algorithm \ref{alg: training transformer - actor critic} in Appendix \ref{app: training transformer}. The convergence analysis of Section~\ref{subsec: training-dynamics} extends to the actor-critic setting: in Appendix~\ref{app: ac convergence}, we show that under analogous richness conditions on the training-MDP distribution (Assumption~\ref{asp: ac richness}), the gradient flow of the population imitation loss converges locally and exponentially to the optimal parameter manifold induced by \eqref{eq:QKV, actor-critic} (Theorem~\ref{thm: ac local convergence}). Empirical validation of convergence is provided in Section \ref{sec: experiments}.

\section{Experiments}\label{sec: experiments}

We empirically demonstrate that gradient-based training of a single linear self-attention block converges to the optimal parameter manifold characterized in Theorem~\ref{thm: semi-sarsa tf} and Theorem~\ref{thm: policy gradient tf}, using Algorithm \ref{alg: training transformer} and Algorithm \ref{alg: training transformer - actor critic} respectively. Furthermore, we evaluate the control performance of the trained transformers, showing their ability to perform policy improvement on unseen MDPs.

\begin{figure}[t]
    \centering
    \captionsetup[subfigure]{skip=2pt,margin=1pt}
    \captionsetup{aboveskip=4pt,belowskip=0pt}
    \begin{minipage}{0.6\textwidth}
        \begin{subfigure}{0.49\linewidth}
            \centering
            \includegraphics[width=\linewidth]{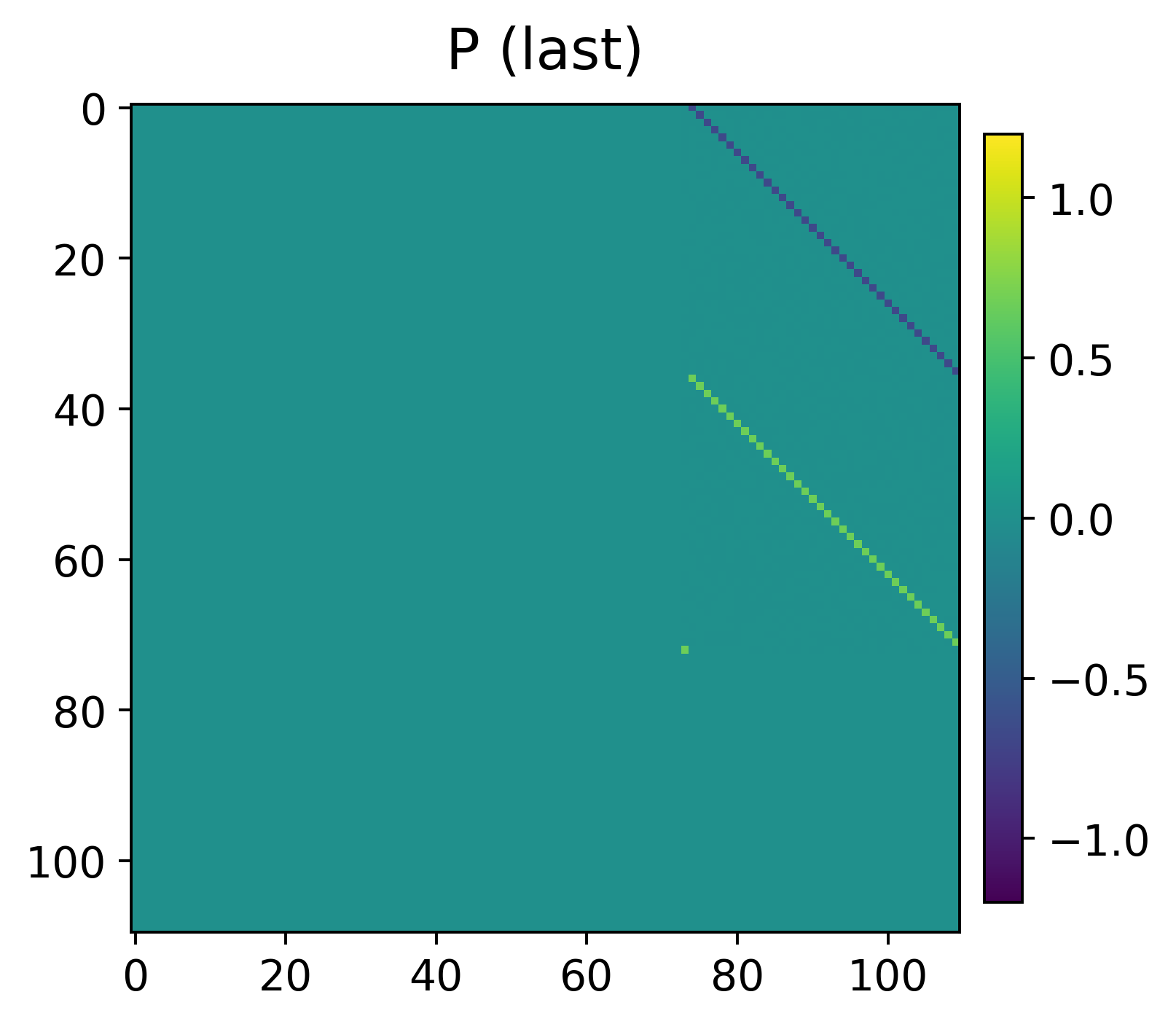}
            \caption{SARSA: learned $\mathbf{P}$.}
            \label{fig: parameter sarsa P}
        \end{subfigure}
        \hfill
        \begin{subfigure}{0.49\linewidth}
            \centering
            \includegraphics[width=\linewidth]{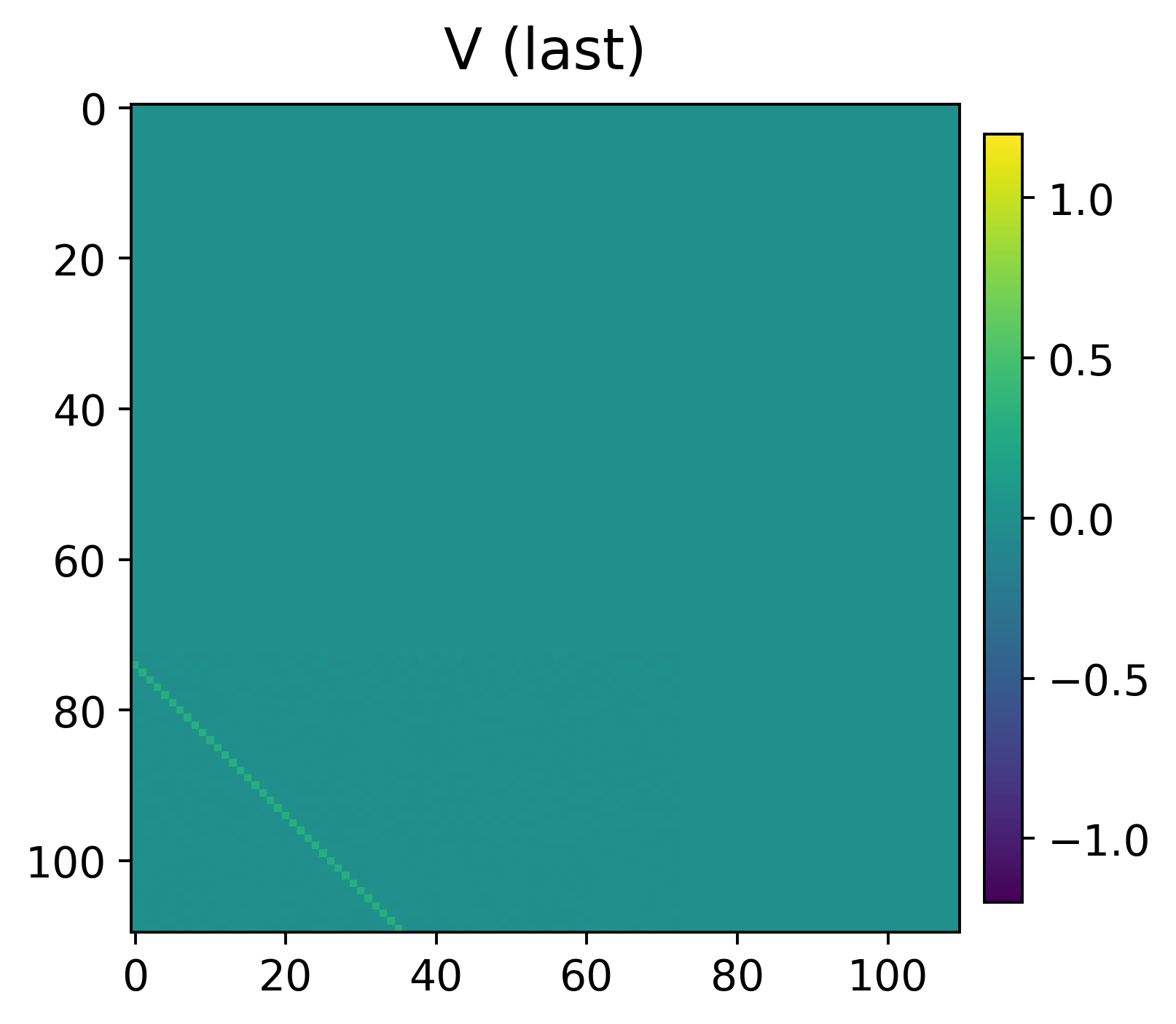}
            \caption{SARSA: learned $\mathbf{V}$.}
            \label{fig: parameter sarsa V}
        \end{subfigure}

        \begin{subfigure}{0.49\linewidth}
            \centering
            \includegraphics[width=\linewidth]{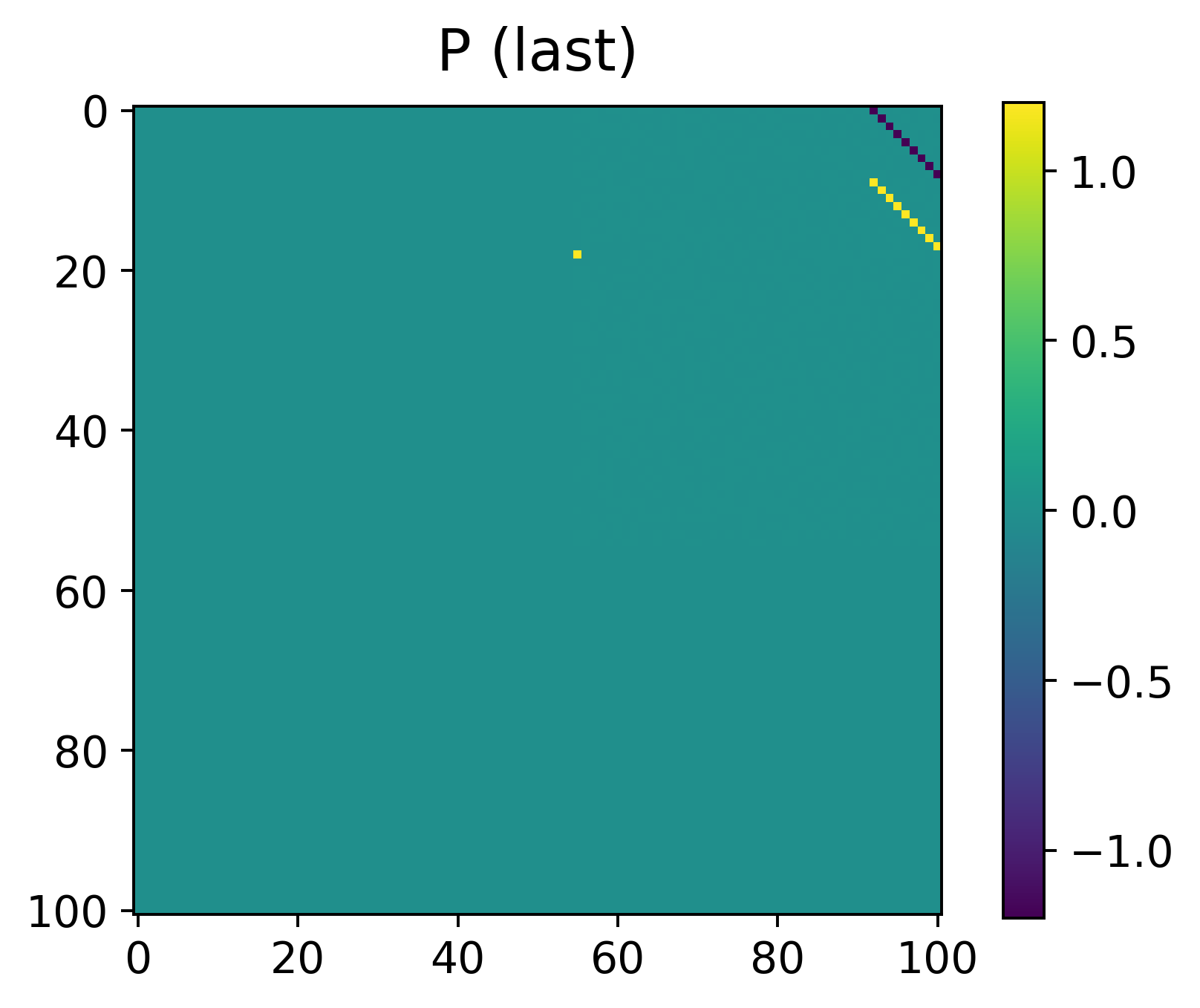}
            \caption{Actor-Critic: learned $\mathbf{P}$.}
            \label{fig: parameter ac P}
        \end{subfigure}
        \hfill
        \begin{subfigure}{0.49\linewidth}
            \centering
            \includegraphics[width=\linewidth]{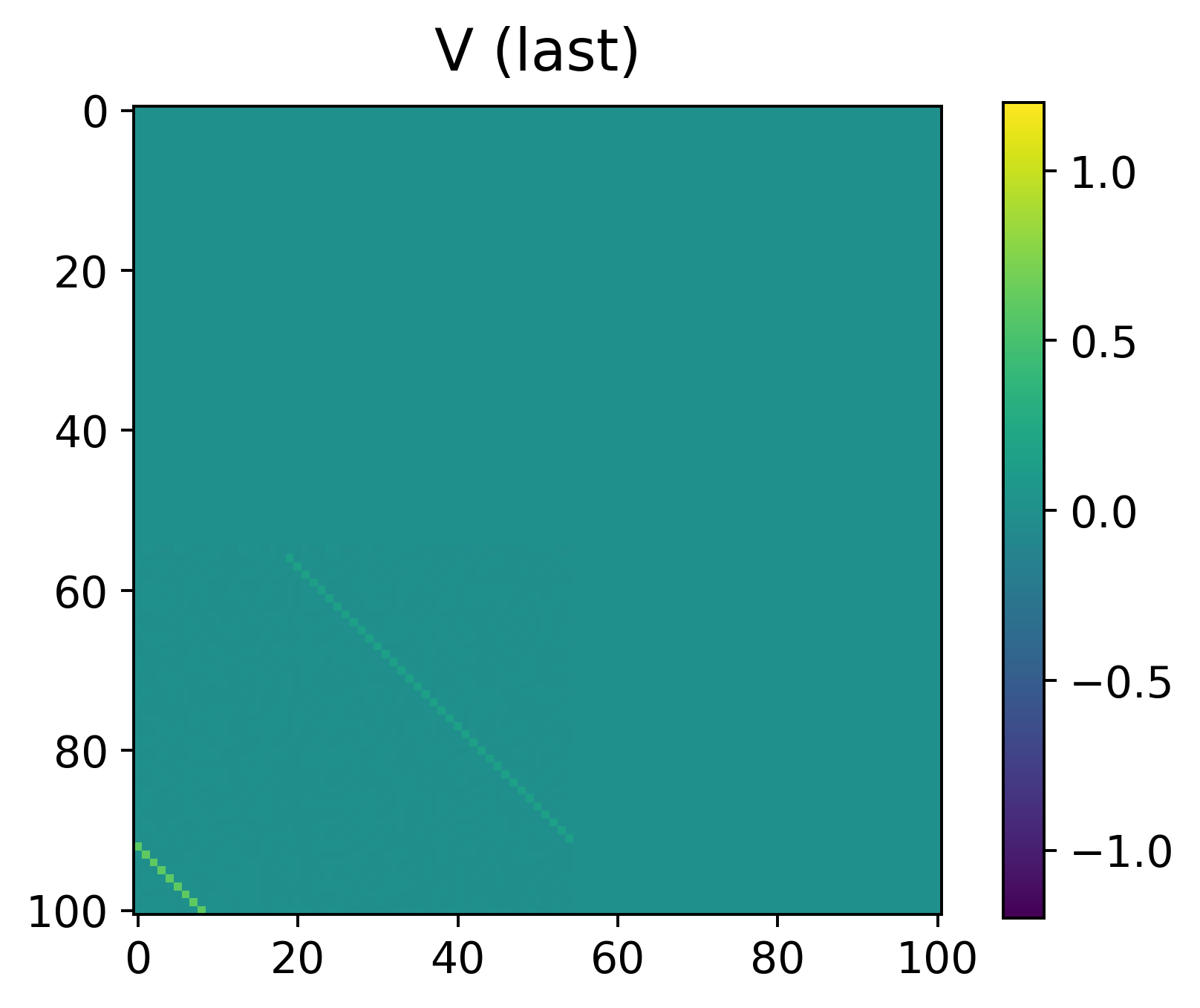}
            \caption{Actor-Critic: learned $\mathbf{V}$.}
            \label{fig: parameter ac V}
        \end{subfigure}
    \end{minipage}
    \caption{Final learned $\mathbf{P}$ and $\mathbf{V}$ matrices after training, for the SARSA transformer (top row) and the Actor-Critic transformer (bottom row). In both settings, the non-zero structure matches the theoretical constructions $(\mathbf{P}^\star,\mathbf{V}^\star)$ in~\eqref{eq:QKV} and \eqref{eq:QKV, actor-critic}.}
    \label{fig: parameter}
\end{figure}
\subsection{Training}\label{subsec: training}
\paragraph{Common Training Setup.}
Both experiments share the same training MDP family. Each MDP has $n_S=9$ states and $n_A=4$ actions, with discount $\gamma=0.5$. For each MDP, the initial-state distribution $P_0$ and transition kernel $P(\cdot\mid s,a)$ are drawn independently from $\mathrm{Dirichlet}(1,\ldots,1)$ (uniform over the probability simplex on $n_S$ states), and the rewards $r(a,s')\sim\mathrm{Unif}(-1,1)$ are sampled i.i.d. per $(a,s')$. Trajectories of window length $n=20$ are rolled out with horizon $T=1000$. In both experiments, the in-context learner is a one-layer transformer with corresponding embedding dimension $D$. Following Proposition~\ref{prop:invariant-set}, the inert blocks $\mathbf{P}_{11},\mathbf{P}_{21},\mathbf{V}_{11},\mathbf{V}_{12},\mathbf{P}_{22},\mathbf{V}_{22}$ are initialized to zero, while the effective sub-blocks $(\mathbf{P}_{12},\bar{\mathbf{V}}_{21})$ are initialized with Xavier-normal entries (gain $0.1$). Training is performed over $K=5{,}000$ independently sampled MDPs using Adam optimizer \citep{kingma2015adam} with initial learning rate $\eta=10^{-3}$ and exponential decay by a factor of $0.99$ every 10 MDPs.

\paragraph{SARSA experiment specifications.}
For each MDP, the state–action feature map $\boldsymbol{\phi}:\mathcal{S}\times\mathcal{A}\to\mathbb{R}^{d}$ ($d=n_S n_A=36$) and the initial weight $\mathbf{w}^{(0)}\in\mathbb{R}^{d}$ have entries drawn i.i.d.\ from $\mathrm{Unif}(-1,1)$. The behavior policy is $\epsilon$-greedy ($\epsilon=0.1$) with $Q(s,a)=\boldsymbol{\phi}(s,a)^\top\mathbf{w}$, and the SARSA target $\mathbf{w}_{\text{SARSA}}$ is computed from the trajectory window with step size $\alpha=0.2$. The transformer uses embedding dimension $D=3d+2=110$.

\paragraph{Actor-Critic experiment specifications.}
For each MDP, the policy feature $\boldsymbol{\phi}_\pi:\mathcal{S}\times\mathcal{A}\to\mathbb{R}^m$ ($m=36$), value feature $\boldsymbol{\phi}_V:\mathcal{S}\to\mathbb{R}^d$ ($d=9$), and initial weights $\boldsymbol{\lambda}^{(0)}\in\mathbb{R}^m$, $\mathbf{w}^{(0)}\in\mathbb{R}^d$ are drawn i.i.d.\ from $\mathrm{Unif}(-1,1)$. The behavior policy follows $\pi(a\mid s)\propto\exp(\boldsymbol{\phi}_\pi(s,a)^\top\boldsymbol{\lambda})$, with $\epsilon=0.1$ -- random exploration. At each step, the teacher applies a one-step critic update to $\mathbf{w}$ (step size $\beta=0.8$) and an actor update to $\boldsymbol{\lambda}$ (step size $\alpha=0.2$). The transformer uses embedding dimension $D=3d+2m+2=101$.

\paragraph{Results.}
In both experiments, the training loss decays approximately exponentially to zero (see Appendix \ref{app: training loss}). Figure~\ref{fig: parameter} shows the learned $\mathbf{P}$ and $\mathbf{V}$ matrices: the upper-right block of $\mathbf{P}$ and the relevant rows of the lower-left block of $\mathbf{V}$ recover the structure of $(\mathbf{P}^\star,\mathbf{V}^\star)$ in~\eqref{eq:QKV} and \eqref{eq:QKV, actor-critic} (up to scaling), while other blocks remain near zero. This confirms that the ICRL training algorithms \ref{alg: training transformer} and \ref{alg: training transformer - actor critic} drive the parameters toward the optimal manifold.

\subsection{Evaluation}\label{subsec: evaluation}

Beyond parameter-level fidelity, we evaluate whether the trained transformer, when deployed as an in-context update rule on unseen MDPs, induces high-return policies. We randomly sample $K=100$ test MDPs with corresponding features, disjoint from the training set. For each MDP, we run closed-loop in-context updates: at each step, the prompt $\mathbf{H}^{(t)}$ is fed to the transformer to produce $\mathbf{w}^{(t+1)}$ (or $(\boldsymbol{\lambda}^{(t+1)}, \mathbf{w}^{(t+1)})$ in the actor–critic setting), which defines the behavior policy used to generate the next trajectory window. Every 10 update steps, we freeze the iterate and estimate the discounted return via Monte Carlo using 32 trajectories of length 50, with initial state $s_0 \sim P_0^{(k)}$.

We compare against analytical teacher updates (SARSA \eqref{eq:semi-sarsa linear batch} and actor–critic \eqref{eq: actor-critic batch}) executed in the same closed loop, the oracle greedy policy under $Q^\star$ obtained by exact value iteration on $(P^{(k)}, R^{(k)})$, and a uniform random policy. Figure~\ref{fig: control performance} reports the average return across MDPs. In both settings, the transformer closely tracks the teacher: performance improves rapidly from near-random initialization and approaches the oracle as $t$ increases. Crucially, this consistency persists across long in-context horizons, indicating that the transformer faithfully implements the teacher update rule rather than merely mimicking short-horizon behavior. 

Notably, the remaining gap to the oracle reflects limitations of the linear function approximation rather than imperfect imitation, since the same gap is also observed for the analytical teacher. These results show that convergence to the optimal parameter manifold (Theorem~\ref{thm:local-convergence-manifold}) translates into strong policy-level performance, with the trained transformer accurately reproducing the analytical update rules at deployment.
\begin{figure}[t]
    \centering
    \begin{minipage}{0.88\textwidth}
        \begin{subfigure}{0.49\linewidth}
            \centering
            \includegraphics[width=\linewidth]{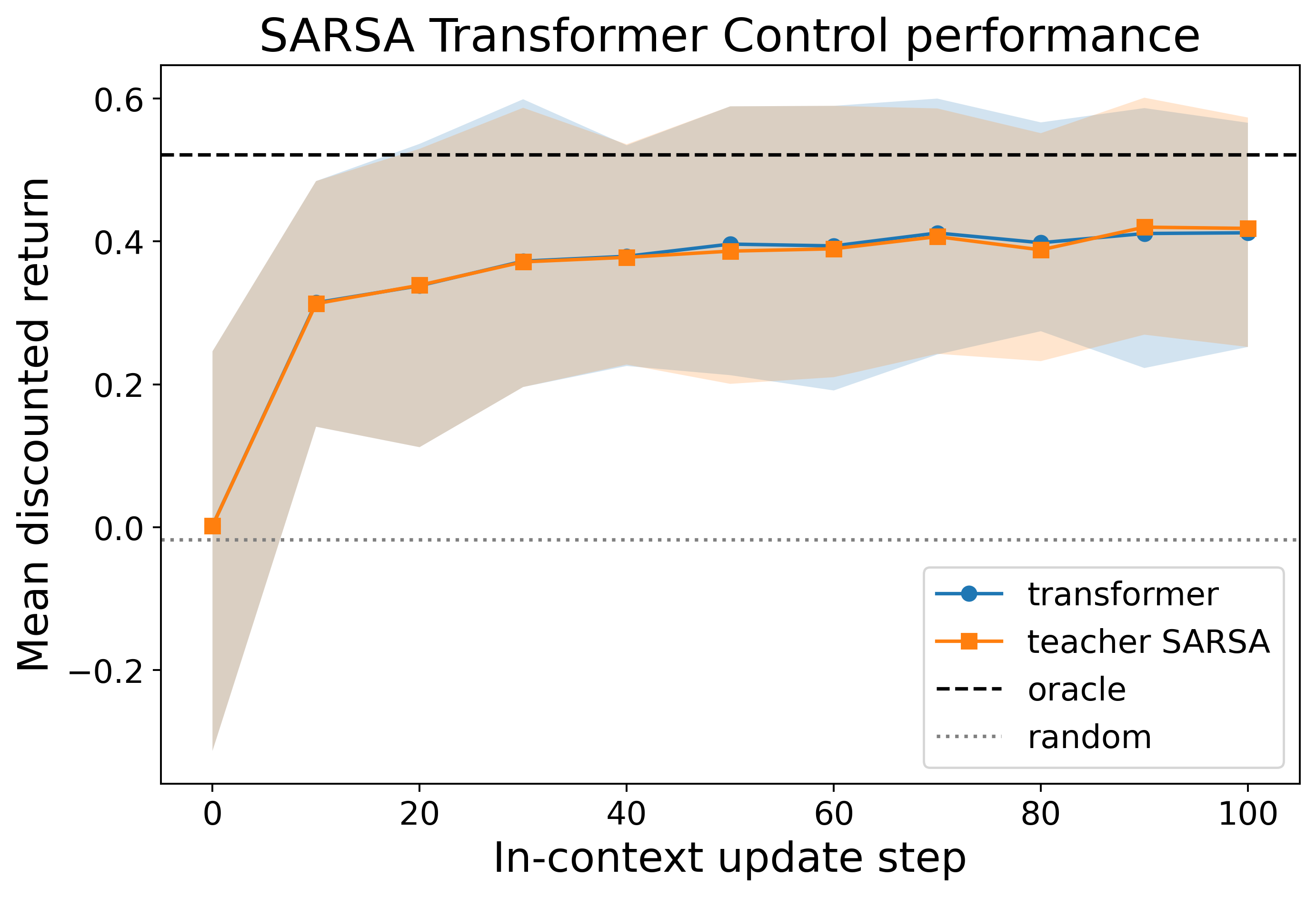}
            \caption{SARSA transformer.}
            \label{fig: control sarsa}
        \end{subfigure}
        \hfill
        \begin{subfigure}{0.49\linewidth}
            \centering
            \includegraphics[width=\linewidth]{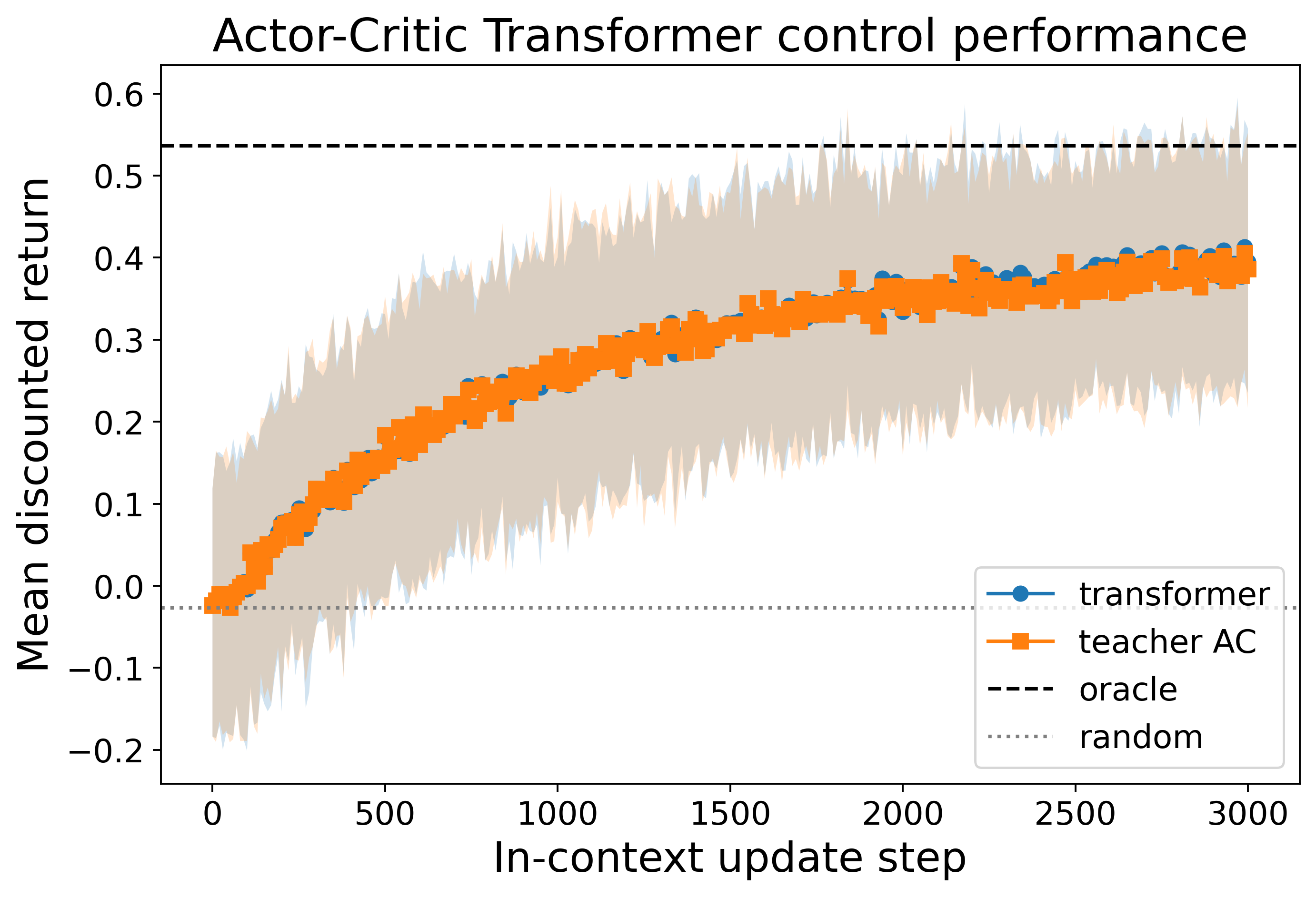}
            \caption{Actor-Critic transformer.}
            \label{fig: control ac}
        \end{subfigure}
    \end{minipage}
    \caption{Closed-loop control performance on $100$ held-out random MDPs versus in-context update step $t$. Curves show the per-MDP return averaged across MDPs (shaded: $25$--$75\%$ interquartile band). The transformer curve closely tracks the teacher curve in both settings, indicating that the trained model approximately implements the analytical update rule at deployment.}
    \label{fig: control performance}
\end{figure}

\subsection{Additional Experiments}\label{subsec: additional experiments}

Additional details and experiments can be found in Appendix~\ref{app:experiments}. We summarize some of the findings below.

\begin{itemize}[leftmargin=*,nosep]
    \item The cosine similarity between the learned effective blocks $(\mathbf{P}_{12},\bar{\mathbf{V}}_{21})$ and the theoretical constructions increases monotonically with the training budget and reaches $1.000$ at $K=5{,}000$ in both settings (Appendix~\ref{app: cosine similarity}).
    \item The predicted block structure also emerges when training on a non-tabular, continuous-state MountainCar task family (Appendix~\ref{app: mountaincar}).
    \item For nonlinear attention, a two-head ReLU-attention transformer matches the teacher, whereas single-head ReLU and softmax variants fall short (Appendix~\ref{app: nonlinear attention}).
    \item The trained linear-attention transformer also generalizes beyond the training distribution: deployed on an unseen structured cliff-walk task, the same frozen SARSA transformer matches the analytical teacher to two decimals and finds the shortest safe path around the cliff (Appendix~\ref{app: cliffwalk}).
    \item We ablate the actor-critic prompt by replacing the pre-computed score-function terms with raw policy features; the transformer no longer recovers the exact actor update (Appendix~\ref{app: raw prompt}).
\end{itemize}

\section{Conclusion}

In this work, we demonstrate that transformers can provably implement ICRL algorithms that perform policy improvement. Through explicit constructions, we show that a single linear-attention block can realize batch SARSA and actor-critic updates. Furthermore, we provide the first convergence guarantee in the ICRL literature, proving that under suitable richness conditions on the training MDPs, gradient-flow training to mimic the teacher update converges locally to the optimal parameter manifold. Our empirical results validate these theoretical findings, showing that trained transformers recover the predicted parameter structures and achieve strong control performance on unseen MDPs. These results help bridge the gap between the theoretical understanding and practical implementation of ICRL, suggesting that the transformer architecture naturally favors learning and implementing reinforcement-learning algorithmic structures. Overall, this work advances our understanding of how transformers perform RL algorithms in context, with implications for both AI safety and the design of more capable learning systems.

\subsection{Limitations}
Our work has several limitations. First, our convergence guarantee is local and requires initialization near the optimal manifold; understanding the global optimization landscape is an important direction for future work. Second, our theoretical results focus on linear attention; extending the convergence analysis to nonlinear attention mechanisms such as ReLU and softmax remains open. Finally, our analysis focuses on deterministic gradient flow in a simplified setting, leaving extensions to stochastic optimization and more complex transformer architectures for future study.

\begin{ack}

\end{ack}
This work was supported in part by National Science Foundation under grants CCF-2232907 and ECCS-2448268. The authors declare no competing interests.
\newpage
\bibliographystyle{unsrtnat}
\bibliography{neurips_2026}


\appendix
\section{Additional Related Works}\label{app: related works}
In this appendix, we discuss additional works that are related to our paper.

\paragraph{In-context learning theory for supervised tasks.}
\citet{garg2023transformers} initiate this discussion on simple function classes, showing empirically that a transformer trained from scratch can in-context learn linear regression, two-layer neural networks, and decision trees. A subsequent line of mechanistic work then asks \emph{how} such behavior arises. \citet{vonoswald2023transformers} and \citet{akyurek2023what} argue that transformers implement gradient descent or closed-form least squares on the in-context examples, and \citet{ahn2023transformers} and \citet{mahankali2024one} prove that, on linear regression, the optimal one-layer linear self-attention block is exactly one step of preconditioned gradient descent. \citet{bai2023transformers} broaden the picture by showing that transformers can implement a wide family of statistical estimators -- including ridge regression, Lasso, and gradient descent on generalized linear models -- and can even perform algorithm selection in context. \citet{liang2025transformers} further extend this line to in-context linear regression with endogenous covariates, constructing a multi-layer transformer that implements two-stage least squares and showing that it consistently estimates the structural parameter under standard instrumental-variable assumptions.

\paragraph{In-context reinforcement learning.}
The view of RL as in-context sequence modeling was popularized by the Decision Transformer \citep{chen2021decision} and the Trajectory Transformer \citep{janner2021trajectory}, which cast policy improvement as autoregressive prediction over offline trajectories. Building on this paradigm, \citet{laskin2023algorithm} introduce \emph{algorithm distillation}, in which a causal transformer trained on RL learning histories reproduces the underlying learning algorithm at inference time, and \citet{lee2023supervised} show that supervised pretraining over offline trajectories suffices to elicit ICRL on new tasks. Subsequent empirical work demonstrates that large language models can in-context execute policy iteration on tabular MDPs \citep{brooks2023large}; that pretraining on sufficiently diverse trajectories enables transformers to generalize to unseen sequential decision-making tasks purely from their context \citep{raparthy2024generalization}; that scalable long-context architectures such as AMAGO learn adaptive policies in context across challenging environments \citep{grigsby2024amago}; and that LLMs can perform nontrivial exploration in bandit and small MDP problems even without external scaffolding \citep{krishnamurthy2024can}. On the theoretical side, \citet{lin2024transformers} prove that supervised pretraining of transformers on offline trajectories yields near-optimal posterior-sampling policies, establishing sample-complexity guarantees for ICRL pretraining. The work most closely related to ours is \citet{TCLTD}, which shows that a linear-attention block admits an invariant set under the expected gradient update of a value-prediction loss, thereby implementing policy evaluation via TD methods. Our work extends this line in two directions: from policy evaluation to policy improvement (SARSA and actor-critic), and from invariance to convergence of gradient training. In concurrent work, \citet{xie2026beyond} move beyond the linear-attention simplification and study in-context policy evaluation under standard softmax attention.

\paragraph{Training Dynamics.}
A complementary line of work studies the optimization landscape and training dynamics of attention itself, asking whether the in-context algorithms identified by mechanistic constructions are actually recoverable by gradient-based learning. \citet{zhang2024trained} prove that gradient flow on a single linear self-attention block trained on random linear-regression tasks converges to the global minimum and recovers the preconditioned-gradient-descent solution predicted by \citet{ahn2023transformers,mahankali2024one}. \citet{huang2023context} extend such convergence guarantees to softmax attention under orthogonal feature assumptions, and \citet{chen2024training} analyze multi-head softmax attention in the in-context regression setting, identifying an emergent task-specialization phenomenon en route to the global minimizer. From a different angle, \citet{tian2023scan} dissect token-level dynamics of one-layer attention during language pretraining, while \citet{nichani2024how} show that transformers can recover hidden causal structure via gradient descent on synthetic tasks. These results are confined to supervised regression-style objectives. Our work pushes the training-dynamics analysis into ICRL, where the loss is non-convex and the optima form a one-dimensional scaling manifold rather than an isolated point.

\section{Training Algorithm}\label{app: training transformer}
\subsection{Actor-Critic Transformer}

\begin{algorithm}[H]
    \begin{algorithmic}[1]
        \caption{Actor Critic Transformer Training}\label{alg: training transformer - actor critic}
        \State \textbf{Input:} Transformer parameter $\boldsymbol{\theta}^{(0)}=(\boldsymbol{P}^{(0)}, \boldsymbol{V}^{(0)})$, trajectory window length $n$, number of frames per MDP $T$, number of MDPs $K$, $\epsilon$-greedy parameter $\epsilon$, learning rate $\eta$, actor step size $\alpha$, and critic step size $\beta$
        \For{$k=0,\ldots,K-1$}
        \State Sample a MDP $\mathcal{M}^{(k)}=(\mathcal{S},\mathcal{A},P^{(k)},R^{(k)}, \gamma)$, and generate corresponding feature map for policy $\boldsymbol{\phi}_\pi^{(k)}:\mathcal{S}\times\mathcal{A}\to\mathbb{R}^m$, and feature map for state value function $\boldsymbol{\phi}_V^{(k)}:\mathcal{S}\to\mathbb{R}^d$.Initialize parameters  $\boldsymbol{\lambda}^{(0)}, \mathbf{w}^{(0)}$, and state $s_0^{(0)}$ randomly
        \For{$t=0,\ldots,T-1$}
        \State Calculate policy $\pi_{\boldsymbol{\lambda}^{(t)}}^{(k)}(a|s)=\frac{\exp\{\boldsymbol{\lambda}^{(t)\top}\boldsymbol{\phi}_\pi^{(k)}(s,a)\}}{\sum_{b}\exp\{\boldsymbol{\lambda}^{(t)\top}\boldsymbol{\phi}_\pi^{(k)}(s,b)\}}, \forall s\in\mathcal{S}, a\in\mathcal{A}$
        \State Estimate state value function $v_{\mathbf{w}^{(t)}}^{(k)}(s)=\mathbf{w}^{(t)\top}\boldsymbol{\phi}_V^{(k)}(s),\quad \forall s\in\mathcal{S}$
        \State Compute policy gradient $\mathbf{g}_{\boldsymbol{\lambda}^{(t)}}^{(k)}(s,a)=\boldsymbol{\phi}_\pi^{(k)}(s,a)-\sum_{b}\boldsymbol{\phi}_\pi^{(k)}(s,b)\pi_{\boldsymbol{\lambda}^{(t)}}^{(k)}(b|s), \forall (s,a)\in\mathcal{S}\times\mathcal{A}$
        \State Sample n steps $(s_0^{(t)},a_0^{(t)},r_1^{(t)},s_1^{(t)},a_1^{(t)},\ldots,r_{n}^{(t)},s_{n}^{(t)},a_{n}^{(t)})$ from the MDP $\mathcal{M}^{(k)}$, according to policy $\pi_{\boldsymbol{\lambda}^{(t)}}^{(k)}$ with $\epsilon$-greedy exploration
        \State Formulate $\mathbf{H}^{(t)}$ as in \eqref{eq: ac prompt} with the sampled trajectory and current parameters $\boldsymbol{\lambda}^{(t)}, \mathbf{w}^{(t)}$
        \State Let $\begin{bmatrix}
        \boldsymbol{\lambda}^{(t+1)}\\
        \mathbf{w}^{(t+1)}
        \end{bmatrix}=\text{TF}_{\boldsymbol{\theta}^{(t)}}(\mathbf{H}^{(t)}):=\begin{bmatrix}
            \boldsymbol{\lambda}^{(t)}\\
            \mathbf{w}^{(t)}
        \end{bmatrix}+\frac{1}{n}\left[\mathbf{V}^{(t)}\mathbf{H}^{(t)}\mathbf{H}^{(t)\top}\mathbf{P}^{(t)}\mathbf{H}^{(t)}\right]_{-(d+m):,-1}$
        \State Compute $\mathbf{w}_{\text{AC}}^{(t+1)}$ and $\boldsymbol{\lambda}_{\text{AC}}^{(t+1)}$ according to \eqref{eq: actor-critic batch}
        \State Compute loss $\hat{L}^{(t)}=\frac{1}{2}\|\mathbf{w}^{(t+1)}-\mathbf{w}_{\text{AC}}^{(t+1)}\|^2+\frac{1}{2}\|\boldsymbol{\lambda}^{(t+1)}-\boldsymbol{\lambda}_{\text{AC}}^{(t+1)}\|^2$
        \State Update $\boldsymbol{\theta}^{(t+1)}\leftarrow \boldsymbol{\theta}^{(t)}-\eta\nabla_{\boldsymbol{\theta}}\hat{L}^{(t)}$
        \State Update $s_0^{(t+1)}\leftarrow s_{n}^{(t)}$, $\boldsymbol{\lambda}^{(t+1)}\leftarrow \boldsymbol{\lambda}_{\text{AC}}^{(t+1)}$,  $\mathbf{w}^{(t+1)}\leftarrow \mathbf{w}_{\text{AC}}^{(t+1)}$
        \EndFor
        \State Update $\boldsymbol{\theta}^{(0)}\leftarrow \boldsymbol{\theta}^{(T)}$
        \EndFor
        \State \textbf{Return} Learned transformer parameter $\boldsymbol{\theta}$
    \end{algorithmic}
\end{algorithm}
Algorithm \ref{alg: training transformer - actor critic} extends Algorithm \ref{alg: training transformer} to jointly learn both the value and policy parameters. At each iteration, trajectories are sampled from the MDP, and the corresponding actor-critic updates are computed as teacher signals for both the critic parameter $\mathbf{w}$ and the actor parameter $\boldsymbol{\lambda}$. The transformer parameters are then updated via gradient descent to match these targets. The empirical loss combines the squared errors for both parameter updates, encouraging the transformer to approximate the full actor-critic learning rule. 

\section{Supporting Lemmas}\label{app: supporting lemmas}
\begin{lemma}[Transformer output decomposition]\label{lem: transformer output}
   Consider a transformer block with parameters $\boldsymbol{\theta}=\{(\mathbf{P},\mathbf{V})\}\subseteq\mathbb{R}^{D\times D}$. We decompose the parameters into blocks as follows:
    \begin{align*}
        \mathbf{P}=\begin{bmatrix}
            \mathbf{P}_{11} & \mathbf{P}_{12}\\
            \mathbf{P}_{21} & \mathbf{P}_{22}
        \end{bmatrix},\qquad
        \mathbf{V}=\begin{bmatrix}
            \mathbf{V}_{11} & \mathbf{V}_{12}\\
            \mathbf{V}_{21} & \mathbf{V}_{22}
        \end{bmatrix},
    \end{align*}
    where $\mathbf{P}_{11},\mathbf{V}_{11}\in\mathbb{R}^{(2d+1)\times(2d+1)}$, $\mathbf{P}_{12},\mathbf{V}_{12}\in\mathbb{R}^{(2d+1)\times(d+1)}$, $\mathbf{P}_{21},\mathbf{V}_{21}\in\mathbb{R}^{(d+1)\times(2d+1)}$, and $\mathbf{P}_{22},\mathbf{V}_{22}\in\mathbb{R}^{(d+1)\times(d+1)}$. Let the output projection function be $f_{\text{read}}(\cdot)=\cdot_{-d:,-1}$, which extracts the last $d$ elements of the last column. Denote $\tilde{\mathbf{w}}:=\begin{bmatrix} 1 & \mathbf{w} \end{bmatrix}^\top$. Given formulated input prompt 
    \begin{align*}
    \mathbf{H}=\begin{bmatrix}
    \mathbf{x}_0 & \cdots & \mathbf{x}_{n-1} & \mathbf{0}\\
    \mathbf{0} & \cdots & \mathbf{0} & \tilde{\mathbf{w}}\\
    \end{bmatrix}\in\mathbb{R}^{D\times (n+1)},
    \end{align*}
 the transformer output can then be written as:
    \begin{align}\label{eq:transformer-output}
        \textsf{TF}_{\boldsymbol{\theta}}(\mathbf{H})=\mathbf{w}+\bar{\mathbf{V}}_{21}\hat{\boldsymbol{\Sigma}}\mathbf{P}_{12}\tilde{\mathbf{w}}+ \frac{1}{n}\bar{\mathbf{V}}_{22}\tilde{\mathbf{w}}\tilde{\mathbf{w}}^\top\mathbf{P}_{22}\tilde{\mathbf{w}},
    \end{align}
    where $\hat{\boldsymbol{\Sigma}}:=\frac{1}{n}\sum_{i=0}^{n-1}\mathbf{x}_i\mathbf{x}_i^\top$, and $\bar{\mathbf{V}}_{21}, \bar{\mathbf{V}}_{22}$ denote the last $d$ rows of $\mathbf{V}_{21}$ and $\mathbf{V}_{22}$, respectively.
\end{lemma}
\begin{lemma}[Local Polyak--\L{}ojasiewicz inequality]\label{lem:local-pl-ic-sarsa}
Suppose Assumption~\ref{asp:training-MDP assumptions} holds, and let the optimal manifold $\mathfrak{M}_I$, the constants $B_{\tilde{w}},B_\Sigma,C_Q,m_0,M_0$, and the tube $\mathcal{T}_r(I)$ be as defined in Theorem~\ref{thm:local-convergence-manifold}. Then $\mathcal{L}$ attains its minimum value $\mathcal{L}^\star=0$ on $\mathfrak{M}_I$, and for every $0<r<\frac{\sqrt{m_0}}{3C_Q}$ the population loss satisfies
\begin{align}\label{eq: local pl}
    \frac{1}{2}\|\nabla\mathcal{L}(\boldsymbol{\theta}_{\text{eff}})\|^2
    \geq
    \mu_r\bigl(\mathcal{L}(\boldsymbol{\theta}_{\text{eff}})-\mathcal{L}^\star\bigr),
    \qquad\forall\,\boldsymbol{\theta}_{\text{eff}}\in\mathcal{T}_r(I),
\end{align}
where
\begin{align*}
    \mu_r:=\frac{\bigl(m_0-3C_Q\sqrt{m_0}\,r\bigr)^2}{\bigl(\sqrt{M_0}+C_Q r\bigr)^2}.
\end{align*}
\end{lemma}

\section{Proof of Supporting Lemmas}\label{app: proof of supporting lemmas}
\subsection{Proof of Lemma \ref{lem: transformer output}}
\begin{proof}
From \eqref{eq:input-matrix}, the input matrix is 
\begin{align*}
     \mathbf{H}=\begin{bmatrix}
    \mathbf{x}_0 & \cdots & \mathbf{x}_{n-1} & \mathbf{0}\\
    \mathbf{0} & \cdots & \mathbf{0} & \tilde{\mathbf{w}}\\
    \end{bmatrix}\in\mathbb{R}^{D\times (n+1)}.    
\end{align*}
By the definition of the linear self-attention block in \eqref{def:hLin},
\begin{align*}
    \mathbf{H}_{\text{out}}=\mathbf{H}+\frac{1}{n}(\mathbf{V}\mathbf{H})(\mathbf{H}^\top\mathbf{P}\mathbf{H})\in\mathbb{R}^{D\times(n+1)},
\end{align*}
and the readout extracts the last $d$ entries of the last column of $\mathbf{H}_{\text{out}}$. 
Denote the last column of the input by
\begin{align*}
    \mathbf{h}_{n+1}:=\begin{bmatrix}\mathbf{0}_{2d+1}\\ \tilde{\mathbf{w}}\end{bmatrix}\in\mathbb{R}^{D}.
\end{align*}
Then we have
\begin{align}\label{eq:tf-last-col}
    \mathbf{H}_{\text{out}}[:,n+1]=\mathbf{h}_{n+1}+\frac{1}{n}(\mathbf{V}\mathbf{H})(\mathbf{H}^\top\mathbf{P}\mathbf{h}_{n+1}).
\end{align}
Using the block partition \eqref{eq:block partition},
\begin{align*}
    \mathbf{P}\mathbf{h}_{n+1}=\begin{bmatrix}
        \mathbf{P}_{11} & \mathbf{P}_{12}\\
        \mathbf{P}_{21} & \mathbf{P}_{22}
    \end{bmatrix}\begin{bmatrix}\mathbf{0}\\ \tilde{\mathbf{w}}\end{bmatrix}=\begin{bmatrix}\mathbf{P}_{12}\tilde{\mathbf{w}}\\ \mathbf{P}_{22}\tilde{\mathbf{w}}\end{bmatrix}.
\end{align*}
Hence
\begin{equation}\label{eq:tf-lastcol-P}
\begin{gathered}
    \bigl(\mathbf{H}^\top\mathbf{P}\mathbf{h}_{n+1}\bigr)_{i+1}=\begin{bmatrix}\mathbf{x}_i^\top & \mathbf{0}\end{bmatrix}\begin{bmatrix}\mathbf{P}_{12}\tilde{\mathbf{w}}\\ \mathbf{P}_{22}\tilde{\mathbf{w}}\end{bmatrix}=\mathbf{x}_i^\top\mathbf{P}_{12}\tilde{\mathbf{w}},\quad i=0,\ldots,n-1,\\
    \bigl(\mathbf{H}^\top\mathbf{P}\mathbf{h}_{n+1}\bigr)_{n+1}=\begin{bmatrix}\mathbf{0} & \tilde{\mathbf{w}}^\top\end{bmatrix}\begin{bmatrix}\mathbf{P}_{12}\tilde{\mathbf{w}}\\ \mathbf{P}_{22}\tilde{\mathbf{w}}\end{bmatrix}=\tilde{\mathbf{w}}^\top\mathbf{P}_{22}\tilde{\mathbf{w}}.
\end{gathered}
\end{equation}
Similarly,
\begin{equation}\label{eq:tf-lastcol-V}
\begin{gathered}
    (\mathbf{V}\mathbf{H})[:,i+1]=\mathbf{V}\begin{bmatrix}\mathbf{x}_i\\ \mathbf{0}\end{bmatrix}=\begin{bmatrix}\mathbf{V}_{11}\mathbf{x}_i\\ \mathbf{V}_{21}\mathbf{x}_i\end{bmatrix},\quad i=0,\ldots,n-1,\\
    (\mathbf{V}\mathbf{H})[:,n+1]=\mathbf{V}\begin{bmatrix}\mathbf{0}\\ \tilde{\mathbf{w}}\end{bmatrix}=\begin{bmatrix}\mathbf{V}_{12}\tilde{\mathbf{w}}\\ \mathbf{V}_{22}\tilde{\mathbf{w}}\end{bmatrix}.
\end{gathered}
\end{equation}
Expanding the matrix-vector product as a sum over columns and combining \eqref{eq:tf-lastcol-P} and \eqref{eq:tf-lastcol-V}, we have
\begin{align*}
    (\mathbf{V}\mathbf{H})\bigl(\mathbf{H}^\top\mathbf{P}\mathbf{h}_{n+1}\bigr)
    &=\sum_{i=0}^{n-1}\begin{bmatrix}\mathbf{V}_{11}\mathbf{x}_i\\ \mathbf{V}_{21}\mathbf{x}_i\end{bmatrix}\mathbf{x}_i^\top\mathbf{P}_{12}\tilde{\mathbf{w}}+\begin{bmatrix}\mathbf{V}_{12}\tilde{\mathbf{w}}\\ \mathbf{V}_{22}\tilde{\mathbf{w}}\end{bmatrix}\tilde{\mathbf{w}}^\top\mathbf{P}_{22}\tilde{\mathbf{w}}\\
    &=\begin{bmatrix}
        \mathbf{V}_{11}\bigl(\sum_{i=0}^{n-1}\mathbf{x}_i\mathbf{x}_i^\top\bigr)\mathbf{P}_{12}\tilde{\mathbf{w}}+\mathbf{V}_{12}\tilde{\mathbf{w}}\tilde{\mathbf{w}}^\top\mathbf{P}_{22}\tilde{\mathbf{w}}\\
        \mathbf{V}_{21}\bigl(\sum_{i=0}^{n-1}\mathbf{x}_i\mathbf{x}_i^\top\bigr)\mathbf{P}_{12}\tilde{\mathbf{w}}+\mathbf{V}_{22}\tilde{\mathbf{w}}\tilde{\mathbf{w}}^\top\mathbf{P}_{22}\tilde{\mathbf{w}}
    \end{bmatrix}.
\end{align*}
Thus, rescaling by $\frac{1}{n}$, the lower $(d+1)$-dimensional block of \eqref{eq:tf-last-col} reads
\begin{align*}
    \tilde{\mathbf{w}}+\mathbf{V}_{21}\hat{\boldsymbol{\Sigma}}\mathbf{P}_{12}\tilde{\mathbf{w}}+\frac{1}{n}\mathbf{V}_{22}\tilde{\mathbf{w}}\tilde{\mathbf{w}}^\top\mathbf{P}_{22}\tilde{\mathbf{w}}.
\end{align*}
Since $\tilde{\mathbf{w}}=[1,\mathbf{w}^\top]^\top$, the last $d$ entries of $\tilde{\mathbf{w}}$ equal $\mathbf{w}$. The readout $f_{\text{read}}(\cdot)=\cdot_{-d:,-1}$ extracts the last $d$ rows of the lower block, which selects the last $d$ rows of $\mathbf{V}_{21}$ and $\mathbf{V}_{22}$, namely $\bar{\mathbf{V}}_{21}$ and $\bar{\mathbf{V}}_{22}$. Therefore,
\begin{align*}
    \textsf{TF}_{\boldsymbol{\theta}}(\mathbf{H})=\mathbf{w}+\bar{\mathbf{V}}_{21}\hat{\boldsymbol{\Sigma}}\mathbf{P}_{12}\tilde{\mathbf{w}}+\frac{1}{n}\bar{\mathbf{V}}_{22}\tilde{\mathbf{w}}\tilde{\mathbf{w}}^\top\mathbf{P}_{22}\tilde{\mathbf{w}},
\end{align*}
which completes the proof.
\end{proof}
\subsection{Proof of Lemma \ref{lem:local-pl-ic-sarsa}}
\begin{proof}
For any effective parameter $\boldsymbol{\theta}_{\text{eff}}=(\mathbf{P}_{12},\bar{\mathbf{V}}_{21})$, define the auxiliary predictor
\begin{align*}
    f_{\boldsymbol{\theta}_{\text{eff}}}(z):=\bar{\mathbf{V}}_{21}\hat{\boldsymbol{\Sigma}}(z)\mathbf{P}_{12}\tilde{\mathbf{w}}\in\mathbb{R}^d.
\end{align*}
By Theorem~\ref{thm: semi-sarsa tf}, $(\mathbf{P}_{12}^\star,\bar{\mathbf{V}}_{21}^\star)$ is a global minimizer of $\mathcal{L}$ with $\mathcal{L}^\star=0$, and the scaling identity $f_{(c\mathbf{P}_{12}^\star,c^{-1}\bar{\mathbf{V}}_{21}^\star)}=f_{(\mathbf{P}_{12}^\star,\bar{\mathbf{V}}_{21}^\star)}$ for every $c>0$ implies that every point of $\mathfrak{M}_I$ is a global minimizer. Consequently,
\begin{align}\label{eq: excess-loss-as-prediction-error}
    \mathcal{L}(\boldsymbol{\theta}_{\text{eff}})-\mathcal{L}^\star=\tfrac{1}{2}\mathbb{E}_z\bigl\|f_{\boldsymbol{\theta}_{\text{eff}}}(z)-f_{\bar{\boldsymbol{\theta}}_{\text{eff}}(c)}(z)\bigr\|^2,\qquad\forall c\in I.
\end{align}
Throughout the proof, we denote $\Delta=(\mathbf{U},\mathbf{W})\in N_c\mathfrak{M}_I$ with $\|\Delta\|\leq r$, where $\|\Delta\|^2:=\|\mathbf{U}\|_F^2+\|\mathbf{W}\|_F^2$, and set $\boldsymbol{\theta}_{\text{eff}}:=\bar{\boldsymbol{\theta}}_{\text{eff}}(c)+\Delta$. Write $\|\cdot\|_{L^2}:=\sqrt{\mathbb{E}_z\|\cdot\|^2}$.

\paragraph{Step 1: Decomposition of the prediction error.}
By the definition of $f_{\boldsymbol{\theta}_{\text{eff}}}$ together with $\boldsymbol{\theta}_{\text{eff}}=(c\mathbf{P}_{12}^\star+\mathbf{U},\,c^{-1}\bar{\mathbf{V}}_{21}^\star+\mathbf{W})$,
\begin{align*}
    f_{\boldsymbol{\theta}_{\text{eff}}}(z)
    &=(c^{-1}\bar{\mathbf{V}}_{21}^\star+\mathbf{W})\hat{\boldsymbol{\Sigma}}(z)(c\mathbf{P}_{12}^\star+\mathbf{U})\tilde{\mathbf{w}},\\
    f_{\bar{\boldsymbol{\theta}}_{\text{eff}}(c)}(z)
    &=c^{-1}\bar{\mathbf{V}}_{21}^\star\hat{\boldsymbol{\Sigma}}(z)\,c\mathbf{P}_{12}^\star\tilde{\mathbf{w}}=\bar{\mathbf{V}}_{21}^\star\hat{\boldsymbol{\Sigma}}(z)\mathbf{P}_{12}^\star\tilde{\mathbf{w}}.
\end{align*}
Expanding the product and cancelling the matched leading term,
\begin{align*}
    f_{\boldsymbol{\theta}_{\text{eff}}}(z)-f_{\bar{\boldsymbol{\theta}}_{\text{eff}}(c)}(z)
    =c^{-1}\bar{\mathbf{V}}_{21}^\star\hat{\boldsymbol{\Sigma}}(z)\mathbf{U}\tilde{\mathbf{w}}+c\mathbf{W}\hat{\boldsymbol{\Sigma}}(z)\mathbf{P}_{12}^\star\tilde{\mathbf{w}}+\mathbf{W}\hat{\boldsymbol{\Sigma}}(z)\mathbf{U}\tilde{\mathbf{w}}.
\end{align*}
We now identify the first two terms with the trajectory statistics $\mathbf{R}(z)$ and $\mathbf{b}(z)$. Since $\bar{\mathbf{V}}_{21}^\star=[\alpha\mathbf{I}_d\;\mathbf{0}_{d\times d}\;\mathbf{0}_{d\times 1}]$, we have $\bar{\mathbf{V}}_{21}^\star\mathbf{x}_i=\alpha\boldsymbol{\phi}_i$, so
\begin{align*}
    \bar{\mathbf{V}}_{21}^\star\hat{\boldsymbol{\Sigma}}(z)
    =\frac{1}{n}\sum_{i=0}^{n-1}(\bar{\mathbf{V}}_{21}^\star\mathbf{x}_i)\mathbf{x}_i^\top
    =\frac{\alpha}{n}\sum_{i=0}^{n-1}\boldsymbol{\phi}_i\mathbf{x}_i^\top
    =\alpha\,\mathbf{R}(z).
\end{align*}
Similarly, $\mathbf{P}_{12}^\star\tilde{\mathbf{w}}=[-\mathbf{w};\,\mathbf{w};\,1]$, hence $\mathbf{x}_i^\top\mathbf{P}_{12}^\star\tilde{\mathbf{w}}=-\boldsymbol{\phi}_i^\top\mathbf{w}+\gamma(\boldsymbol{\phi}_i^+)^\top\mathbf{w}+r_{i+1}=\delta_i$, so
\begin{align*}
    \hat{\boldsymbol{\Sigma}}(z)\mathbf{P}_{12}^\star\tilde{\mathbf{w}}
    =\frac{1}{n}\sum_{i=0}^{n-1}\mathbf{x}_i(\mathbf{x}_i^\top\mathbf{P}_{12}^\star\tilde{\mathbf{w}})
    =\frac{1}{n}\sum_{i=0}^{n-1}\mathbf{x}_i\delta_i
    =\mathbf{b}(z).
\end{align*}
Substituting these identities,
\begin{align}\label{eq:X Y decomposition}
    f_{\boldsymbol{\theta}_{\text{eff}}}(z)-f_{\bar{\boldsymbol{\theta}}_{\text{eff}}(c)}(z)
    &=c^{-1}\alpha\,\mathbf{R}(z)\mathbf{U}\tilde{\mathbf{w}}+c\,\mathbf{W}\mathbf{b}(z)+\mathbf{W}\hat{\boldsymbol{\Sigma}}(z)\mathbf{U}\tilde{\mathbf{w}}\nonumber\\
    &=:\mathbf{X}(z)+\mathbf{Y}(z),
\end{align}
where
\begin{align*}
    \mathbf{X}(z):=c^{-1}\alpha\,\mathbf{R}(z)\mathbf{U}\tilde{\mathbf{w}}+c\,\mathbf{W}\mathbf{b}(z),\qquad\mathbf{Y}(z):=\mathbf{W}\hat{\boldsymbol{\Sigma}}(z)\mathbf{U}\tilde{\mathbf{w}}.
\end{align*}

\paragraph{Step 2: Lower bound on $\|\mathbf{X}\|_{L^2}$.}
We first lower-bound the two summands of $\mathbf{X}(z)$ separately. Expanding the squared norm and conditioning on $\tilde{\mathbf{w}}$,
\begin{align}\label{eq: lower bound term 1}
    \mathbb{E}_z\|\mathbf{R}(z)\mathbf{U}\tilde{\mathbf{w}}\|^2
    &=\mathbb{E}_z\bigl[\tilde{\mathbf{w}}^\top\mathbf{U}^\top\mathbf{R}(z)^\top\mathbf{R}(z)\,\mathbf{U}\tilde{\mathbf{w}}\bigr]\nonumber\\
    &=\mathbb{E}_z\bigl[\tilde{\mathbf{w}}^\top\mathbf{U}^\top\mathbb{E}[\mathbf{R}(z)^\top\mathbf{R}(z)\mid\tilde{\mathbf{w}}]\,\mathbf{U}\tilde{\mathbf{w}}\bigr]\nonumber\\
    &\geq\kappa_R\,\mathbb{E}_z\|\mathbf{U}\tilde{\mathbf{w}}\|^2\nonumber\\
    &=\kappa_R\,\mathrm{tr}\bigl(\mathbf{U}\,\mathbb{E}_z[\tilde{\mathbf{w}}\tilde{\mathbf{w}}^\top]\,\mathbf{U}^\top\bigr)\nonumber\\
    &\geq\kappa_R\kappa_{\tilde{w}}\|\mathbf{U}\|_F^2,
\end{align}
where the third line uses Assumption~\ref{asp:training-MDP assumptions} (3), and the last inequality uses Assumption~\ref{asp:training-MDP assumptions} (2). Similarly,
\begin{align}\label{eq: lower bound term 2}
    \mathbb{E}_z\|\mathbf{W}\mathbf{b}(z)\|^2
    &=\mathbb{E}_z\bigl[\mathbf{b}(z)^\top\mathbf{W}^\top\mathbf{W}\mathbf{b}(z)\bigr]\nonumber\\
    &=\mathbb{E}_z\bigl[\mathrm{tr}\bigl(\mathbf{W}\mathbf{b}(z)\mathbf{b}(z)^\top\mathbf{W}^\top\bigr)\bigr]\nonumber\\
    &=\mathrm{tr}\bigl(\mathbf{W}\,\mathbb{E}_z[\mathbf{b}(z)\mathbf{b}(z)^\top]\,\mathbf{W}^\top\bigr)\nonumber\\
    &\geq\kappa_q\,\mathrm{tr}(\mathbf{W}\mathbf{W}^\top)\nonumber\\
    &=\kappa_q\|\mathbf{W}\|_F^2,
\end{align}
where the fourth line uses Assumption~\ref{asp:training-MDP assumptions} (4).
Set $A:=\mathbb{E}_z\|c^{-1}\alpha\mathbf{R}(z)\mathbf{U}\tilde{\mathbf{w}}\|^2$ and $B:=\mathbb{E}_z\|c\mathbf{W}\mathbf{b}(z)\|^2$. Assumption~\ref{asp:training-MDP assumptions} (5) lower-bounds the cross term as
\begin{align*}
    2\mathbb{E}_z\langle c^{-1}\alpha\mathbf{R}(z)\mathbf{U}\tilde{\mathbf{w}},\,c\mathbf{W}\mathbf{b}(z)\rangle\geq-2\rho\sqrt{AB},
\end{align*}
so expanding $\|\mathbf{X}(z)\|^2$ and applying the cross-term bound,
\begin{align}\label{eq: lower bound X intermediate}
    \mathbb{E}_z\|\mathbf{X}(z)\|^2
    &=A+B+2\mathbb{E}_z\bigl\langle c^{-1}\alpha\mathbf{R}(z)\mathbf{U}\tilde{\mathbf{w}},\,c\mathbf{W}\mathbf{b}(z)\bigr\rangle\nonumber\\
    &\geq A+B-2\rho\sqrt{AB}\nonumber\\
    &=(1-\rho)(A+B)+\rho(\sqrt{A}-\sqrt{B})^2\nonumber\\
    &\geq(1-\rho)(A+B),
\end{align}
The two summands $A,B$ are bounded below using the bounds \eqref{eq: lower bound term 1} \eqref{eq: lower bound term 2} and $|c|\in[c_-,c_+]$:
\begin{align*}
    A&=c^{-2}\alpha^2\,\mathbb{E}_z\|\mathbf{R}(z)\mathbf{U}\tilde{\mathbf{w}}\|^2\geq c^{-2}\alpha^2\kappa_R\kappa_{\tilde{w}}\|\mathbf{U}\|_F^2\geq c_+^{-2}\alpha^2\kappa_R\kappa_{\tilde{w}}\|\mathbf{U}\|_F^2,\\
    B&=c^2\,\mathbb{E}_z\|\mathbf{W}\mathbf{b}(z)\|^2\geq c^2\kappa_q\|\mathbf{W}\|_F^2\geq c_-^2\kappa_q\|\mathbf{W}\|_F^2.
\end{align*}
Plugging into \eqref{eq: lower bound X intermediate}, with $\|\Delta\|^2=\|\mathbf{U}\|_F^2+\|\mathbf{W}\|_F^2$,
\begin{align*}
    \mathbb{E}_z\|\mathbf{X}(z)\|^2
    &\geq(1-\rho)\bigl(c_+^{-2}\alpha^2\kappa_R\kappa_{\tilde{w}}\|\mathbf{U}\|_F^2+c_-^2\kappa_q\|\mathbf{W}\|_F^2\bigr)\\
    &\geq(1-\rho)\min\bigl\{c_+^{-2}\alpha^2\kappa_R\kappa_{\tilde{w}},\,c_-^2\kappa_q\bigr\}\bigl(\|\mathbf{U}\|_F^2+\|\mathbf{W}\|_F^2\bigr)\\
    &=m_0\|\Delta\|^2.
\end{align*}
Hence
\begin{align}\label{eq: lower bound X}
    \|\mathbf{X}\|_{L^2}\geq\sqrt{m_0}\,\|\Delta\|.
\end{align}

\paragraph{Step 3: Upper bounds on $\|\mathbf{X}\|_{L^2}$ and $\|\mathbf{Y}\|_{L^2}$.}
By Assumption~\ref{asp:training-MDP assumptions} (1), 
\begin{align*}
    \|\mathbf{x}_i\|^2=\|\boldsymbol{\phi}_i\|^2+\gamma^2\|\boldsymbol{\phi}_i^+\|^2+r_{i+1}^2\leq 2B_\phi^2+B_r^2=B_\Sigma,
\end{align*}
hence, applying the triangle inequality to $\hat{\boldsymbol{\Sigma}}(z)=\tfrac{1}{n}\sum_{i=0}^{n-1}\mathbf{x}_i\mathbf{x}_i^\top$,
\begin{align*}
    \|\hat{\boldsymbol{\Sigma}}(z)\|_{\mathrm{op}}\leq\tfrac{1}{n}\sum_{i=0}^{n-1}\|\mathbf{x}_i\mathbf{x}_i^\top\|_{\mathrm{op}}=\tfrac{1}{n}\sum_{i=0}^{n-1}\|\mathbf{x}_i\|^2\leq B_\Sigma.
\end{align*}
Writing $\boldsymbol{\phi}_i=\mathbf{S}\mathbf{x}_i$ with $\mathbf{S}=[\mathbf{I}_d,\,\mathbf{0}_{d\times d},\,\mathbf{0}_{d\times 1}]\in\mathbb{R}^{d\times(2d+1)}$, we have 
\begin{align*}
    \mathbf{R}(z)=\tfrac{1}{n}\sum_i\mathbf{S}\mathbf{x}_i\mathbf{x}_i^\top=\mathbf{S}\hat{\boldsymbol{\Sigma}}(z),
\end{align*}
and $\mathbf{S}\mathbf{S}^\top=\mathbf{I}_d$ gives $\|\mathbf{S}\|_{\mathrm{op}}=1$. Submultiplicativity then yields
\begin{align*}
    \|\mathbf{R}(z)\|_{\mathrm{op}}&=\|\mathbf{S}\hat{\boldsymbol{\Sigma}}(z)\|_{\mathrm{op}}\\
    &\leq\|\mathbf{S}\|_{\mathrm{op}}\,\|\hat{\boldsymbol{\Sigma}}(z)\|_{\mathrm{op}}\\
    &\leq B_\Sigma.
\end{align*} 
Also $\|\tilde{\mathbf{w}}\|\leq B_{\tilde{w}}$. Since $\mathbf{P}_{12}^{\star\top}\mathbf{P}_{12}^\star=\mathrm{diag}(1,\,2\mathbf{I}_d)$, $\|\mathbf{P}_{12}^\star\|_{\mathrm{op}}=\sqrt{2}$, hence 
\begin{align*}
    \|\mathbf{b}(z)\|&\leq\|\mathbf{P}_{12}^\star\|_{\mathrm{op}}\|\hat{\boldsymbol{\Sigma}}(z)\|_{\mathrm{op}}\|\tilde{\mathbf{w}}\|\\
    &\leq\sqrt{2}\,B_\Sigma B_{\tilde{w}}.
\end{align*} 
By the triangle inequality and submultiplicativity, 
\begin{align*}
    \|\mathbf{X}\|_{L^2}
    &\leq\bigl\|c^{-1}\alpha\mathbf{R}(z)\mathbf{U}\tilde{\mathbf{w}}\bigr\|_{L^2}+\bigl\|c\,\mathbf{W}\mathbf{b}(z)\bigr\|_{L^2}\\
    &\leq|c|^{-1}\alpha\,\|\mathbf{R}(z)\|_{\mathrm{op}}\|\mathbf{U}\|_{\mathrm{op}}\|\tilde{\mathbf{w}}\|+|c|\,\|\mathbf{W}\|_{\mathrm{op}}\|\mathbf{b}(z)\|\\
    &\leq c_-^{-1}\alpha B_\Sigma B_{\tilde{w}}\|\mathbf{U}\|_F+\sqrt{2}\,c_+B_\Sigma B_{\tilde{w}}\|\mathbf{W}\|_F.
\end{align*}
Applying Cauchy--Schwarz to the right-hand side,
\begin{align}\label{eq: upper bound X}
    \|\mathbf{X}\|_{L^2}&\leq B_\Sigma B_{\tilde{w}}\sqrt{c_-^{-2}\alpha^2+2c_+^2}\,\sqrt{\|\mathbf{U}\|_F^2+\|\mathbf{W}\|_F^2}\\
    &=\sqrt{M_0}\,\|\Delta\|.
\end{align}
For the quadratic term, 
\begin{align*}
    \|\mathbf{Y}(z)\|&=\|\mathbf{W}\hat{\boldsymbol{\Sigma}}(z)\mathbf{U}\tilde{\mathbf{w}}\|\\
    &\leq\|\mathbf{W}\|_{\mathrm{op}}\|\hat{\boldsymbol{\Sigma}}(z)\|_{\mathrm{op}}\|\mathbf{U}\|_{\mathrm{op}}\|\tilde{\mathbf{w}}\|\\
    &\leq B_\Sigma B_{\tilde{w}}\|\mathbf{U}\|_F\|\mathbf{W}\|_F\\
    &\leq B_\Sigma B_{\tilde{w}}\bigl(\tfrac{1}{2}\|\mathbf{U}\|_F^2+\tfrac{1}{2}\|\mathbf{W}\|_F^2\bigr)\\
    &=C_Q\|\Delta\|^2.
\end{align*}
Thus
\begin{align}\label{eq: upper bound Y}
    \|\mathbf{Y}\|_{L^2}\leq C_Q\|\Delta\|^2.
\end{align}

\paragraph{Step 4: Excess-loss upper bound.}
By \eqref{eq: excess-loss-as-prediction-error}, \eqref{eq:X Y decomposition}, and \eqref{eq: upper bound X}--\eqref{eq: upper bound Y},
\begin{align}\label{eq: loss-upper-bound}
    \mathcal{L}(\boldsymbol{\theta}_{\text{eff}})-\mathcal{L}^\star
    &=\tfrac{1}{2}\|\mathbf{X}+\mathbf{Y}\|_{L^2}^2\nonumber\\
    &\leq\tfrac{1}{2}\bigl(\sqrt{M_0}\,\|\Delta\|+C_Q\|\Delta\|^2\bigr)^2\nonumber\\
    &\leq\tfrac{1}{2}\bigl(\sqrt{M_0}+C_Q r\bigr)^2\|\Delta\|^2,
\end{align}
where the last line uses $\|\Delta\|\leq r$.

\paragraph{Step 5: Directional derivative lower bound.}
Define 
\begin{align*}
    g(t):=\mathcal{L}(\bar{\boldsymbol{\theta}}_{\text{eff}}(c)+t\Delta). 
\end{align*}
By the chain rule, 
\begin{align*}
    g'(1)=\langle\nabla\mathcal{L}(\boldsymbol{\theta}_{\text{eff}}),\Delta\rangle, 
\end{align*}
so Cauchy--Schwarz gives
\begin{align}\label{eq: directional-derivative-lower-bound}
    \|\nabla\mathcal{L}(\boldsymbol{\theta}_{\text{eff}})\|\,\|\Delta\|\geq g'(1).
\end{align}
Thus, to establish the PL inequality, it suffices to lower-bound $g'(1)$ with respect to $\|\Delta\|^2$. We now compute $g'(1)$.
Replacing $\Delta=(\mathbf{U},\mathbf{W})$ by $t\Delta=(t\mathbf{U},\,t\mathbf{W})$ in the decomposition~\eqref{eq:X Y decomposition},
\begin{align*}
    f_{\bar{\boldsymbol{\theta}}_{\text{eff}}(c)+t\Delta}(z)-f_{\bar{\boldsymbol{\theta}}_{\text{eff}}(c)}(z)
    &=c^{-1}\alpha\,\mathbf{R}(z)(t\mathbf{U})\tilde{\mathbf{w}}+c\,(t\mathbf{W})\mathbf{b}(z)+(t\mathbf{W})\hat{\boldsymbol{\Sigma}}(z)(t\mathbf{U})\tilde{\mathbf{w}}\\
    &=t\bigl[c^{-1}\alpha\,\mathbf{R}(z)\mathbf{U}\tilde{\mathbf{w}}+c\,\mathbf{W}\mathbf{b}(z)\bigr]+t^2\,\mathbf{W}\hat{\boldsymbol{\Sigma}}(z)\mathbf{U}\tilde{\mathbf{w}}\\
    &=t\mathbf{X}(z)+t^2\mathbf{Y}(z),
\end{align*}
hence
\begin{align*}
    g(t)=\tfrac{1}{2}\mathbb{E}_z\|t\mathbf{X}(z)+t^2\mathbf{Y}(z)\|^2.
\end{align*}
Differentiating and evaluating at $t=1$,
\begin{align*}
    g'(1)&=\|\mathbf{X}\|_{L^2}^2+3\mathbb{E}_z\langle\mathbf{X}(z),\mathbf{Y}(z)\rangle+2\|\mathbf{Y}\|_{L^2}^2\\
    &\geq\|\mathbf{X}\|_{L^2}^2-3\|\mathbf{X}\|_{L^2}\|\mathbf{Y}\|_{L^2}.
\end{align*}
Set $u:=\|\mathbf{X}\|_{L^2}$ and $v:=\|\mathbf{Y}\|_{L^2}$. The hypothesis $r<\sqrt{m_0}/(3C_Q)$ together with \eqref{eq: lower bound X} and \eqref{eq: upper bound Y} gives
\begin{align*}
    u\geq\sqrt{m_0}\,\|\Delta\|>3C_Q\|\Delta\|^2\geq 3v,
\end{align*}
so $u\mapsto u^2-3uv$ is increasing at $u$. Substituting the lower bound on $\mathbf{u}$ \eqref{eq: lower bound X} and upper bound on $\mathbf{v}$ \eqref{eq: upper bound Y},
\begin{align}\label{eq: lower bound g(1)}
    g'(1)&\geq u^2-3uv\nonumber\\
    &\geq m_0\|\Delta\|^2-3C_Q\sqrt{m_0}\,\|\Delta\|^3\nonumber\\
    &\geq m_0\|\Delta\|^2-3C_Q\sqrt{m_0}r\,\|\Delta\|^2\nonumber\\
    &=\bigl(m_0-3C_Q\sqrt{m_0}\,r\bigr)\|\Delta\|^2.
\end{align}

\paragraph{Step 6: PL inequality.}
Combining~\eqref{eq: directional-derivative-lower-bound} with~\eqref{eq: lower bound g(1)},
\begin{align*}
    \|\nabla\mathcal{L}(\boldsymbol{\theta}_{\text{eff}})\|\,\|\Delta\|\geq g'(1)\geq\bigl(m_0-3C_Q\sqrt{m_0}\,r\bigr)\|\Delta\|^2.
\end{align*}
Squaring (and using \eqref{eq: loss-upper-bound} to convert $\|\Delta\|^2$ into the excess loss),
\begin{align*}
    \tfrac{1}{2}\|\nabla\mathcal{L}(\boldsymbol{\theta}_{\text{eff}})\|^2
    \geq\frac{\bigl(m_0-3C_Q\sqrt{m_0}\,r\bigr)^2}{\bigl(\sqrt{M_0}+C_Q r\bigr)^2}\bigl(\mathcal{L}(\boldsymbol{\theta}_{\text{eff}})-\mathcal{L}^\star\bigr)=\mu_r\bigl(\mathcal{L}(\boldsymbol{\theta}_{\text{eff}})-\mathcal{L}^\star\bigr),
\end{align*}
which is the claimed PL inequality on $\mathcal{T}_r(I)$.
\end{proof}
\section{Proof of Theorem \ref{thm: semi-sarsa tf}}\label{proof of thm: semi-sarsa tf}
\begin{proof}
We construct $\boldsymbol{\theta}^\star=(\mathbf{P}^\star,\mathbf{V}^\star)$ explicitly in four steps: (i) reduce the readout to four blocks of $\boldsymbol{\theta}$; (ii) cancel the cubic-in-$\mathbf{w}$ contribution; (iii) match the affine contribution to the SARSA target; (iv) assemble $\boldsymbol{\theta}^\star$ and verify its scaling invariance.

\paragraph{Step 1: Reduction to $\bar{\mathbf{V}}_{21},\bar{\mathbf{V}}_{22},\mathbf{P}_{12},\mathbf{P}_{22}$.}
Partition $\mathbf{V},\mathbf{P}$ as in \eqref{eq:block partition}, conformal with $D=(2d+1)+(d+1)$, and write $\tilde{\mathbf{w}}=[1;\mathbf{w}]\in\mathbb{R}^{d+1}$. Let $\bar{\mathbf{V}}_{21},\bar{\mathbf{V}}_{22}$ denote the last $d$ rows of $\mathbf{V}_{21},\mathbf{V}_{22}$, respectively. By Lemma~\ref{lem: transformer output}, the transformer update increment simplifies to
\begin{equation}\label{eq: dw_tf}
    \Delta\mathbf{w}_{\text{TF}}(\boldsymbol{\theta})
    \;:=\;\textsf{TF}_{\boldsymbol{\theta}}(\mathbf{H})-\mathbf{w}
    \;=\;\bar{\mathbf{V}}_{21}\hat{\boldsymbol{\Sigma}}\mathbf{P}_{12}\tilde{\mathbf{w}}+\tfrac{1}{n}\bar{\mathbf{V}}_{22}\tilde{\mathbf{w}}\tilde{\mathbf{w}}^\top\mathbf{P}_{22}\tilde{\mathbf{w}},
\end{equation}
where $\hat{\boldsymbol{\Sigma}}\in\mathbb{R}^{(2d+1)\times(2d+1)}$ is the empirical second-moment matrix of the trajectory features, with the block decomposition
\begin{align*}
    \hat{\boldsymbol{\Sigma}}&=\begin{bmatrix}
        \frac{\sum_{i=0}^{n-1}\boldsymbol{\phi}(s_i,a_i)\boldsymbol{\phi}(s_i,a_i)^\top}{n} & \frac{\sum_{i=0}^{n-1}\gamma\boldsymbol{\phi}(s_i,a_i)\boldsymbol{\phi}(s_{i+1},a_{i+1})^\top}{n} & \frac{\sum_{i=0}^{n-1} r_{i+1}\boldsymbol{\phi}(s_i,a_i)}{n}\\
        \frac{\sum_{i=0}^{n-1}\gamma\boldsymbol{\phi}(s_{i+1},a_{i+1})\boldsymbol{\phi}(s_i,a_i)^\top}{n} & \frac{\sum_{i=0}^{n-1}\gamma^2\boldsymbol{\phi}(s_{i+1},a_{i+1})\boldsymbol{\phi}(s_{i+1},a_{i+1})^\top}{n} & \frac{\sum_{i=0}^{n-1}\gamma r_{i+1}\boldsymbol{\phi}(s_{i+1},a_{i+1})}{n}\\
        \frac{\sum_{i=0}^{n-1} r_{i+1}\boldsymbol{\phi}(s_i,a_i)^\top}{n} & \frac{\sum_{i=0}^{n-1}\gamma r_{i+1}\boldsymbol{\phi}(s_{i+1},a_{i+1})^\top}{n} & \frac{\sum_{i=0}^{n-1} r_{i+1}^2}{n}
    \end{bmatrix}\\
    &:=\begin{bmatrix}
        \hat{\boldsymbol{\Sigma}}_{\phi\phi} & \gamma\hat{\boldsymbol{\Sigma}}_{\phi\phi_+} & \hat{\boldsymbol{\Sigma}}_{\phi r}\\
        \gamma\hat{\boldsymbol{\Sigma}}_{\phi_+\phi} & \gamma^2\hat{\boldsymbol{\Sigma}}_{\phi_+\phi_+} & \gamma\hat{\boldsymbol{\Sigma}}_{\phi_+ r}\\
        \hat{\boldsymbol{\Sigma}}_{r\phi} & \gamma\hat{\boldsymbol{\Sigma}}_{r\phi_+} & \hat{\boldsymbol{\Sigma}}_{rr}
    \end{bmatrix}.
\end{align*}
In particular, the blocks $\mathbf{V}_{11},\mathbf{V}_{12},\mathbf{P}_{11},\mathbf{P}_{21}$ and the first rows of $\mathbf{V}_{21},\mathbf{V}_{22}$ do not enter the readout, so they are unconstrained.

\paragraph{Step 2: Cancel the cubic term.}
The batch SARSA update increment \eqref{eq:semi-sarsa linear batch} can be written compactly as
\begin{equation}\label{eq: dw_sarsa}
    \Delta\mathbf{w}_{\text{SARSA}}=\alpha\bigl[\hat{\boldsymbol{\Sigma}}_{\phi r}+\gamma\hat{\boldsymbol{\Sigma}}_{\phi\phi_+}\mathbf{w}-\hat{\boldsymbol{\Sigma}}_{\phi\phi}\mathbf{w}\bigr],
\end{equation}
which is affine in $\mathbf{w}$. The first term in \eqref{eq: dw_tf} is also affine via $\tilde{\mathbf{w}}=[1;\mathbf{w}]$, while the second is cubic in $\mathbf{w}$. Enforcing $\Delta\mathbf{w}_{\text{TF}}\equiv\Delta\mathbf{w}_{\text{SARSA}}$ for every $\mathbf{w}$ therefore requires the cubic term to vanish, for which we set
\begin{equation*}
    \bar{\mathbf{V}}_{22}^\star=\mathbf{0},\quad \mathbf{P}_{22}^\star=\mathbf{0}.
\end{equation*}

\paragraph{Step 3: Match the affine term.}
It remains to choose $\bar{\mathbf{V}}_{21}^\star,\mathbf{P}_{12}^\star$ so that
\begin{equation*}
    \bar{\mathbf{V}}_{21}^\star\hat{\boldsymbol{\Sigma}}\mathbf{P}_{12}^\star\tilde{\mathbf{w}}=\alpha\bigl[\hat{\boldsymbol{\Sigma}}_{\phi r}+\gamma\hat{\boldsymbol{\Sigma}}_{\phi\phi_+}\mathbf{w}-\hat{\boldsymbol{\Sigma}}_{\phi\phi}\mathbf{w}\bigr].
\end{equation*}
Take $\bar{\mathbf{V}}_{21}^\star$ to extract the first $d$ rows of $\hat{\boldsymbol{\Sigma}}$ with prefactor $\alpha$:
\begin{equation*}
    \bar{\mathbf{V}}_{21}^\star=\bigl[\alpha\mathbf{I}_d,\ \mathbf{0}_{d\times d},\ \mathbf{0}_{d\times 1}\bigr]
    \quad\Rightarrow\quad
    \bar{\mathbf{V}}_{21}^\star\hat{\boldsymbol{\Sigma}}=\alpha\bigl[\hat{\boldsymbol{\Sigma}}_{\phi\phi},\ \gamma\hat{\boldsymbol{\Sigma}}_{\phi\phi_+},\ \hat{\boldsymbol{\Sigma}}_{\phi r}\bigr].
\end{equation*}
The matching condition then reduces to $\mathbf{P}_{12}^\star\tilde{\mathbf{w}}=[-\mathbf{w};\,\mathbf{w};\,1]$, which holds for every $\mathbf{w}\in\mathbb{R}^d$ when
\begin{equation*}
    \mathbf{P}_{12}^\star=\begin{bmatrix}\mathbf{0}_{d\times 1} & -\mathbf{I}_d\\ \mathbf{0}_{d\times 1} & \mathbf{I}_d\\ 1 & \mathbf{0}_{1\times d}\end{bmatrix}.
\end{equation*}

\paragraph{Step 4: Assembly and scaling.}
Leave the unconstrained blocks identified in Step~1 as arbitrary entries (denoted $*$). This yields the structure $(\mathbf{P}^\star,\mathbf{V}^\star)$ in \eqref{eq:QKV}, and by construction
\begin{equation*}
    \textsf{TF}_{\boldsymbol{\theta}^\star}(\mathbf{H})=\mathbf{w}+\Delta\mathbf{w}_{\text{SARSA}}.
\end{equation*}
Finally, both terms in \eqref{eq: dw_tf} are bilinear in $(\mathbf{V},\mathbf{P})$, so the rescaled pair $(c\mathbf{P}^\star,c^{-1}\mathbf{V}^\star)$ produces the same readout for every $c\neq 0$, completing the construction.
\end{proof}
\section{Proof of Proposition \ref{prop:invariant-set}}\label{proof of prop:invariant-set}
\begin{proof}
Let $\boldsymbol{\delta}(\boldsymbol{\theta},z):=\textsf{TF}_{\boldsymbol{\theta}}(\mathbf{H}(z))-\mathbf{w}_{\text{SARSA}}(z)$ denote the per-sample residual, so that 
\begin{align*}
    \mathcal{L}(\boldsymbol{\theta})=\tfrac{1}{2}\mathbb{E}_z\|\boldsymbol{\delta}(\boldsymbol{\theta},z)\|^2.
\end{align*} 
By Lemma~\ref{lem: transformer output},
\begin{align}\label{eq: delta decomp}
    \boldsymbol{\delta}(\boldsymbol{\theta},z)=\bar{\mathbf{V}}_{21}\hat{\boldsymbol{\Sigma}}(z)\mathbf{P}_{12}\tilde{\mathbf{w}}(z)+\tfrac{1}{n}\bar{\mathbf{V}}_{22}\tilde{\mathbf{w}}(z)\tilde{\mathbf{w}}(z)^{\!\top}\mathbf{P}_{22}\tilde{\mathbf{w}}(z)+\bigl(\mathbf{w}(z)-\mathbf{w}_{\text{SARSA}}(z)\bigr),
\end{align}
so $\mathcal{L}$ depends on $(\mathbf{V},\mathbf{P})$ only through the four sub-blocks $\bar{\mathbf{V}}_{21},\bar{\mathbf{V}}_{22},\mathbf{P}_{12},\mathbf{P}_{22}$.

\paragraph{Part (i).}
Since $\mathbf{V}_{11},\mathbf{V}_{12},\mathbf{P}_{11},\mathbf{P}_{21}$ do not appear in~\eqref{eq: delta decomp},
\begin{align*}
    \nabla_{\mathbf{V}_{11}}\mathcal{L}=\nabla_{\mathbf{V}_{12}}\mathcal{L}=\nabla_{\mathbf{P}_{11}}\mathcal{L}=\nabla_{\mathbf{P}_{21}}\mathcal{L}=\mathbf{0},
\end{align*}
so each is stationary under the gradient flow~\eqref{eq: gradient flow}. For $\mathbf{V}_{21}$, the loss factors through $\bar{\mathbf{V}}_{21}=\mathbf{S}\mathbf{V}_{21}$ with $\mathbf{S}=[\mathbf{0}_{d\times 1},\mathbf{I}_d]$, and the chain rule yields \begin{align*}
    \nabla_{\mathbf{V}_{21}}\mathcal{L}=\mathbf{S}^{\!\top}\nabla_{\bar{\mathbf{V}}_{21}}\mathcal{L}.
\end{align*} 
Since $\mathbf{S}^{\!\top}=[\mathbf{0}_{1\times d};\mathbf{I}_d]$ has zero first row, so does $\nabla_{\mathbf{V}_{21}}\mathcal{L}$, and the first row of $\mathbf{V}_{21}$ remains unchanged. The same argument applies to $\mathbf{V}_{22}$.

\paragraph{Part (ii).}
A direct computation from~\eqref{eq: delta decomp} gives
\begin{align}
    \nabla_{\bar{\mathbf{V}}_{22}}\mathcal{L}(\boldsymbol{\theta})&=\tfrac{1}{n}\,\mathbb{E}_z\!\left[\boldsymbol{\delta}(\boldsymbol{\theta},z)\bigl(\tilde{\mathbf{w}}\tilde{\mathbf{w}}^{\!\top}\mathbf{P}_{22}\tilde{\mathbf{w}}\bigr)^{\!\top}\right],\label{eq: grad V22bar}\\
    \nabla_{\mathbf{P}_{22}}\mathcal{L}(\boldsymbol{\theta})&=\tfrac{1}{n}\,\mathbb{E}_z\!\left[\bigl(\boldsymbol{\delta}(\boldsymbol{\theta},z)^{\!\top}\bar{\mathbf{V}}_{22}\tilde{\mathbf{w}}\bigr)\,\tilde{\mathbf{w}}\tilde{\mathbf{w}}^{\!\top}\right].\label{eq: grad P22}
\end{align}
Both~\eqref{eq: grad V22bar} and~\eqref{eq: grad P22} vanish whenever $\mathbf{P}_{22}=\mathbf{0}$ and $\bar{\mathbf{V}}_{22}=\mathbf{0}$. Hence the gradient flow~\eqref{eq: gradient flow} admits a trajectory in which $\mathbf{P}_{22}^{(t)}\equiv\mathbf{0}$ and $\bar{\mathbf{V}}_{22}^{(t)}\equiv\mathbf{0}$, with the remaining blocks evolving freely. 
\end{proof}
\section{Proof of Theorem \ref{thm:local-convergence-manifold}}\label{proof of thm:local-convergence-manifold}
\begin{proof}
We recall the local PL inequality \eqref{eq: local pl} from Lemma~\ref{lem:local-pl-ic-sarsa}: for every $\boldsymbol{\theta}_{\text{eff}}\in\mathcal{T}_r(I)$,
\begin{align*}
    \frac{1}{2}\|\nabla\mathcal{L}(\boldsymbol{\theta}_{\text{eff}})\|^2
    \geq\mu_r\bigl(\mathcal{L}(\boldsymbol{\theta}_{\text{eff}})-\mathcal{L}^\star\bigr).
\end{align*}
We keep the $\mathcal{L}^\star$ term explicit for clarity, but it can be dropped since $\mathcal{L}^\star=0$.

\paragraph{Step 1: Bounding the parameter distance to the optimum manifold.}
Fix any $\boldsymbol{\theta}_{\text{eff}}\in\mathcal{T}_r(I)$ and write its normal decomposition as
\begin{align*}
    \boldsymbol{\theta}_{\text{eff}}=\bar{\boldsymbol{\theta}}_{\text{eff}}(c)+\Delta,
    \quad
    \Delta=(\mathbf{U},\mathbf{W})\in N_c\mathfrak{M}_I,
    \quad
    \|\Delta\|<r.
\end{align*}
As in the proof of Lemma~\ref{lem:local-pl-ic-sarsa}, the residual decomposes as
\begin{align*}
    f_{\boldsymbol{\theta}_{\text{eff}}}(z)-f_{\bar{\boldsymbol{\theta}}_{\text{eff}}(c)}(z)
    =\mathbf{X}(z)+\mathbf{Y}(z),
\end{align*}
where, by~\eqref{eq:X Y decomposition},
\begin{align*}
    \mathbf{X}(z)
    &:=c^{-1}\alpha\,\mathbf{R}(z)\mathbf{U}\tilde{\mathbf{w}}+c\,\mathbf{W}\mathbf{b}(z),\\
    \mathbf{Y}(z)
    &:=\mathbf{W}\hat{\boldsymbol{\Sigma}}(z)\mathbf{U}\tilde{\mathbf{w}}.
\end{align*}
By~\eqref{eq: lower bound X},
\begin{align*}
    \|\mathbf{X}\|_{L^2}\geq\sqrt{m_0}\,\|\Delta\|,
\end{align*}
and by~\eqref{eq: upper bound Y},
\begin{align*}
    \|\mathbf{Y}\|_{L^2}\leq C_Q\,\|\Delta\|^2.
\end{align*}
Therefore, applying the reverse triangle inequality,
\begin{align*}
    \sqrt{2\bigl(\mathcal{L}(\boldsymbol{\theta}_{\text{eff}})-\mathcal{L}^\star\bigr)}
    &=\|\mathbf{X}+\mathbf{Y}\|_{L^2}\\
    &\geq\|\mathbf{X}\|_{L^2}-\|\mathbf{Y}\|_{L^2}\\
    &\geq\bigl(\sqrt{m_0}-C_Q\|\Delta\|\bigr)\|\Delta\|\\
    &\geq\bigl(\sqrt{m_0}-C_Q r\bigr)\|\Delta\|,
\end{align*}
where the last inequality uses $\|\Delta\|<r$. Squaring both sides,
\begin{align}\label{eq: excess loss lower bound}
    \mathcal{L}(\boldsymbol{\theta}_{\text{eff}})-\mathcal{L}^\star
    \geq\frac{1}{2}\bigl(\sqrt{m_0}-C_Q r\bigr)^2\|\Delta\|^2
    =\lambda_r\,\|\Delta\|^2,
\end{align}
where we set
\begin{align*}
    \lambda_r:=\tfrac{1}{2}\bigl(\sqrt{m_0}-C_Q r\bigr)^2.
\end{align*}
By definition of the tube~\eqref{def:tube}, every point $\boldsymbol{\theta}_{\text{eff}}\in\mathcal{T}_r(I)$ admits a unique normal decomposition
\begin{align*}
    \boldsymbol{\theta}_{\text{eff}}=\bar{\boldsymbol{\theta}}_{\text{eff}}(c)+\Delta,
    \qquad\Delta\in N_c\mathfrak{M}_I,
\end{align*}
so the normal displacement equals the distance to the manifold patch:
\begin{align*}
    \|\Delta\|=\mathrm{dist}(\boldsymbol{\theta}_{\text{eff}},\mathfrak{M}_I).
\end{align*}
Combining this identity with~\eqref{eq: excess loss lower bound}, for all $\boldsymbol{\theta}_{\text{eff}}\in\mathcal{T}_r(I)$,
\begin{align*}
    \mathcal{L}(\boldsymbol{\theta}_{\text{eff}})-\mathcal{L}^\star
    \geq\lambda_r\,\mathrm{dist}(\boldsymbol{\theta}_{\text{eff}},\mathfrak{M}_I)^2,
\end{align*}
and hence
\begin{align}\label{eq: distance bound}
    \mathrm{dist}(\boldsymbol{\theta}_{\text{eff}},\mathfrak{M}_I)
    \leq\frac{1}{\sqrt{\lambda_r}}\sqrt{\mathcal{L}(\boldsymbol{\theta}_{\text{eff}})-\mathcal{L}^\star}.
\end{align}

\paragraph{Step 2: Local upper bound on the gradient norm.}
Define the prediction error
\begin{align*}
    \mathbf{e}_{\boldsymbol{\theta}_{\text{eff}}}(z)
    &:=f_{\mathbf{P}_{12},\bar{\mathbf{V}}_{21}}(z)
    -f_{\mathbf{P}_{12}^\star,\bar{\mathbf{V}}_{21}^\star}(z)\\
    &=\bar{\mathbf{V}}_{21}\hat{\boldsymbol{\Sigma}}(z)\mathbf{P}_{12}\tilde{\mathbf{w}}
    -\bar{\mathbf{V}}_{21}^\star\hat{\boldsymbol{\Sigma}}(z)\mathbf{P}_{12}^\star\tilde{\mathbf{w}}.
\end{align*}
Then
\begin{align*}
    \mathcal{L}(\boldsymbol{\theta}_{\text{eff}})-\mathcal{L}^\star
    =\frac{1}{2}\,\mathbb{E}_z\|\mathbf{e}_{\boldsymbol{\theta}_{\text{eff}}}(z)\|^2,
\end{align*}
so that
\begin{align*}
    \mathbb{E}_z\|\mathbf{e}_{\boldsymbol{\theta}_{\text{eff}}}(z)\|^2
    =2\bigl(\mathcal{L}(\boldsymbol{\theta}_{\text{eff}})-\mathcal{L}^\star\bigr).
\end{align*}
Differentiating the bilinear predictor with respect to $\bar{\mathbf{V}}_{21}$ and $\mathbf{P}_{12}$ gives
\begin{align*}
    \nabla_{\bar{\mathbf{V}}_{21}}\mathcal{L}(\boldsymbol{\theta}_{\text{eff}})
    &=\mathbb{E}_z\!\left[\mathbf{e}_{\boldsymbol{\theta}_{\text{eff}}}(z)\,
    \bigl(\hat{\boldsymbol{\Sigma}}(z)\mathbf{P}_{12}\tilde{\mathbf{w}}\bigr)^\top\right],\\
    \nabla_{\mathbf{P}_{12}}\mathcal{L}(\boldsymbol{\theta}_{\text{eff}})
    &=\mathbb{E}_z\!\left[\hat{\boldsymbol{\Sigma}}(z)^\top\,\bar{\mathbf{V}}_{21}^\top\,
    \mathbf{e}_{\boldsymbol{\theta}_{\text{eff}}}(z)\,\tilde{\mathbf{w}}^\top\right].
\end{align*}
For every $\boldsymbol{\theta}_{\text{eff}}\in\mathcal{T}_r(I)$, the triangle inequality together with $\|\Delta\|<r$ yields
\begin{align*}
    \|\mathbf{P}_{12}\|_F
    &\leq c_+\|\mathbf{P}_{12}^\star\|_F+r\\
    &=c_+\sqrt{2d+1}+r\\
    &=:P_{\max,r},
\end{align*}
and
\begin{align*}
    \|\bar{\mathbf{V}}_{21}\|_F
    &\leq c_-^{-1}\|\bar{\mathbf{V}}_{21}^\star\|_F+r\\
    &=c_-^{-1}\alpha\sqrt{d}+r\\
    &=:V_{\max,r}.
\end{align*}
Moreover, by Assumption~\ref{asp:training-MDP assumptions}~(1),
\begin{align*}
    \|\hat{\boldsymbol{\Sigma}}(z)\|_{\mathrm{op}}\leq B_\Sigma,
\end{align*}
and
\begin{align*}
    \|\tilde{\mathbf{w}}\|\leq B_{\tilde{w}}.
\end{align*}
Therefore, by Cauchy--Schwarz and submultiplicativity,
\begin{align*}
    \|\nabla_{\bar{\mathbf{V}}_{21}}\mathcal{L}(\boldsymbol{\theta}_{\text{eff}})\|_F
    &\leq B_\Sigma B_{\tilde{w}}\,P_{\max,r}\,
    \bigl(\mathbb{E}_z\|\mathbf{e}_{\boldsymbol{\theta}_{\text{eff}}}(z)\|^2\bigr)^{1/2}\\
    &=\sqrt{2}\,B_\Sigma B_{\tilde{w}}\,P_{\max,r}\,
    \sqrt{\mathcal{L}(\boldsymbol{\theta}_{\text{eff}})-\mathcal{L}^\star},
\end{align*}
and similarly,
\begin{align*}
    \|\nabla_{\mathbf{P}_{12}}\mathcal{L}(\boldsymbol{\theta}_{\text{eff}})\|_F
    &\leq\sqrt{2}\,B_\Sigma B_{\tilde{w}}\,V_{\max,r}\,
    \sqrt{\mathcal{L}(\boldsymbol{\theta}_{\text{eff}})-\mathcal{L}^\star}.
\end{align*}
Combining the two estimates and recalling that
\begin{align*}
    K_r=\sqrt{2}\,B_\Sigma B_{\tilde{w}}\sqrt{P_{\max,r}^2+V_{\max,r}^2},
\end{align*}
we obtain the gradient upper bound
\begin{align}\label{eq: gradient upper bound}
    \|\nabla\mathcal{L}(\boldsymbol{\theta}_{\text{eff}})\|
    \leq K_r\,\sqrt{\mathcal{L}(\boldsymbol{\theta}_{\text{eff}})-\mathcal{L}^\star},
    \qquad\forall\,\boldsymbol{\theta}_{\text{eff}}\in\mathcal{T}_r(I).
\end{align}

\paragraph{Step 3: Invariance of the tube.}
Let $\boldsymbol{\theta}_{\text{eff}}(t)$ denote the maximal solution of the gradient flow~\eqref{eq: gradient flow} starting from $\boldsymbol{\theta}_{\text{eff}}(0)$. Since $\nabla\mathcal{L}$ is locally Lipschitz, a unique local solution exists. Define the exit time
\begin{align*}
    \tau:=\inf\bigl\{t>0:\boldsymbol{\theta}_{\text{eff}}(t)\notin\mathcal{T}_r(I)\bigr\}.
\end{align*}
We will show that $\tau=\infty$. For all $t<\tau$, the local PL inequality~\eqref{eq: local pl} holds, and along the gradient flow,
\begin{align*}
    \frac{d}{dt}\mathcal{L}(\boldsymbol{\theta}_{\text{eff}}(t))
    &=\bigl\langle\nabla\mathcal{L}(\boldsymbol{\theta}_{\text{eff}}(t)),
    \dot{\boldsymbol{\theta}}_{\text{eff}}(t)\bigr\rangle\\
    &=-\|\nabla\mathcal{L}(\boldsymbol{\theta}_{\text{eff}}(t))\|^2\\
    &\leq-2\mu_r\bigl(\mathcal{L}(\boldsymbol{\theta}_{\text{eff}}(t))-\mathcal{L}^\star\bigr).
\end{align*}
Since $\mathcal{L}^\star=0$, Gronwall's inequality yields the loss decay bound
\begin{align}\label{eq: loss decay}
    \mathcal{L}(\boldsymbol{\theta}_{\text{eff}}(t))-\mathcal{L}^\star
    \leq e^{-2\mu_r t}\bigl(\mathcal{L}(\boldsymbol{\theta}_{\text{eff}}(0))-\mathcal{L}^\star\bigr),
    \qquad\forall\,t<\tau.
\end{align}
Combining~\eqref{eq: gradient upper bound} with~\eqref{eq: loss decay}, for every $t<\tau$,
\begin{align*}
    \|\dot{\boldsymbol{\theta}}_{\text{eff}}(t)\|
    &=\|\nabla\mathcal{L}(\boldsymbol{\theta}_{\text{eff}}(t))\|\\
    &\leq K_r\,\sqrt{\mathcal{L}(\boldsymbol{\theta}_{\text{eff}}(t))-\mathcal{L}^\star}\\
    &\leq K_r\,\sqrt{\mathcal{L}(\boldsymbol{\theta}_{\text{eff}}(0))-\mathcal{L}^\star}\,e^{-\mu_r t}.
\end{align*}
Integrating, for every $t<\tau$,
\begin{align}\label{eq: parameter upper bound along trajectory}
    \|\boldsymbol{\theta}_{\text{eff}}(t)-\boldsymbol{\theta}_{\text{eff}}(0)\|
    &\leq\int_0^t\|\dot{\boldsymbol{\theta}}_{\text{eff}}(s)\|\,ds\nonumber\\
    &\leq K_r\,\sqrt{\mathcal{L}(\boldsymbol{\theta}_{\text{eff}}(0))-\mathcal{L}^\star}
    \int_0^t e^{-\mu_r s}\,ds\nonumber\\
    &\leq\frac{K_r}{\mu_r}\,\sqrt{\mathcal{L}(\boldsymbol{\theta}_{\text{eff}}(0))-\mathcal{L}^\star}.
\end{align}
We next upper-bound the initial excess loss in terms of the distance to $\mathfrak{M}_{I_0}$. From~\eqref{eq: loss-upper-bound}, since $\|\Delta\|=\mathrm{dist}(\boldsymbol{\theta}_{\text{eff}},\mathfrak{M}_I)$ and $\mathfrak{M}_{I_0}\subset\mathfrak{M}_I$,
\begin{align}\label{eq: excess loss upper bound 2}
    \mathcal{L}(\boldsymbol{\theta}_{\text{eff}})-\mathcal{L}^\star
    &\leq\tfrac{1}{2}\bigl(\sqrt{M_0}+C_Q r\bigr)^2\|\Delta\|^2\nonumber\\
    &=\tfrac{1}{2}\bigl(\sqrt{M_0}+C_Q r\bigr)^2\,
    \mathrm{dist}(\boldsymbol{\theta}_{\text{eff}},\mathfrak{M}_I)^2\nonumber\\
    &\leq\tfrac{1}{2}\bigl(\sqrt{M_0}+C_Q r\bigr)^2\,
    \mathrm{dist}(\boldsymbol{\theta}_{\text{eff}},\mathfrak{M}_{I_0})^2.
\end{align}
Combining~\eqref{eq: parameter upper bound along trajectory} and~\eqref{eq: excess loss upper bound 2}, for every $t<\tau$,
\begin{align*}
    \mathrm{dist}(\boldsymbol{\theta}_{\text{eff}}(t),\mathfrak{M}_{I_0})
    &\leq\mathrm{dist}(\boldsymbol{\theta}_{\text{eff}}(0),\mathfrak{M}_{I_0})
    +\|\boldsymbol{\theta}_{\text{eff}}(t)-\boldsymbol{\theta}_{\text{eff}}(0)\|\\
    &\leq\mathrm{dist}(\boldsymbol{\theta}_{\text{eff}}(0),\mathfrak{M}_{I_0})
    +\frac{K_r}{\mu_r}\,\sqrt{\mathcal{L}(\boldsymbol{\theta}_{\text{eff}}(0))-\mathcal{L}^\star}\\
    &\leq\mathrm{dist}(\boldsymbol{\theta}_{\text{eff}}(0),\mathfrak{M}_{I_0})
    \left(1+\frac{K_r\bigl(\sqrt{M_0}+C_Q r\bigr)}{\sqrt{2}\,\mu_r}\right)\\
    &<\eta_{I_0,r},
\end{align*}
where the last inequality follows from the initialization condition~\eqref{eq:initialization condition}. By the definition of $\eta_{I_0,r}$ in~\eqref{eq: def eta}, this implies
\begin{align*}
    \boldsymbol{\theta}_{\text{eff}}(t)\in\mathcal{T}_r(I),
    \qquad\forall\,t<\tau.
\end{align*}
Hence the solution cannot exit $\mathcal{T}_r(I)$ at any finite time, so $\tau=\infty$.

\paragraph{Step 4: Convergence to the optimum manifold.}
Since the trajectory is confined to the bounded neighborhood
\begin{align*}
    \bigl\{\boldsymbol{\theta}_{\text{eff}}\in\boldsymbol{\Theta}_{\text{eff}}:
    \mathrm{dist}(\boldsymbol{\theta}_{\text{eff}},\mathfrak{M}_{I_0})\leq\eta_{I_0,r}\bigr\},
\end{align*}
no finite-time blow-up occurs and the gradient flow is well defined for all $t\geq 0$. The exponential loss decay~\eqref{eq: loss decay} therefore holds globally:
\begin{align}\label{eq: loss decay 2}
    \mathcal{L}(\boldsymbol{\theta}_{\text{eff}}(t))-\mathcal{L}^\star
    \leq e^{-2\mu_r t}\bigl(\mathcal{L}(\boldsymbol{\theta}_{\text{eff}}(0))-\mathcal{L}^\star\bigr),
    \qquad\forall\,t\geq 0.
\end{align}
Combining~\eqref{eq: distance bound} with~\eqref{eq: loss decay 2},
\begin{align}\label{eq: parameter distance bound 2}
    \mathrm{dist}(\boldsymbol{\theta}_{\text{eff}}(t),\mathfrak{M}_I)^2
    &\leq\frac{1}{\lambda_r}\bigl(\mathcal{L}(\boldsymbol{\theta}_{\text{eff}}(t))-\mathcal{L}^\star\bigr)\nonumber\\
    &\leq\frac{1}{\lambda_r}\,e^{-2\mu_r t}
    \bigl(\mathcal{L}(\boldsymbol{\theta}_{\text{eff}}(0))-\mathcal{L}^\star\bigr).
\end{align}
It remains to show that the parameters themselves converge to a point on $\mathfrak{M}_I$. By the same argument as in~\eqref{eq: parameter upper bound along trajectory}, for any $0\leq t_1<t_2$,
\begin{align}\label{eq: parameter upper bound along trajectory 2}
    \|\boldsymbol{\theta}_{\text{eff}}(t_2)-\boldsymbol{\theta}_{\text{eff}}(t_1)\|
    &\leq\int_{t_1}^{t_2}\|\dot{\boldsymbol{\theta}}_{\text{eff}}(s)\|\,ds\nonumber\\
    &\leq K_r\,\sqrt{\mathcal{L}(\boldsymbol{\theta}_{\text{eff}}(0))-\mathcal{L}^\star}
    \int_{t_1}^{t_2} e^{-\mu_r s}\,ds\nonumber\\
    &\leq\frac{K_r}{\mu_r}\,\sqrt{\mathcal{L}(\boldsymbol{\theta}_{\text{eff}}(0))-\mathcal{L}^\star}\,e^{-\mu_r t_1}.
\end{align}
Hence $\boldsymbol{\theta}_{\text{eff}}(t)$ is a Cauchy curve as $t\to\infty$. By completeness of $\boldsymbol{\Theta}_{\text{eff}}$, there exists $\boldsymbol{\theta}_{\text{eff}}(\infty)\in\boldsymbol{\Theta}_{\text{eff}}$ such that
\begin{align*}
    \lim_{t\to\infty}\boldsymbol{\theta}_{\text{eff}}(t)=\boldsymbol{\theta}_{\text{eff}}(\infty).
\end{align*}
Letting $t_2\to\infty$ in~\eqref{eq: parameter upper bound along trajectory 2},
\begin{align*}
    \|\boldsymbol{\theta}_{\text{eff}}(t)-\boldsymbol{\theta}_{\text{eff}}(\infty)\|
    \leq\frac{K_r}{\mu_r}\,\sqrt{\mathcal{L}(\boldsymbol{\theta}_{\text{eff}}(0))-\mathcal{L}^\star}\,
    e^{-\mu_r t}.
\end{align*}
By continuity of $\mathcal{L}$ together with~\eqref{eq: loss decay 2},
\begin{align*}
    \mathcal{L}(\boldsymbol{\theta}_{\text{eff}}(\infty))=\mathcal{L}^\star.
\end{align*}
Moreover, by~\eqref{eq: parameter distance bound 2},
\begin{align*}
    \mathrm{dist}(\boldsymbol{\theta}_{\text{eff}}(t),\mathfrak{M}_I)\to 0
    \qquad\text{as }t\to\infty,
\end{align*}
and since $\mathfrak{M}_I$ is closed, we conclude that
\begin{align*}
    \boldsymbol{\theta}_{\text{eff}}(\infty)\in\mathfrak{M}_I.
\end{align*}
This completes the proof.
\end{proof}

\section{Proof of Theorem \ref{thm: policy gradient tf}}\label{app: proof of thm: policy gradient tf}
\begin{proof}
The argument follows the same four-step structure as the proof of Theorem~\ref{thm: semi-sarsa tf}; only the dimensions and the affine target change.

\paragraph{Step 1: Reduction.}
The actor-critic prompt~\eqref{eq: ac prompt} has the same block-diagonal form as the SARSA prompt~\eqref{eq:input-matrix}, except the trajectory columns $\mathbf{h}_{i+1}:=[\mathbf{x}_{i};\, \mathbf{0}],\, i=0,\ldots,n-1$ occupy the upper $(2d+m+1)$-block, while the parameter column $\mathbf{h}_{n+1}=[\mathbf{0}_{2d+m+1};\,\tilde{\mathbf{w}}]$ occupies the lower $(d+m+1)$-block, with 
\begin{align*}
    \mathbf{x}_i=\begin{bmatrix} \boldsymbol{\phi}_V(s_i)\\ \gamma \boldsymbol{\phi}_V(s_{i+1})\\ r_{i+1}\\\gamma^i \mathbf{g}_{\boldsymbol{\lambda}}(s_i,a_i) \end{bmatrix}\in\mathbb{R}^{2d+m+1}, \,\text{and }\tilde{\mathbf{w}}=\begin{bmatrix}1\\\boldsymbol{\lambda}\\\mathbf{w}\end{bmatrix}\in\mathbb{R}^{d+m+1}.
\end{align*} 
Partitioning $\mathbf{V},\mathbf{P}$ conformally with $D=(2d+m+1)+(d+m+1)$ as in \eqref{eq:block partition}, and letting $\bar{\mathbf{V}}_{21},\bar{\mathbf{V}}_{22}$ now denote the \emph{last $d+m$ rows} of $\mathbf{V}_{21},\mathbf{V}_{22}$, the same block computation as in Lemma~\ref{lem: transformer output} yields
\begin{equation}\label{eq: dw_tf ac}
    \begin{bmatrix}\Delta\boldsymbol{\lambda}_{\textsf{TF}}(\boldsymbol{\theta})\\\Delta\mathbf{w}_{\textsf{TF}}(\boldsymbol{\theta})\end{bmatrix}
    :=\textsf{TF}_{\boldsymbol{\theta}}(\mathbf{H})-\begin{bmatrix}\boldsymbol{\lambda}\\\mathbf{w}\end{bmatrix}
    \;=\;\bar{\mathbf{V}}_{21}\hat{\boldsymbol{\Sigma}}\mathbf{P}_{12}\tilde{\mathbf{w}}+\tfrac{1}{n}\bar{\mathbf{V}}_{22}\tilde{\mathbf{w}}\tilde{\mathbf{w}}^\top\mathbf{P}_{22}\tilde{\mathbf{w}},
\end{equation}
where $\hat{\boldsymbol{\Sigma}}\in\mathbb{R}^{(2d+m+1)\times(2d+m+1)}$ is the empirical second-moment matrix of the trajectory features:
\begin{equation*}
    \hat{\boldsymbol{\Sigma}}=\frac{1}{n}\sum_{i=0}^{n-1}\mathbf{x}_i\mathbf{x}_i^\top:=\begin{bmatrix}
        \hat{\boldsymbol{\Sigma}}_{\phi_V\phi_V} & \gamma\hat{\boldsymbol{\Sigma}}_{\phi_V\phi_{V+}} & \hat{\boldsymbol{\Sigma}}_{\phi_V r} & \hat{\boldsymbol{\Sigma}}_{\phi_V g}\\
        \gamma\hat{\boldsymbol{\Sigma}}_{\phi_{V+}\phi_V} & \gamma^2\hat{\boldsymbol{\Sigma}}_{\phi_{V+}\phi_{V+}} & \gamma\hat{\boldsymbol{\Sigma}}_{\phi_{V+} r} & \gamma\hat{\boldsymbol{\Sigma}}_{\phi_{V+} g}\\
        \hat{\boldsymbol{\Sigma}}_{r\phi_V} & \gamma\hat{\boldsymbol{\Sigma}}_{r\phi_{V+}} & \hat{\boldsymbol{\Sigma}}_{rr} & \hat{\boldsymbol{\Sigma}}_{rg}\\
        \hat{\boldsymbol{\Sigma}}_{g\phi_V} & \gamma\hat{\boldsymbol{\Sigma}}_{g\phi_{V+}} & \hat{\boldsymbol{\Sigma}}_{gr} & \hat{\boldsymbol{\Sigma}}_{gg}
    \end{bmatrix}.
\end{equation*}
The blocks $\mathbf{V}_{11},\mathbf{V}_{12},\mathbf{P}_{11},\mathbf{P}_{21}$ and the first row of $\mathbf{V}_{21},\mathbf{V}_{22}$ are unconstrained.

\paragraph{Step 2: Cancel the cubic term.}
The actor-critic update~\eqref{eq: actor-critic batch} can be written as
\begin{equation}\label{eq: dw_ac}
    \begin{bmatrix}\Delta\boldsymbol{\lambda}_{\text{AC}}\\\Delta\mathbf{w}_{\text{AC}}\end{bmatrix}
    =\begin{bmatrix}\alpha\bigl[\hat{\boldsymbol{\Sigma}}_{gr}+\gamma\hat{\boldsymbol{\Sigma}}_{g\phi_{V+}}\mathbf{w}-\hat{\boldsymbol{\Sigma}}_{g\phi_V}\mathbf{w}\bigr]\\ \beta\bigl[\hat{\boldsymbol{\Sigma}}_{\phi_V r}+\gamma\hat{\boldsymbol{\Sigma}}_{\phi_V\phi_{V+}}\mathbf{w}-\hat{\boldsymbol{\Sigma}}_{\phi_V\phi_V}\mathbf{w}\bigr]\end{bmatrix},
\end{equation}
which is affine in $(\boldsymbol{\lambda},\mathbf{w})$. Setting $\bar{\mathbf{V}}_{22}^\star=\mathbf{P}_{22}^\star=\mathbf{0}$ removes the cubic term in \eqref{eq: dw_tf ac}.

\paragraph{Step 3: Match the affine term.}
We require $\bar{\mathbf{V}}_{21}^\star\hat{\boldsymbol{\Sigma}}\mathbf{P}_{12}^\star\tilde{\mathbf{w}}=\bigl[\Delta\boldsymbol{\lambda}_{\text{AC}};\,\Delta\mathbf{w}_{\text{AC}}\bigr]$ for every $(\boldsymbol{\lambda},\mathbf{w})$. Take
\begin{equation*}
    \bar{\mathbf{V}}_{21}^\star=\begin{bmatrix}\mathbf{0}_{m\times d} & \mathbf{0}_{m\times d} & \mathbf{0}_{m\times 1} & \alpha\mathbf{I}_m\\ \beta\mathbf{I}_d & \mathbf{0}_{d\times d} & \mathbf{0}_{d\times 1} & \mathbf{0}_{d\times m}\end{bmatrix},
\end{equation*}
which extracts the bottom $m$ rows (prefactor by $\alpha$), and the top $d$ rows (prefactor by $\beta$) of $\hat{\boldsymbol{\Sigma}}$:
\begin{equation*}
    \bar{\mathbf{V}}_{21}^\star\hat{\boldsymbol{\Sigma}}=\begin{bmatrix}\alpha\hat{\boldsymbol{\Sigma}}_{g\phi_V} & \alpha\gamma\hat{\boldsymbol{\Sigma}}_{g\phi_{V+}} & \alpha\hat{\boldsymbol{\Sigma}}_{gr} & \alpha\hat{\boldsymbol{\Sigma}}_{gg}\\ \beta\hat{\boldsymbol{\Sigma}}_{\phi_V\phi_V} & \beta\gamma\hat{\boldsymbol{\Sigma}}_{\phi_V\phi_{V+}} & \beta\hat{\boldsymbol{\Sigma}}_{\phi_V r} & \beta\hat{\boldsymbol{\Sigma}}_{\phi_V g}\end{bmatrix}.
\end{equation*}
The matching condition then reduces to $\mathbf{P}_{12}^\star\tilde{\mathbf{w}}=[-\mathbf{w};\,\mathbf{w};\,1;\,\mathbf{0}_m]$. With $\tilde{\mathbf{w}}=[1;\boldsymbol{\lambda};\mathbf{w}]$, this holds for every $(\boldsymbol{\lambda},\mathbf{w})$ when
\begin{equation*}
    \mathbf{P}_{12}^\star=\begin{bmatrix}\mathbf{0}_{d\times 1} & \mathbf{0}_{d\times m} & -\mathbf{I}_d\\ \mathbf{0}_{d\times 1} & \mathbf{0}_{d\times m} & \mathbf{I}_d\\ 1 & \mathbf{0}_{1\times m} & \mathbf{0}_{1\times d}\\ \mathbf{0}_{m\times 1} & \mathbf{0}_{m\times m} & \mathbf{0}_{m\times d}\end{bmatrix}.
\end{equation*}

\paragraph{Step 4: Assembly and scaling.}
Leaving the unconstrained blocks as $*$, we recover $(\mathbf{P}^\star,\mathbf{V}^\star)$ as displayed in \eqref{eq:QKV, actor-critic}. Bilinearity of \eqref{eq: dw_tf ac} in $(\mathbf{V},\mathbf{P})$ gives the same scaling invariance under $(\mathbf{P}^\star,\mathbf{V}^\star)\mapsto(c\mathbf{P}^\star,c^{-1}\mathbf{V}^\star)$ for any $c\neq 0$, completing the construction.
\end{proof}

\section{Convergence Analysis for the Actor-Critic Transformer}\label{app: ac convergence}
In this appendix, we extend the training-dynamics analysis of Section~\ref{subsec: training-dynamics} to the actor-critic transformer and establish a local convergence guarantee (Theorem~\ref{thm: ac local convergence}) analogous to Theorem~\ref{thm:local-convergence-manifold}. With a slight abuse of notation, the quantities $\mathbf{x}_i$, $\tilde{\mathbf{w}}$, $\delta_i$, $\hat{\boldsymbol{\Sigma}}(z)$, $\mathbf{R}(z)$, $\mathbf{b}(z)$, as well as the constants $B_\phi,B_r,B_{\tilde{w}},\kappa_{\tilde{w}},\kappa_R,\kappa_q,\rho,B_\Sigma,C_Q,m_0,M_0,K_r,\lambda_r,\mu_r$, are redefined below in their actor-critic form; within this appendix they supersede the SARSA quantities of the same name.

\subsection{Setup and invariant parameter groups}
A training sample is now $z=(\mathcal{M},\boldsymbol{\phi}_\pi,\boldsymbol{\phi}_V,\boldsymbol{\lambda},\mathbf{w},\tau)$, where $\mathcal{M}$ is the training MDP, $\boldsymbol{\phi}_\pi:\mathcal{S}\times\mathcal{A}\to\mathbb{R}^m$ and $\boldsymbol{\phi}_V:\mathcal{S}\to\mathbb{R}^d$ are the policy and value feature maps, $(\boldsymbol{\lambda},\mathbf{w})$ are the current actor and critic parameters defining the behavior policy $\pi_{\boldsymbol{\lambda}}$, and $\tau=(s_0,a_0,r_1,\ldots,s_n,a_n)$ is the resulting $n$-step trajectory. For $i=0,\ldots,n-1$, define
\begin{gather*}
    \boldsymbol{\phi}_{V,i}:=\boldsymbol{\phi}_V(s_i),\qquad
    \mathbf{g}_i:=\mathbf{g}_{\boldsymbol{\lambda}}(s_i,a_i),\qquad
    \delta_i:=r_{i+1}+\gamma\mathbf{w}^\top\boldsymbol{\phi}_{V,i+1}-\mathbf{w}^\top\boldsymbol{\phi}_{V,i},\\
    \mathbf{x}_i:=\begin{bmatrix}\boldsymbol{\phi}_{V,i}\\ \gamma\boldsymbol{\phi}_{V,i+1}\\ r_{i+1}\\ \gamma^i\mathbf{g}_i\end{bmatrix}\in\mathbb{R}^{2d+m+1},\qquad
    \boldsymbol{\psi}_i:=\begin{bmatrix}\gamma^i\mathbf{g}_i\\ \boldsymbol{\phi}_{V,i}\end{bmatrix}\in\mathbb{R}^{m+d},\qquad
    \tilde{\mathbf{w}}:=\begin{bmatrix}1\\ \boldsymbol{\lambda}\\ \mathbf{w}\end{bmatrix}\in\mathbb{R}^{d+m+1},
\end{gather*}
together with the trajectory statistics and the step-size matrix
\begin{gather*}
    \hat{\boldsymbol{\Sigma}}(z):=\frac{1}{n}\sum_{i=0}^{n-1}\mathbf{x}_i\mathbf{x}_i^\top,\qquad
    \mathbf{R}(z):=\frac{1}{n}\sum_{i=0}^{n-1}\boldsymbol{\psi}_i\mathbf{x}_i^\top,\qquad
    \mathbf{b}(z):=\frac{1}{n}\sum_{i=0}^{n-1}\mathbf{x}_i\,\delta_i,\qquad
    \mathbf{D}:=\begin{bmatrix}\alpha\mathbf{I}_m & \mathbf{0}\\ \mathbf{0} & \beta\mathbf{I}_d\end{bmatrix}.
\end{gather*}
Mimicking the teacher update \eqref{eq: actor-critic batch} as in Algorithm~\ref{alg: training transformer - actor critic}, the population loss is
\begin{align}\label{eq: ac loss}
    \mathcal{L}_{\text{AC}}(\boldsymbol{\theta}):=\mathbb{E}_z\!\left[\frac{1}{2}\left\|\textsf{TF}_{\boldsymbol{\theta}}\bigl(\mathbf{H}(z)\bigr)-\begin{bmatrix}\boldsymbol{\lambda}_{\text{AC}}(z)\\ \mathbf{w}_{\text{AC}}(z)\end{bmatrix}\right\|^2\right],
\end{align}
where $\mathbf{H}(z)$ is the prompt \eqref{eq: ac prompt} and $(\boldsymbol{\lambda}_{\text{AC}}(z),\mathbf{w}_{\text{AC}}(z))$ is the batch actor-critic target \eqref{eq: actor-critic batch}. We study the corresponding gradient flow
\begin{align}\label{eq: ac gradient flow}
    \frac{d\boldsymbol{\theta}(t)}{dt}=-\nabla\mathcal{L}_{\text{AC}}(\boldsymbol{\theta}(t)).
\end{align}
The invariance structure of Proposition~\ref{prop:invariant-set} carries over unchanged.
\begin{prop}[Invariant parameter groups, actor-critic]\label{prop: ac invariant}
Partition $\mathbf{P},\mathbf{V}$ into sub-blocks as in \eqref{eq:block partition}, conformally with $D=(2d+m+1)+(d+m+1)$, and let $\bar{\mathbf{V}}_{21},\bar{\mathbf{V}}_{22}$ denote the last $d+m$ rows of $\mathbf{V}_{21},\mathbf{V}_{22}$. Under the gradient flow \eqref{eq: ac gradient flow}, the following parameters are invariant:
\begin{enumerate}[noitemsep]
    \item[(i)] The blocks $\mathbf{P}_{11},\mathbf{P}_{21},\mathbf{V}_{11},\mathbf{V}_{12}$ and the first rows of $\mathbf{V}_{21},\mathbf{V}_{22}$ do not affect the transformer output and remain unchanged throughout training.
    \item[(ii)] If $\mathbf{P}_{22}^{(0)}=\mathbf{0}$ and $\bar{\mathbf{V}}_{22}^{(0)}=\mathbf{0}$, then $\mathbf{P}_{22}^{(t)}=\mathbf{0}$ and $\bar{\mathbf{V}}_{22}^{(t)}=\mathbf{0}$ for all $t\geq 0$.
\end{enumerate}
\end{prop}
\begin{proof}
By the readout decomposition \eqref{eq: dw_tf ac}, the per-sample residual
\begin{align*}
    \boldsymbol{\delta}(\boldsymbol{\theta},z)=\bar{\mathbf{V}}_{21}\hat{\boldsymbol{\Sigma}}(z)\mathbf{P}_{12}\tilde{\mathbf{w}}+\tfrac{1}{n}\bar{\mathbf{V}}_{22}\tilde{\mathbf{w}}\tilde{\mathbf{w}}^\top\mathbf{P}_{22}\tilde{\mathbf{w}}+\begin{bmatrix}\boldsymbol{\lambda}-\boldsymbol{\lambda}_{\text{AC}}(z)\\ \mathbf{w}-\mathbf{w}_{\text{AC}}(z)\end{bmatrix}
\end{align*}
has exactly the same bilinear-plus-cubic dependence on the sub-blocks $(\mathbf{P}_{12},\mathbf{P}_{22},\bar{\mathbf{V}}_{21},\bar{\mathbf{V}}_{22})$ as \eqref{eq: delta decomp}, so the proof of Proposition~\ref{prop:invariant-set} applies verbatim.
\end{proof}
As in Section~\ref{subsec: training-dynamics}, we henceforth restrict to initializations with $\mathbf{P}_{22}^{(0)}=\mathbf{0}$ and $\bar{\mathbf{V}}_{22}^{(0)}=\mathbf{0}$, which reduces the training dynamics to the effective parameter $\boldsymbol{\theta}_{\text{eff}}:=(\mathbf{P}_{12},\bar{\mathbf{V}}_{21})\in\boldsymbol{\Theta}_{\text{eff}}:=\mathbb{R}^{(2d+m+1)\times(d+m+1)}\times\mathbb{R}^{(d+m)\times(2d+m+1)}$, with the associated prediction map
\begin{align*}
    f_{\boldsymbol{\theta}_{\text{eff}}}(z):=\bar{\mathbf{V}}_{21}\hat{\boldsymbol{\Sigma}}(z)\mathbf{P}_{12}\tilde{\mathbf{w}}\in\mathbb{R}^{d+m}.
\end{align*}

\subsection{Optimal parameter manifold}
From Step 3 of the proof of Theorem~\ref{thm: policy gradient tf}, the optimal effective sub-blocks are
\begin{align*}
    \mathbf{P}_{12}^\star=\begin{bmatrix}\mathbf{0}_{d\times 1} & \mathbf{0}_{d\times m} & -\mathbf{I}_d\\ \mathbf{0}_{d\times 1} & \mathbf{0}_{d\times m} & \mathbf{I}_d\\ 1 & \mathbf{0}_{1\times m} & \mathbf{0}_{1\times d}\\ \mathbf{0}_{m\times 1} & \mathbf{0}_{m\times m} & \mathbf{0}_{m\times d}\end{bmatrix},\qquad
    \bar{\mathbf{V}}_{21}^\star=\begin{bmatrix}\mathbf{0}_{m\times d} & \mathbf{0}_{m\times d} & \mathbf{0}_{m\times 1} & \alpha\mathbf{I}_m\\ \beta\mathbf{I}_d & \mathbf{0}_{d\times d} & \mathbf{0}_{d\times 1} & \mathbf{0}_{d\times m}\end{bmatrix},
\end{align*}
which satisfy the two key identities
\begin{align}\label{eq: ac key identities}
    \bar{\mathbf{V}}_{21}^\star\hat{\boldsymbol{\Sigma}}(z)=\mathbf{D}\mathbf{R}(z),\qquad
    \hat{\boldsymbol{\Sigma}}(z)\mathbf{P}_{12}^\star\tilde{\mathbf{w}}=\mathbf{b}(z),
\end{align}
since $\bar{\mathbf{V}}_{21}^\star\mathbf{x}_i=\mathbf{D}\boldsymbol{\psi}_i$ and $\mathbf{x}_i^\top\mathbf{P}_{12}^\star\tilde{\mathbf{w}}=-\boldsymbol{\phi}_{V,i}^\top\mathbf{w}+\gamma\boldsymbol{\phi}_{V,i+1}^\top\mathbf{w}+r_{i+1}=\delta_i$. Consequently,
\begin{align*}
    f_{(\mathbf{P}_{12}^\star,\bar{\mathbf{V}}_{21}^\star)}(z)=\mathbf{D}\cdot\frac{1}{n}\sum_{i=0}^{n-1}\boldsymbol{\psi}_i\delta_i=\begin{bmatrix}\boldsymbol{\lambda}_{\text{AC}}(z)-\boldsymbol{\lambda}\\ \mathbf{w}_{\text{AC}}(z)-\mathbf{w}\end{bmatrix},
\end{align*}
so $\mathcal{L}_{\text{AC}}$ attains its minimum value $\mathcal{L}_{\text{AC}}^\star=0$ at $(\mathbf{P}_{12}^\star,\bar{\mathbf{V}}_{21}^\star)$ and, by bilinearity of the predictor, at every rescaling $(c\mathbf{P}_{12}^\star,c^{-1}\bar{\mathbf{V}}_{21}^\star)$ with $c\neq 0$. As in Section~\ref{subsec: training-dynamics}, we restrict without loss of generality to $c>0$ and a compact interval $I=[c_-,c_+]\subset(0,\infty)$, and define the optimal parameter manifold and its normal spaces
\begin{align}\label{def: ac manifold}
    \mathfrak{M}_I&:=\bigl\{\bar{\boldsymbol{\theta}}_{\text{eff}}(c):=(c\mathbf{P}_{12}^\star,\,c^{-1}\bar{\mathbf{V}}_{21}^\star)\,:\,c\in I\bigr\},\\
    N_c\mathfrak{M}_I&:=\bigl\{(\mathbf{U},\mathbf{W})\in\boldsymbol{\Theta}_{\text{eff}}:\langle\mathbf{U},\mathbf{P}_{12}^\star\rangle_F-c^{-2}\langle\mathbf{W},\bar{\mathbf{V}}_{21}^\star\rangle_F=0\bigr\}.\nonumber
\end{align}

\subsection{Richness assumptions and main result}
\begin{asp}[Richness of training MDPs, actor-critic]\label{asp: ac richness}
Let $\Pi$ denote the class of behavior policies used to generate training trajectories. There exist positive constants $B_\phi,B_g,B_r,B_{\tilde{w}},\kappa_{\tilde{w}},\kappa_R,\kappa_q$ and $\rho\in[0,1)$ such that, for every $\pi\in\Pi$:
\begin{enumerate}[noitemsep]
    \item[(1)] \textbf{Boundedness.} Almost surely, $\|\boldsymbol{\phi}_V(s)\|\leq B_\phi$, $\|\mathbf{g}_{\boldsymbol{\lambda}}(s,a)\|\leq B_g$, $|r|\leq B_r$, and $\|\tilde{\mathbf{w}}\|\leq B_{\tilde{w}}$.
    \item[(2)] \textbf{Parameter excitation.} $\mathbb{E}[\tilde{\mathbf{w}}\tilde{\mathbf{w}}^\top]\succeq\kappa_{\tilde{w}}\mathbf{I}_{d+m+1}$.
    \item[(3)] \textbf{Bellman-regressor excitation.} $\mathbb{E}[\mathbf{R}(z)^\top\mathbf{R}(z)\mid\tilde{\mathbf{w}}]\succeq\kappa_R\mathbf{I}_{2d+m+1}$ for every $\tilde{\mathbf{w}}$.
    \item[(4)] \textbf{TD-target excitation.} $\mathbb{E}[\mathbf{b}(z)\mathbf{b}(z)^\top]\succeq\kappa_q\mathbf{I}_{2d+m+1}$.
    \item[(5)] \textbf{No destructive cancellation.} For every $c\in I$ and $(\mathbf{U},\mathbf{W})\in N_c\mathfrak{M}_I$,
    \begin{align*}
        \mathbb{E}_z\langle\mathbf{D}\mathbf{R}(z)\mathbf{U}\tilde{\mathbf{w}},\,\mathbf{W}\mathbf{b}(z)\rangle
        \geq-\rho\sqrt{\mathbb{E}_z\|\mathbf{D}\mathbf{R}(z)\mathbf{U}\tilde{\mathbf{w}}\|^2\cdot\mathbb{E}_z\|\mathbf{W}\mathbf{b}(z)\|^2}.
    \end{align*}
\end{enumerate}
As in Assumption~\ref{asp:training-MDP assumptions}, conditions (1)--(2) can be enforced by construction; in particular, the score function is automatically bounded, $\|\mathbf{g}_{\boldsymbol{\lambda}}(s,a)\|\leq 2\max_{s,b}\|\boldsymbol{\phi}_\pi(s,b)\|$, so condition (1) holds whenever the policy features are normalized. Conditions (3)--(5) are structural excitation conditions on the trajectory distribution, analogous to those in the SARSA setting.
\end{asp}
Define the constants
\begin{gather*}
    \alpha_\vee:=\max\{\alpha,\beta\},\qquad\alpha_\wedge:=\min\{\alpha,\beta\},\qquad
    B_\Sigma:=2B_\phi^2+B_r^2+B_g^2,\qquad C_Q:=\frac{1}{2}B_\Sigma B_{\tilde{w}},\\
    m_0:=(1-\rho)\min\bigl\{c_+^{-2}\alpha_\wedge^2\kappa_R\kappa_{\tilde{w}},\,c_-^2\kappa_q\bigr\},\qquad
    M_0:=B_\Sigma^2B_{\tilde{w}}^2\bigl(c_-^{-2}\alpha_\vee^2+2c_+^2\bigr),
\end{gather*}
and, for $0<r<\frac{\sqrt{m_0}}{3C_Q}$, the tube
\begin{align}\label{def: ac tube}
    \mathcal{T}_r(I):=\bigl\{\bar{\boldsymbol{\theta}}_{\text{eff}}(c)+\Delta:c\in I,\,\Delta\in N_c\mathfrak{M}_I,\,\|\Delta\|<r\bigr\}.
\end{align}
We first state the actor-critic counterpart of the local PL inequality (Lemma~\ref{lem:local-pl-ic-sarsa}).
\begin{lemma}[Local Polyak--\L{}ojasiewicz inequality, actor-critic]\label{lem: local-pl-ic-ac}
Suppose Assumption~\ref{asp: ac richness} holds. Then $\mathcal{L}_{\text{AC}}$ attains its minimum value $\mathcal{L}_{\text{AC}}^\star=0$ on $\mathfrak{M}_I$, and for every $0<r<\frac{\sqrt{m_0}}{3C_Q}$,
\begin{align}\label{eq: ac local pl}
    \frac{1}{2}\|\nabla\mathcal{L}_{\text{AC}}(\boldsymbol{\theta}_{\text{eff}})\|^2\geq\mu_r\bigl(\mathcal{L}_{\text{AC}}(\boldsymbol{\theta}_{\text{eff}})-\mathcal{L}_{\text{AC}}^\star\bigr),\qquad\forall\,\boldsymbol{\theta}_{\text{eff}}\in\mathcal{T}_r(I),
\end{align}
where $\mu_r:=\frac{\left(m_0-3C_Q\sqrt{m_0}\,r\right)^2}{\left(\sqrt{M_0}+C_Qr\right)^2}$.
\end{lemma}
We can now state the main result of this appendix.
\begin{thm}[Local convergence to the optimal parameter manifold, actor-critic]\label{thm: ac local convergence}
Suppose Assumption~\ref{asp: ac richness} holds, and let $0<r<\frac{\sqrt{m_0}}{3C_Q}$. Fix a compact subinterval $I_0\subset\mathrm{int}(I)$ and let
\begin{align}\label{eq: ac def eta}
    \eta_{I_0,r}:=\mathrm{dist}\bigl(\mathfrak{M}_{I_0},\boldsymbol{\Theta}_{\text{eff}}\setminus\mathcal{T}_r(I)\bigr).
\end{align}
Suppose the initial parameter $\boldsymbol{\theta}_{\text{eff}}(0)$ satisfies
\begin{align}\label{eq: ac initialization condition}
    \mathrm{dist}(\boldsymbol{\theta}_{\text{eff}}(0),\mathfrak{M}_{I_0})<\min\left(r,\ \frac{\eta_{I_0,r}}{1+\frac{K_r\left(\sqrt{M_0}+C_Qr\right)}{\sqrt{2}\mu_r}}\right),
\end{align}
where
\begin{align*}
    K_r:=\sqrt{2}B_\Sigma B_{\tilde{w}}\Bigl(\bigl(c_+\sqrt{2d+1}+r\bigr)^2+\bigl(c_-^{-1}\sqrt{\alpha^2m+\beta^2d}+r\bigr)^2\Bigr)^{1/2}.
\end{align*}
Consider the gradient flow \eqref{eq: ac gradient flow} restricted to the effective parameters, with initial point $\boldsymbol{\theta}_{\text{eff}}(0)$. Then the solution exists for all $t\geq 0$, remains in $\mathcal{T}_r(I)$, and satisfies
\begin{align*}
    \mathcal{L}_{\text{AC}}(\boldsymbol{\theta}_{\text{eff}}(t))\leq e^{-2\mu_rt}\,\mathcal{L}_{\text{AC}}(\boldsymbol{\theta}_{\text{eff}}(0)),
\end{align*}
where $\mu_r$ is as in Lemma~\ref{lem: local-pl-ic-ac}. Furthermore, $\boldsymbol{\theta}_{\text{eff}}(t)$ converges to a point $\boldsymbol{\theta}_{\text{eff}}(\infty)\in\mathfrak{M}_I$, and
\begin{align*}
    \|\boldsymbol{\theta}_{\text{eff}}(t)-\boldsymbol{\theta}_{\text{eff}}(\infty)\|\leq\frac{K_r}{\mu_r}\sqrt{\mathcal{L}_{\text{AC}}(\boldsymbol{\theta}_{\text{eff}}(0))}\,e^{-\mu_rt}.
\end{align*}
\end{thm}

\subsection{Proof of Lemma~\ref{lem: local-pl-ic-ac}}
\begin{proof}
Fix $c\in I$ and $\Delta=(\mathbf{U},\mathbf{W})\in N_c\mathfrak{M}_I$ with $\|\Delta\|\leq r$, where $\|\Delta\|^2:=\|\mathbf{U}\|_F^2+\|\mathbf{W}\|_F^2$, and set $\boldsymbol{\theta}_{\text{eff}}:=\bar{\boldsymbol{\theta}}_{\text{eff}}(c)+\Delta$. Write $\|\cdot\|_{L^2}:=\sqrt{\mathbb{E}_z\|\cdot\|^2}$. Since every point of $\mathfrak{M}_I$ is a global minimizer of $\mathcal{L}_{\text{AC}}$ with $\mathcal{L}_{\text{AC}}^\star=0$,
\begin{align}\label{eq: ac excess loss}
    \mathcal{L}_{\text{AC}}(\boldsymbol{\theta}_{\text{eff}})-\mathcal{L}_{\text{AC}}^\star=\tfrac{1}{2}\mathbb{E}_z\bigl\|f_{\boldsymbol{\theta}_{\text{eff}}}(z)-f_{\bar{\boldsymbol{\theta}}_{\text{eff}}(c)}(z)\bigr\|^2.
\end{align}

\paragraph{Step 1: Decomposition of the prediction error.}
Expanding $f_{\boldsymbol{\theta}_{\text{eff}}}(z)=(c^{-1}\bar{\mathbf{V}}_{21}^\star+\mathbf{W})\hat{\boldsymbol{\Sigma}}(z)(c\mathbf{P}_{12}^\star+\mathbf{U})\tilde{\mathbf{w}}$, cancelling the matched leading term, and applying the identities \eqref{eq: ac key identities},
\begin{align}\label{eq: ac X Y decomposition}
    f_{\boldsymbol{\theta}_{\text{eff}}}(z)-f_{\bar{\boldsymbol{\theta}}_{\text{eff}}(c)}(z)=\underbrace{c^{-1}\mathbf{D}\mathbf{R}(z)\mathbf{U}\tilde{\mathbf{w}}+c\,\mathbf{W}\mathbf{b}(z)}_{=:\mathbf{X}(z)}+\underbrace{\mathbf{W}\hat{\boldsymbol{\Sigma}}(z)\mathbf{U}\tilde{\mathbf{w}}}_{=:\mathbf{Y}(z)},
\end{align}
which is the exact analogue of \eqref{eq:X Y decomposition}, with the scalar step size $\alpha$ there replaced by the step-size matrix $\mathbf{D}$.

\paragraph{Step 2: Lower bound on $\|\mathbf{X}\|_{L^2}$.}
Since $\mathbf{D}\succeq\alpha_\wedge\mathbf{I}_{m+d}$, we have $\|\mathbf{D}\mathbf{v}\|\geq\alpha_\wedge\|\mathbf{v}\|$ for every $\mathbf{v}\in\mathbb{R}^{m+d}$. Conditioning on $\tilde{\mathbf{w}}$ and applying Assumption~\ref{asp: ac richness}(3) and then Assumption~\ref{asp: ac richness}(2), exactly as in \eqref{eq: lower bound term 1},
\begin{align*}
    \mathbb{E}_z\|\mathbf{D}\mathbf{R}(z)\mathbf{U}\tilde{\mathbf{w}}\|^2\geq\alpha_\wedge^2\,\mathbb{E}_z\|\mathbf{R}(z)\mathbf{U}\tilde{\mathbf{w}}\|^2\geq\alpha_\wedge^2\kappa_R\kappa_{\tilde{w}}\|\mathbf{U}\|_F^2.
\end{align*}
Similarly, Assumption~\ref{asp: ac richness}(4) yields $\mathbb{E}_z\|\mathbf{W}\mathbf{b}(z)\|^2\geq\kappa_q\|\mathbf{W}\|_F^2$, as in \eqref{eq: lower bound term 2}. Setting $A:=\mathbb{E}_z\|c^{-1}\mathbf{D}\mathbf{R}(z)\mathbf{U}\tilde{\mathbf{w}}\|^2$ and $B:=\mathbb{E}_z\|c\,\mathbf{W}\mathbf{b}(z)\|^2$, Assumption~\ref{asp: ac richness}(5) lower-bounds the cross term of $\mathbb{E}_z\|\mathbf{X}(z)\|^2$ by $-2\rho\sqrt{AB}$, and the same algebra as in \eqref{eq: lower bound X intermediate}, together with $c\in[c_-,c_+]$, gives
\begin{align}\label{eq: ac lower bound X}
    \mathbb{E}_z\|\mathbf{X}(z)\|^2\geq(1-\rho)(A+B)\geq(1-\rho)\min\bigl\{c_+^{-2}\alpha_\wedge^2\kappa_R\kappa_{\tilde{w}},\,c_-^2\kappa_q\bigr\}\|\Delta\|^2=m_0\|\Delta\|^2,
\end{align}
i.e., $\|\mathbf{X}\|_{L^2}\geq\sqrt{m_0}\,\|\Delta\|$.

\paragraph{Step 3: Upper bounds on $\|\mathbf{X}\|_{L^2}$ and $\|\mathbf{Y}\|_{L^2}$.}
By Assumption~\ref{asp: ac richness}(1) and $\gamma\in[0,1)$,
\begin{align*}
    \|\mathbf{x}_i\|^2=\|\boldsymbol{\phi}_{V,i}\|^2+\gamma^2\|\boldsymbol{\phi}_{V,i+1}\|^2+r_{i+1}^2+\gamma^{2i}\|\mathbf{g}_i\|^2\leq 2B_\phi^2+B_r^2+B_g^2=B_\Sigma,
\end{align*}
hence $\|\hat{\boldsymbol{\Sigma}}(z)\|_{\mathrm{op}}\leq B_\Sigma$ by the triangle inequality. Since $\boldsymbol{\psi}_i=\mathbf{S}\mathbf{x}_i$ with the selection matrix
\begin{align*}
    \mathbf{S}:=\begin{bmatrix}\mathbf{0}_{m\times d} & \mathbf{0}_{m\times d} & \mathbf{0}_{m\times 1} & \mathbf{I}_m\\ \mathbf{I}_d & \mathbf{0}_{d\times d} & \mathbf{0}_{d\times 1} & \mathbf{0}_{d\times m}\end{bmatrix}\in\mathbb{R}^{(m+d)\times(2d+m+1)},\qquad\mathbf{S}\mathbf{S}^\top=\mathbf{I}_{m+d},
\end{align*}
we have $\mathbf{R}(z)=\mathbf{S}\hat{\boldsymbol{\Sigma}}(z)$ and thus $\|\mathbf{R}(z)\|_{\mathrm{op}}\leq B_\Sigma$. Moreover, $\|\mathbf{D}\|_{\mathrm{op}}=\alpha_\vee$, and $\mathbf{P}_{12}^{\star\top}\mathbf{P}_{12}^\star=\mathrm{diag}(1,\mathbf{0}_{m\times m},2\mathbf{I}_d)$ gives $\|\mathbf{P}_{12}^\star\|_{\mathrm{op}}=\sqrt{2}$, whence $\|\mathbf{b}(z)\|\leq\|\mathbf{P}_{12}^\star\|_{\mathrm{op}}\|\hat{\boldsymbol{\Sigma}}(z)\|_{\mathrm{op}}\|\tilde{\mathbf{w}}\|\leq\sqrt{2}B_\Sigma B_{\tilde{w}}$. Therefore, following the derivations of \eqref{eq: upper bound X} and \eqref{eq: upper bound Y},
\begin{align}
    \|\mathbf{X}\|_{L^2}&\leq c_-^{-1}\alpha_\vee B_\Sigma B_{\tilde{w}}\|\mathbf{U}\|_F+\sqrt{2}\,c_+B_\Sigma B_{\tilde{w}}\|\mathbf{W}\|_F\leq\sqrt{M_0}\,\|\Delta\|,\label{eq: ac upper bound X}\\
    \|\mathbf{Y}\|_{L^2}&\leq B_\Sigma B_{\tilde{w}}\|\mathbf{U}\|_F\|\mathbf{W}\|_F\leq C_Q\|\Delta\|^2.\label{eq: ac upper bound Y}
\end{align}

\paragraph{Step 4: Conclusion.}
By bilinearity of the predictor, $f_{\bar{\boldsymbol{\theta}}_{\text{eff}}(c)+t\Delta}(z)-f_{\bar{\boldsymbol{\theta}}_{\text{eff}}(c)}(z)=t\mathbf{X}(z)+t^2\mathbf{Y}(z)$ for all $t\in\mathbb{R}$. Steps 4--6 of the proof of Lemma~\ref{lem:local-pl-ic-sarsa} use only this scaling structure, the identity \eqref{eq: ac excess loss}, and the three bounds \eqref{eq: ac lower bound X}--\eqref{eq: ac upper bound Y}; repeating them verbatim yields the excess-loss upper bound
\begin{align}\label{eq: ac loss upper bound}
    \mathcal{L}_{\text{AC}}(\boldsymbol{\theta}_{\text{eff}})-\mathcal{L}_{\text{AC}}^\star\leq\tfrac{1}{2}\bigl(\sqrt{M_0}+C_Qr\bigr)^2\|\Delta\|^2,
\end{align}
the directional-derivative lower bound $\langle\nabla\mathcal{L}_{\text{AC}}(\boldsymbol{\theta}_{\text{eff}}),\Delta\rangle\geq\bigl(m_0-3C_Q\sqrt{m_0}\,r\bigr)\|\Delta\|^2$, and finally the PL inequality \eqref{eq: ac local pl} with the stated $\mu_r$.
\end{proof}

\subsection{Proof of Theorem~\ref{thm: ac local convergence}}
\begin{proof}
The proof follows the four steps of the proof of Theorem~\ref{thm:local-convergence-manifold} in Appendix~\ref{proof of thm:local-convergence-manifold}; we verify that each ingredient holds with the actor-critic quantities.

\paragraph{Step 1: Distance bound.}
For $\boldsymbol{\theta}_{\text{eff}}=\bar{\boldsymbol{\theta}}_{\text{eff}}(c)+\Delta\in\mathcal{T}_r(I)$, the reverse triangle inequality applied to \eqref{eq: ac X Y decomposition}, together with \eqref{eq: ac lower bound X} and \eqref{eq: ac upper bound Y}, gives
\begin{align*}
    \sqrt{2\bigl(\mathcal{L}_{\text{AC}}(\boldsymbol{\theta}_{\text{eff}})-\mathcal{L}_{\text{AC}}^\star\bigr)}=\|\mathbf{X}+\mathbf{Y}\|_{L^2}\geq\|\mathbf{X}\|_{L^2}-\|\mathbf{Y}\|_{L^2}\geq\bigl(\sqrt{m_0}-C_Qr\bigr)\|\Delta\|,
\end{align*}
so, with $\lambda_r:=\tfrac{1}{2}\bigl(\sqrt{m_0}-C_Qr\bigr)^2$ and using that the normal displacement equals the distance to the manifold,
\begin{align}\label{eq: ac distance bound}
    \mathrm{dist}(\boldsymbol{\theta}_{\text{eff}},\mathfrak{M}_I)=\|\Delta\|\leq\frac{1}{\sqrt{\lambda_r}}\sqrt{\mathcal{L}_{\text{AC}}(\boldsymbol{\theta}_{\text{eff}})-\mathcal{L}_{\text{AC}}^\star}.
\end{align}

\paragraph{Step 2: Gradient upper bound.}
Let $\mathbf{e}_{\boldsymbol{\theta}_{\text{eff}}}(z):=f_{\boldsymbol{\theta}_{\text{eff}}}(z)-f_{(\mathbf{P}_{12}^\star,\bar{\mathbf{V}}_{21}^\star)}(z)$, so that $\mathbb{E}_z\|\mathbf{e}_{\boldsymbol{\theta}_{\text{eff}}}(z)\|^2=2\bigl(\mathcal{L}_{\text{AC}}(\boldsymbol{\theta}_{\text{eff}})-\mathcal{L}_{\text{AC}}^\star\bigr)$. Differentiating the bilinear predictor,
\begin{align*}
    \nabla_{\bar{\mathbf{V}}_{21}}\mathcal{L}_{\text{AC}}(\boldsymbol{\theta}_{\text{eff}})=\mathbb{E}_z\bigl[\mathbf{e}_{\boldsymbol{\theta}_{\text{eff}}}(z)\bigl(\hat{\boldsymbol{\Sigma}}(z)\mathbf{P}_{12}\tilde{\mathbf{w}}\bigr)^\top\bigr],\qquad
    \nabla_{\mathbf{P}_{12}}\mathcal{L}_{\text{AC}}(\boldsymbol{\theta}_{\text{eff}})=\mathbb{E}_z\bigl[\hat{\boldsymbol{\Sigma}}(z)^\top\bar{\mathbf{V}}_{21}^\top\mathbf{e}_{\boldsymbol{\theta}_{\text{eff}}}(z)\tilde{\mathbf{w}}^\top\bigr].
\end{align*}
On $\mathcal{T}_r(I)$, using $\|\mathbf{P}_{12}^\star\|_F=\sqrt{2d+1}$ and $\|\bar{\mathbf{V}}_{21}^\star\|_F=\sqrt{\alpha^2m+\beta^2d}$, the triangle inequality gives
\begin{align*}
    \|\mathbf{P}_{12}\|_F\leq c_+\sqrt{2d+1}+r=:P_{\max,r},\qquad
    \|\bar{\mathbf{V}}_{21}\|_F\leq c_-^{-1}\sqrt{\alpha^2m+\beta^2d}+r=:V_{\max,r},
\end{align*}
and Cauchy--Schwarz together with $\|\hat{\boldsymbol{\Sigma}}(z)\|_{\mathrm{op}}\leq B_\Sigma$ and $\|\tilde{\mathbf{w}}\|\leq B_{\tilde{w}}$ yields, exactly as in the derivation of \eqref{eq: gradient upper bound},
\begin{align}\label{eq: ac gradient upper bound}
    \|\nabla\mathcal{L}_{\text{AC}}(\boldsymbol{\theta}_{\text{eff}})\|\leq K_r\sqrt{\mathcal{L}_{\text{AC}}(\boldsymbol{\theta}_{\text{eff}})-\mathcal{L}_{\text{AC}}^\star},\qquad\forall\,\boldsymbol{\theta}_{\text{eff}}\in\mathcal{T}_r(I),
\end{align}
with $K_r=\sqrt{2}B_\Sigma B_{\tilde{w}}\sqrt{P_{\max,r}^2+V_{\max,r}^2}$ as defined in the theorem statement.

\paragraph{Steps 3--4: Tube invariance and convergence.}
Steps 3 and 4 of the proof of Theorem~\ref{thm:local-convergence-manifold} rely only on the following ingredients: the local PL inequality (here Lemma~\ref{lem: local-pl-ic-ac}), the gradient upper bound (here \eqref{eq: ac gradient upper bound}), the excess-loss upper bound (here \eqref{eq: ac loss upper bound}), the distance bound (here \eqref{eq: ac distance bound}), the definitions of the buffer distance and the initialization condition (here \eqref{eq: ac def eta} and \eqref{eq: ac initialization condition}), and local Lipschitzness of $\nabla\mathcal{L}_{\text{AC}}$ (immediate, since $\mathcal{L}_{\text{AC}}$ is a polynomial in $\boldsymbol{\theta}_{\text{eff}}$). Repeating those steps verbatim with $(\mathcal{L},m_0,M_0,C_Q,K_r,\lambda_r,\mu_r)$ replaced by their actor-critic counterparts shows that the gradient flow solution exists for all $t\geq 0$ and remains in $\mathcal{T}_r(I)$, that the loss decays as $\mathcal{L}_{\text{AC}}(\boldsymbol{\theta}_{\text{eff}}(t))\leq e^{-2\mu_rt}\mathcal{L}_{\text{AC}}(\boldsymbol{\theta}_{\text{eff}}(0))$, and that $\boldsymbol{\theta}_{\text{eff}}(t)$ converges to a point $\boldsymbol{\theta}_{\text{eff}}(\infty)\in\mathfrak{M}_I$ with $\|\boldsymbol{\theta}_{\text{eff}}(t)-\boldsymbol{\theta}_{\text{eff}}(\infty)\|\leq\frac{K_r}{\mu_r}\sqrt{\mathcal{L}_{\text{AC}}(\boldsymbol{\theta}_{\text{eff}}(0))}\,e^{-\mu_rt}$. This completes the proof.
\end{proof}

\section{Discussion and Numerical Verification of Assumption~\ref{asp:training-MDP assumptions}}\label{app: assumption verification}

In this section, we discuss the intuitions of Assumption~\ref{asp:training-MDP assumptions} and provide numerical evidence that it holds in our experiments.

\subsection{The diagnostic conditions}
Our goal is to establish whether gradient
training converges to an exact implementation, rather than merely whether such
an implementation exists. This requires the training distribution to provide
enough information in the parameter directions relevant to the target update,
and Assumption~\ref{asp:training-MDP assumptions} formalizes the requirement as excitation conditions on that
distribution.

These conditions are not additional structure imposed on the problem: every
training distribution has well-defined values of $\kappa_{\tilde w}$,
$\kappa_R$, $\kappa_q\geq 0$ and $\rho\in[0,1]$, so the substantive richness
requirement for our convergence result is only the non-degeneracy condition
\[
\kappa_{\tilde w},\kappa_R,
\kappa_q>0,\qquad \rho<1.
\]

\emph{What the condition asks of the training tasks.} Each quantity
corresponds to a concrete requirement on the diversity of the training tasks:

\begin{itemize}[leftmargin=1.2em, labelsep=0.4em]
\item $\kappa_{\tilde{w}}>0$ requires the behavior parameter to vary in all
directions across tasks, which can be ensured by drawing a fresh behavior
parameter for each task from a full-dimensional distribution. Initializing
every prompt at the same point, e.g.\ $\mathbf{w}^{(0)}=\mathbf{0}$, reduces
this diversity.
\item $\kappa_R>0$ requires sufficient variation in both the feature maps
across tasks and the visited state--action pairs, the latter of which could be
supported by using $\epsilon>0$ exploration.
\item $\kappa_q>0$ suggests that the training trajectories should not be too
long, because near the task's SARSA fixed point the mean update
$\mathbb{E}[\boldsymbol{\phi}\delta]$ vanishes and late windows become largely
uninformative. Thus, for a fixed data budget, $\kappa_q$ favors sampling more
tasks with fewer windows per task.
\item $\rho<1$ requires the two update channels not to be perfectly aligned,
which is supported by jointly varying rewards, transitions, and features across
tasks.
\end{itemize}
Together, these conditions ensure that the training distribution is
sufficiently rich for the transformer to learn the underlying update rule
rather than a fixed distribution-specific pattern.

\subsection{Numerical check}
The conditions are checkable in practice from the trajectory statistics $(\tilde{\mathbf{w}},\mathbf{R}(z),\mathbf{b}(z))$ computed from the prompts generated during training. All the required quantities are determined by second moments of these statistics, so the numerical check reduces to eigenvalue and constrained canonical-correlation computations.

The expectations in Assumption~\ref{asp:training-MDP assumptions} are over the prompts encountered during training. This distribution mixes over the random task \emph{and} the position of the context window along the teacher trajectory, so the behavior parameter is not fixed but evolves between consecutive windows through the teacher update. We sample directly from this distribution using Monte Carlo for the random-MDP family in our experiments. For each of $N=10{,}000$ samples, we draw a fresh MDP, feature map, and initial parameter $\mathbf{w}^{(0)}$ exactly as in Section~\ref{subsec: training}; draw a window index $t\sim\operatorname{Unif}\{0,\ldots,T/n-2\}$ (the final window is excluded because the training loop does not train on it); roll the teacher forward for $t$ windows, applying the teacher update \eqref{eq:semi-sarsa linear batch}, $\mathbf{w}^{(j+1)}=\mathbf{w}^{(j)}+\frac{\alpha}{n}\sum_{i=0}^{n-1}\boldsymbol{\phi}_i\delta_i$, after each; and record $\tilde{\mathbf{w}}_j$, $\mathbf{R}_j:=\mathbf{R}(z_j)$, and $\mathbf{b}_j:=\mathbf{b}(z_j)$ from the selected window.

\subsubsection{The excitation constants}
The estimates of $\kappa_{\tilde{w}}$ and $\kappa_q$ are the smallest eigenvalues of the empirical second moments,
\begin{align*}
    \hat{\kappa}_{\tilde{w}}:=\lambda_{\min}\biggl(\frac{1}{N}\sum_{j=1}^{N}\tilde{\mathbf{w}}_j\tilde{\mathbf{w}}_j^\top\biggr),\qquad
    \hat{\kappa}_q:=\lambda_{\min}\biggl(\frac{1}{N}\sum_{j=1}^{N}\mathbf{b}_j\mathbf{b}_j^\top\biggr).
\end{align*}
Condition (3) is different: it quantifies over every behavior parameter $\mathbf{w}$, which no finite computation can exhaust. We approximate the quantifier by the sweep $\mathcal{W}_{\mathrm{sweep}}:=\{\mathbf{0},\pm\mathbf{u}_1,\ldots,\pm\mathbf{u}_8,5\mathbf{u}_1\}$, where $\mathbf{u}_1,\ldots,\mathbf{u}_8$ are independent random unit directions. For each fixed $\mathbf{w}\in\mathcal{W}_{\mathrm{sweep}}$ we draw $M=800$ fresh MDPs ($14{,}400$ in total) and set
\begin{align*}
    \hat{\kappa}_R:=\min_{\mathbf{w}\in\mathcal{W}_{\mathrm{sweep}}}\lambda_{\min}\biggl(\frac{1}{M}\sum_{j=1}^{M}\mathbf{R}_j^\top\mathbf{R}_j\biggr).
\end{align*}
For the SARSA task distribution in Section~\ref{subsec: training}, the smallest empirical eigenvalue over $\mathcal{W}_{\mathrm{sweep}}$ occurs at $\mathbf{w}=\mathbf{0}$.

\subsubsection{The no-cancellation constant}

Condition (5) bounds the correlation between the two channels $\mathbf{a}_j(\mathbf{U}):=\mathbf{R}_j\mathbf{U}\tilde{\mathbf{w}}_j\in\mathbb{R}^d$ and $\mathbf{q}_j(\mathbf{W}):=\mathbf{W}\mathbf{b}_j\in\mathbb{R}^d$ through which a perturbation $(\mathbf{U},\mathbf{W})\in\boldsymbol{\Theta}_{\text{eff}}$ of a manifold point $(c\mathbf{P}_{12}^{\star},c^{-1}\bar{\mathbf{V}}_{21}^{\star})$ enters the transformer output. Their empirical energies and cross term are
\begin{align*}
    A(\mathbf{U}):=\frac{1}{N}\sum_{j}\|\mathbf{a}_j\|^2,\qquad
    B(\mathbf{W}):=\frac{1}{N}\sum_{j}\|\mathbf{q}_j\|^2,\qquad
    C(\mathbf{U},\mathbf{W}):=\frac{1}{N}\sum_{j}\langle\mathbf{a}_j,\mathbf{q}_j\rangle.
\end{align*}
The empirical analogue of $\rho$ is the largest normalized negative cross term,
\begin{align}
    \hat{\rho}_{\mathrm{emp}}
    :=\sup_{c>0}\;
    \sup_{\substack{(\mathbf{U},\mathbf{W})\in\mathcal{N}_c\\ A(\mathbf{U})B(\mathbf{W})>0}}
    \frac{-C(\mathbf{U},\mathbf{W})}{\sqrt{A(\mathbf{U})B(\mathbf{W})}}.
    \label{eq: empirical rho}
\end{align}
We reduce the joint supremum in \eqref{eq: empirical rho} to a one-dimensional search over $c$ by expressing it in terms of a constrained eigenvalue for each $c$.

Since $\mathbf{a}_j$ and $\mathbf{q}_j$ are linear in the perturbations, $A$, $B$, and $C$ are quadratic forms in the vectorizations $\mathbf{u}:=\operatorname{rvec}(\mathbf{U})$ and $\mathbf{v}:=\operatorname{rvec}(\mathbf{W})$, where $\operatorname{rvec}$ stacks the rows of a matrix into a column vector. Indeed, the identity $\operatorname{rvec}(\mathbf{A}\mathbf{X}\mathbf{B})=(\mathbf{A}\otimes\mathbf{B}^\top)\operatorname{rvec}(\mathbf{X})$ gives $\mathbf{a}_j=(\mathbf{R}_j\otimes\tilde{\mathbf{w}}_j^\top)\mathbf{u}$ and $\mathbf{q}_j=(\mathbf{I}_d\otimes\mathbf{b}_j^\top)\mathbf{v}$, hence
\begin{align*}
    A(\mathbf{U})=\mathbf{u}^\top\mathbf{M}_A\mathbf{u},\qquad
    B(\mathbf{W})=\mathbf{v}^\top\mathbf{M}_B\mathbf{v},\qquad
    C(\mathbf{U},\mathbf{W})=\mathbf{u}^\top\mathbf{M}_C\mathbf{v},
\end{align*}
with moment matrices accumulated in a single pass over the samples:
\begin{gather*}
    \mathbf{M}_A:=\frac{1}{N}\sum_{j=1}^{N}(\mathbf{R}_j^\top\mathbf{R}_j)\otimes(\tilde{\mathbf{w}}_j\tilde{\mathbf{w}}_j^\top),\qquad
    \mathbf{M}_B:=\mathbf{I}_d\otimes\frac{1}{N}\sum_{j=1}^{N}\mathbf{b}_j\mathbf{b}_j^\top,\\
    \mathbf{M}_C:=\frac{1}{N}\sum_{j=1}^{N}\mathbf{R}_j^\top\otimes(\tilde{\mathbf{w}}_j\mathbf{b}_j^\top).
\end{gather*}
Let $\mathbf{M}_A=\mathbf{L}_A\mathbf{L}_A^\top$ and $\mathbf{M}_B=\mathbf{L}_B\mathbf{L}_B^\top$ be Cholesky factorizations, and pass to the whitened coordinates
\begin{align}
    \boldsymbol{\xi}:=\mathbf{L}_A^\top\mathbf{u},\qquad
    \boldsymbol{\eta}:=\mathbf{L}_B^\top\mathbf{v},\qquad
    \mathbf{K}:=\mathbf{L}_A^{-1}\mathbf{M}_C\mathbf{L}_B^{-\top},
    \label{eq: rho whitened coordinates}
\end{align}
in which $A=\|\boldsymbol{\xi}\|^2$, $B=\|\boldsymbol{\eta}\|^2$, and $C=\boldsymbol{\xi}^\top\mathbf{K}\boldsymbol{\eta}$. Without the normal-space constraint, the inner supremum in \eqref{eq: empirical rho} equals the largest singular value $\sigma_1(\mathbf{K})$. Along the scaling tangent $(\mathbf{U},\mathbf{W})\propto(\mathbf{P}_{12}^{\star},-c^{-2}\bar{\mathbf{V}}_{21}^{\star})$, the two channels satisfy $\mathbf{q}_j=-\alpha c^{-2}\mathbf{a}_j$ for every sample $j$. Their normalized negative correlation therefore attains its upper bound of one, giving $\sigma_1(\mathbf{K})=1$. The normal-space constraint excludes this tangent direction, allowing the diagnostic to test whether comparable cancellation occurs along any remaining perturbation direction.

In the whitened coordinates \eqref{eq: rho whitened coordinates}, the constraint is a single hyperplane. Let $\mathbf{n}_U:=\mathbf{L}_A^{-1}\operatorname{rvec}(\mathbf{P}_{12}^{\star})$ and $\mathbf{n}_W:=\mathbf{L}_B^{-1}\operatorname{rvec}(\bar{\mathbf{V}}_{21}^{\star})$, and define
\begin{align}
    \mathbf{n}(c):=\begin{bmatrix}c\,\mathbf{n}_U\\ -c^{-1}\mathbf{n}_W\end{bmatrix},\qquad
    \mathbf{H}:=\begin{bmatrix}\mathbf{0} & \mathbf{K}\\ \mathbf{K}^\top & \mathbf{0}\end{bmatrix}.
    \label{eq: rho normal and matrix}
\end{align}
For $AB>0$, normalize the two blocks by setting
\begin{align}
    \mathbf{z}:=\begin{bmatrix}
        \dfrac{\boldsymbol{\xi}}{\|\boldsymbol{\xi}\|}\\[6pt]
        \dfrac{\boldsymbol{\eta}}{\|\boldsymbol{\eta}\|}
    \end{bmatrix}.
    \label{eq: rho normalized vector}
\end{align}
Absorbing the normalization factors into the scale $c>0$, the constraint in \eqref{eq: normal hyperplane} becomes
\begin{align}
    \mathbf{n}(c)^\top\mathbf{z}=0.
    \label{eq: rho hyperplane constraint}
\end{align}
The unit block norms in \eqref{eq: rho normalized vector} and the definition of $\mathbf{H}$ in \eqref{eq: rho normal and matrix} give
\begin{align}
    \frac{\mathbf{z}^\top\mathbf{H}\mathbf{z}}{\|\mathbf{z}\|^2}
    =\frac{\boldsymbol{\xi}^\top\mathbf{K}\boldsymbol{\eta}}
          {\|\boldsymbol{\xi}\|\,\|\boldsymbol{\eta}\|}
    =\frac{C}{\sqrt{AB}}.
    \label{eq: rho rayleigh quotient}
\end{align}
For a general nonzero vector $\mathbf{x}$ satisfying the same hyperplane constraint \eqref{eq: rho hyperplane constraint}, partition it into two blocks and set
\begin{align*}
    \mathbf{x}=\begin{bmatrix}\mathbf{x}_1\\ \mathbf{x}_2\end{bmatrix},\qquad
    A:=\|\mathbf{x}_1\|^2,\qquad
    B:=\|\mathbf{x}_2\|^2,\qquad
    C:=\mathbf{x}_1^\top\mathbf{K}\mathbf{x}_2.
\end{align*}
Thus $\|\mathbf{x}\|^2=A+B$ and $\mathbf{x}^\top\mathbf{H}\mathbf{x}=2C$. When $C<0$, the inequality $A+B\geq 2\sqrt{AB}$ gives
\begin{align}
    -\frac{\mathbf{x}^\top\mathbf{H}\mathbf{x}}{\|\mathbf{x}\|^2}
    =\frac{-2C}{A+B}
    \leq\frac{-C}{\sqrt{AB}}
    \leq\hat{\rho}_{\mathrm{emp}}.
    \label{eq: rho rayleigh bound}
\end{align}
Vectors with $C\geq0$ give nonpositive values. Conversely, \eqref{eq: rho rayleigh quotient} shows that the normalized construction \eqref{eq: rho normalized vector} realizes every objective value in \eqref{eq: empirical rho} as a negative Rayleigh quotient at some $c>0$. Together with \eqref{eq: rho rayleigh bound}, this proves that the joint suprema agree, giving
\begin{align}
    \hat{\rho}_{\mathrm{emp}}=\sup_{c>0}\,\bigl[-\mu(c)\bigr],
    \qquad
    \mu(c):=\min_{\substack{\mathbf{x}\neq\mathbf{0}\\ \mathbf{n}(c)^\top\mathbf{x}=0}}
    \frac{\mathbf{x}^\top\mathbf{H}\mathbf{x}}{\|\mathbf{x}\|^2}.
    \label{eq: rho from constrained eigenvalue}
\end{align}

For fixed $c$, the constrained minimum $\mu(c)$ in \eqref{eq: rho from constrained eigenvalue} is available exactly. The singular value decomposition $\mathbf{K}=\sum_i\sigma_i\boldsymbol{\ell}_i\mathbf{r}_i^\top$ gives the nonzero eigenpairs of $\mathbf{H}$ in \eqref{eq: rho normal and matrix} as
\begin{align*}
    \left(\pm\sigma_i,\;
    \frac{1}{\sqrt{2}}\begin{bmatrix}\boldsymbol{\ell}_i\\ \pm\mathbf{r}_i\end{bmatrix}\right),
    \qquad \sigma_i>0.
\end{align*}
Let $(\lambda_j,\mathbf{e}_j)$ enumerate an orthonormal eigenbasis of $\mathbf{H}$, including the zero eigenspace, with $\lambda_1\leq\lambda_2\leq\cdots$. The squared projection of the unit normal onto $\mathbf{e}_j$ is
\begin{align}
    \omega_j(c):=\left(\mathbf{e}_j^\top\frac{\mathbf{n}(c)}{\|\mathbf{n}(c)\|}\right)^2
    =\frac{\bigl(\mathbf{e}_j^\top\mathbf{n}(c)\bigr)^2}{\|\mathbf{n}(c)\|^2}.
    \label{eq: rho spectral weights}
\end{align}
For fixed $c$, the Lagrangian for the constrained minimization in \eqref{eq: rho from constrained eigenvalue} is
\begin{align*}
    \mathcal{L}_c(\mathbf{x},\nu)
    :=\frac{\mathbf{x}^\top\mathbf{H}\mathbf{x}}{\|\mathbf{x}\|^2}
    -\nu\,\mathbf{n}(c)^\top\mathbf{x}.
\end{align*}
At a minimizer, the first-order condition $\nabla_{\mathbf{x}}\mathcal{L}_c=\mathbf{0}$ gives
\begin{align}
    \nabla_{\mathbf{x}}\left(\frac{\mathbf{x}^\top\mathbf{H}\mathbf{x}}{\|\mathbf{x}\|^2}\right)
    =\frac{2}{\|\mathbf{x}\|^2}(\mathbf{H}-\mu(c)\mathbf{I})\mathbf{x}
    =\nu\,\mathbf{n}(c).
    \label{eq: rho stationarity}
\end{align}
Thus the gradient is parallel to $\mathbf{n}(c)$. Writing $\beta:=\frac{\nu\|\mathbf{x}\|^2}{2}$, \eqref{eq: rho stationarity} becomes $(\mathbf{H}-\mu(c)\mathbf{I})\mathbf{x}=\beta\mathbf{n}(c)$. When $\mu(c)\neq\lambda_j$ for every $j$, we have $\beta\neq0$ because $\mathbf{x}\neq\mathbf{0}$. Projecting onto $\mathbf{e}_j$ gives
\begin{align*}
    \mathbf{e}_j^\top\mathbf{x}
    =\beta\frac{\mathbf{e}_j^\top\mathbf{n}(c)}{\lambda_j-\mu(c)}.
\end{align*}
Substituting this expression into the hyperplane constraint in \eqref{eq: rho from constrained eigenvalue} yields
\begin{align*}
    0=\mathbf{n}(c)^\top\mathbf{x}
    =\sum_j\bigl(\mathbf{e}_j^\top\mathbf{n}(c)\bigr)\bigl(\mathbf{e}_j^\top\mathbf{x}\bigr)
    =\beta\sum_j\frac{\bigl(\mathbf{e}_j^\top\mathbf{n}(c)\bigr)^2}{\lambda_j-\mu(c)}.
\end{align*}
Dividing by $\beta\|\mathbf{n}(c)\|^2$ and using the weights in \eqref{eq: rho spectral weights} gives the secular equation
\begin{align}
    \sum_j\frac{\omega_j(c)}{\lambda_j-\mu(c)}=0.
    \label{eq: rho secular equation}
\end{align}
Eigenvalue interlacing places $\mu(c)$ in $[\lambda_1,\lambda_2]$, providing the interval for bracketed root finding in \eqref{eq: rho secular equation}. If $\mu(c)$ coincides with an eigenvalue, the undivided stationarity condition \eqref{eq: rho stationarity} applies. In particular, $\omega_1(c)=0$ in \eqref{eq: rho spectral weights} implies that $\mathbf{e}_1$ satisfies the hyperplane constraint and $\mu(c)=\lambda_1$.

It remains to evaluate the one-dimensional supremum over $c$ in \eqref{eq: rho from constrained eigenvalue}. Only the direction of $\mathbf{n}(c)$ matters. We evaluate $-\mu(c)$ on a $97$-point grid uniform in $\log c\in[-12,12]$ and refine the five best grid neighborhoods by bounded one-dimensional optimization. We also evaluate the two limiting cases, obtained from \eqref{eq: rho normal and matrix} as
\begin{align*}
    \lim_{c\to\infty}\frac{\mathbf{n}(c)}{\|\mathbf{n}(c)\|}
    =\frac{1}{\|\mathbf{n}_U\|}\begin{bmatrix}\mathbf{n}_U\\ \mathbf{0}\end{bmatrix},\qquad
    \lim_{c\to0^+}\frac{\mathbf{n}(c)}{\|\mathbf{n}(c)\|}
    =\frac{1}{\|\mathbf{n}_W\|}\begin{bmatrix}\mathbf{0}\\ -\mathbf{n}_W\end{bmatrix}.
\end{align*}
Let $\mathcal{C}_{\mathrm{num}}$ collect all points evaluated by this search, including the two limiting cases. We define
\begin{align*}
    \hat{\rho}_{\mathrm{num}}
    :=\max_{c\in\mathcal{C}_{\mathrm{num}}}\bigl[-\mu(c)\bigr]
    \approx \sup_{c>0}\bigl[-\mu(c)\bigr]
    =\hat{\rho}_{\mathrm{emp}}.
\end{align*}
We report $\hat{\rho}_{\mathrm{num}}$ as our numerical estimate of $\hat{\rho}_{\mathrm{emp}}$.

\subsubsection{Results}\label{subapp: verification results}

With $N=10{,}000$ mixture samples and the $14{,}400$-MDP sweep for $\kappa_R$, we obtain
\begin{align*}
    \hat{\kappa}_{\tilde{w}}=0.154,\qquad
    \hat{\kappa}_R=0.043,\qquad
    \hat{\kappa}_q=3.9\times10^{-3},\qquad
    \hat{\rho}_{\mathrm{num}}=0.885.
\end{align*}
All three excitation estimates are positive, and $\hat{\rho}_{\mathrm{num}}=0.885<1$.

\section{Additional Experiments and Details}\label{app:experiments}
\subsection{Training loss curves}\label{app: training loss}
We plot the training loss $\hat{\mathcal{L}}(\boldsymbol{\theta}^{(t)})$ over training iterations for both the SARSA transformer and the Actor-Critic transformer in our experiments presented in Section \ref{sec: experiments}. The results are shown in Figure \ref{fig: training loss}.
\begin{figure}[H]
    \centering
    \begin{subfigure}{0.45\textwidth}
        \centering
        \includegraphics[width=\linewidth]{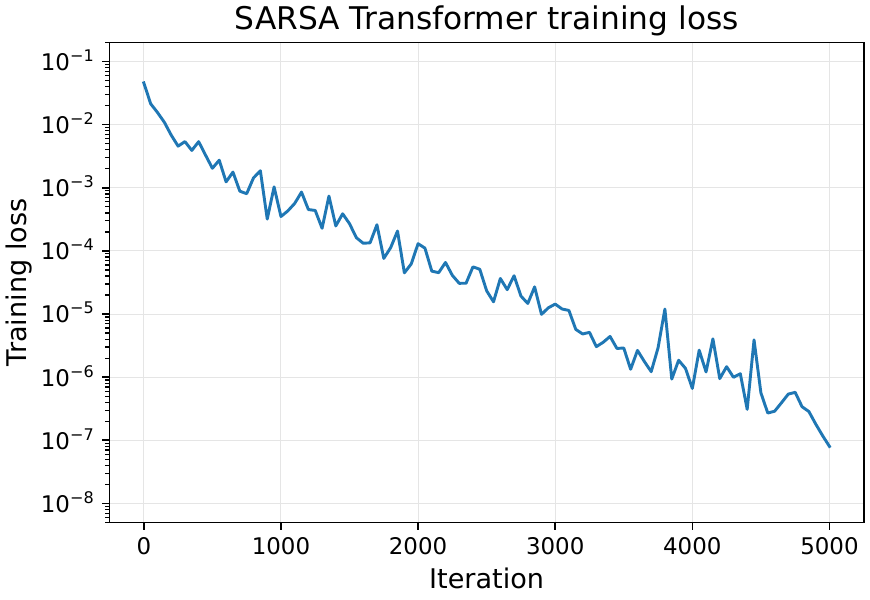}
        \caption{SARSA transformer.}
        \label{fig: training loss sarsa}        
    \end{subfigure}
    \hfill
    \begin{subfigure}{0.45\textwidth}
        \centering
        \includegraphics[width=\linewidth]{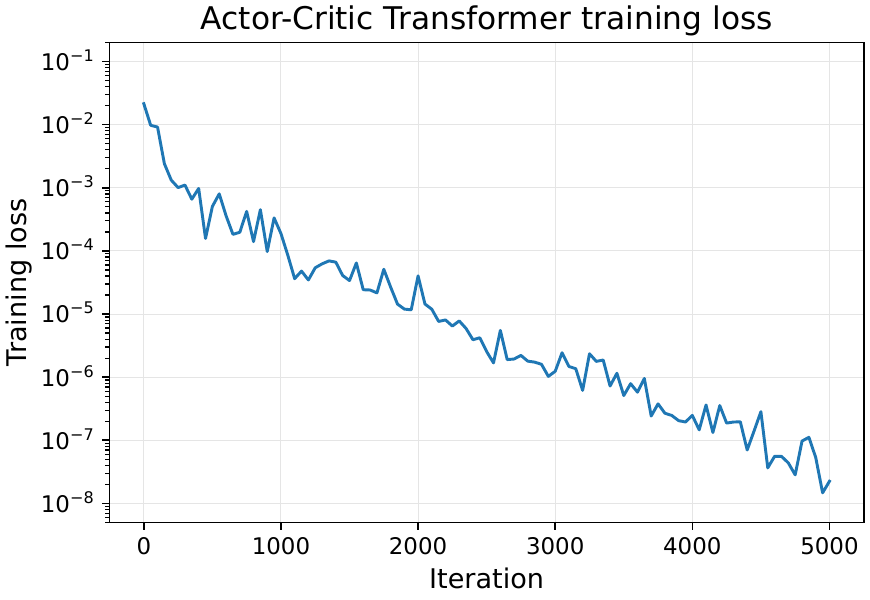}
        \caption{Actor-Critic transformer.}
        \label{fig: training loss ac}
    \end{subfigure}
    \caption{Training loss $\hat{\mathcal{L}}(\boldsymbol{\theta}^{(t)})$ over training iterations for the SARSA transformer (left) and the Actor-Critic transformer (right). The vertical axes are logarithmic. In both settings, the loss decays exponentially, consistent with the local convergence guarantee in Theorem~\ref{thm:local-convergence-manifold}.}
    \label{fig: training loss}
\end{figure}

\subsection{Cosine similarity to the theoretical constructions and effect of the training budget}\label{app: cosine similarity}

We quantify the visual similarity reported in Figure~\ref{fig: parameter} between the learned parameters and the theoretical constructions in \eqref{eq:QKV} and \eqref{eq:QKV, actor-critic}. Because the constructions are recovered only up to the scaling symmetry $(\mathbf{P},\mathbf{V})\mapsto(c\mathbf{P},c^{-1}\mathbf{V})$, cosine similarity is the natural metric: for a learned block $\mathbf{A}$ with analytical counterpart $\mathbf{A}^\star$, we report
\begin{align*}
    \cos(\mathbf{A}):=\frac{\langle\mathbf{A},\mathbf{A}^\star\rangle_F}{\|\mathbf{A}\|_F\,\|\mathbf{A}^\star\|_F},
\end{align*}
which is invariant to the scale $c$, and satisfies $|\cos(\mathbf{A})|=1$ if and only if the learned block is an exact member of the theoretical family. We evaluate it on the effective sub-blocks $(\mathbf{P}_{12},\bar{\mathbf{V}}_{21})$.

Table~\ref{tab: cosine budget} reports the cosine similarities of checkpoints trained with $K=100$, $500$, $1{,}000$, and $5{,}000$ tasks, for both the SARSA and the actor-critic transformers of Section~\ref{sec: experiments}. The similarity increases monotonically with the training budget and, at $K=5{,}000$, the constructions of \eqref{eq:QKV} and \eqref{eq:QKV, actor-critic} are recovered exactly up to the scalar symmetry.

\begin{table}[H]
\centering
\caption{Cosine similarity between the learned effective blocks and the theoretical constructions in \eqref{eq:QKV} (SARSA) and \eqref{eq:QKV, actor-critic} (actor-critic), as a function of the number of training tasks $K$.}
\label{tab: cosine budget}
\begin{tabular}{lcccc}
\toprule
 & \multicolumn{2}{c}{SARSA} & \multicolumn{2}{c}{Actor-critic} \\
Training tasks $K$ & $\cos(\mathbf{P}_{12})$ & $\cos(\bar{\mathbf{V}}_{21})$
 & $\cos(\mathbf{P}_{12})$ & $\cos(\bar{\mathbf{V}}_{21})$ \\
\midrule
$100$      & $0.411$ & $0.856$ & $0.571$ & $0.934$ \\
$500$      & $0.980$ & $0.964$ & $0.973$ & $0.988$ \\
$1{,}000$  & $0.997$ & $0.993$ & $0.993$ & $0.995$ \\
$5{,}000$  & $1.000$ & $1.000$ & $1.000$ & $1.000$ \\
\bottomrule
\end{tabular}
\end{table}

Parameter-level convergence translates into control performance at deployment. Table~\ref{tab: return budget} reports the final mean discounted return of the same checkpoints and of the analytical teacher on $100$ held-out MDPs, following the closed-loop evaluation protocol of Section~\ref{subsec: evaluation} (SARSA: after $100$ in-context updates; actor-critic: after $3{,}000$ in-context updates). The transformer--teacher gap closes monotonically as the training budget grows. 

\begin{table}[H]
\centering
\caption{Final mean discounted return on $100$ held-out MDPs as a function of the number of training tasks $K$.}
\label{tab: return budget}
\begin{tabular}{lcccc}
\toprule
 & \multicolumn{2}{c}{SARSA} & \multicolumn{2}{c}{Actor-critic} \\
Training tasks $K$ & transformer & teacher & transformer & teacher \\
\midrule
$100$      & $0.30$ & $0.42$ & $0.08$ & $0.39$ \\
$500$      & $0.37$ & $0.42$ & $0.18$ & $0.39$ \\
$1{,}000$  & $0.40$ & $0.42$ & $0.32$ & $0.39$ \\
$5{,}000$  & $0.41$ & $0.42$ & $0.39$ & $0.39$ \\
\bottomrule
\end{tabular}
\end{table}

\subsection{Non-tabular MDPs: continuous-state MountainCar}\label{app: mountaincar}

Our theoretical setting operates on linear function approximation with an arbitrary feature map $\boldsymbol{\phi}:\mathcal{S}\times\mathcal{A}\to\mathbb{R}^d$, and all results (Theorems~\ref{thm: semi-sarsa tf}, \ref{thm:local-convergence-manifold}, and \ref{thm: policy gradient tf}) are stated at this level of generality. Moreover, the convergence guarantee depends on the training distribution only through the richness conditions of Assumptions~\ref{asp:training-MDP assumptions} and~\ref{asp: ac richness}, which do not distinguish tabular from non-tabular environments. To verify empirically that the recovered block structure generalizes beyond the tabular case, we repeat the SARSA experiment on a parameterized family of MountainCar tasks with continuous states.

\paragraph{Setup.} The two-dimensional continuous state $s=(x,\dot{x})$ (position and velocity) is represented by a $12$-dimensional Gaussian radial basis feature vector $\boldsymbol{\psi}_a(s)\in\mathbb{R}^{12}$ for each action, with the RBF centers and bandwidth independently sampled for each task and action. The corresponding state--action feature is
\begin{align*}
\boldsymbol{\phi}(s,a)
=\mathbf{e}_a\otimes\boldsymbol{\psi}_a(s)
\in\mathbb{R}^{36},
\end{align*}
where $\mathbf{e}_a\in\mathbb{R}^3$ is the one-hot encoding of action $a\in\{0,1,2\}$ (accelerate left, coast, and accelerate right, respectively). The feature dimension $d=36$ and transformer input dimension $D=3d+2=110$ are identical to those in our tabular SARSA experiments. The transition follows the MountainCar dynamics
\begin{align*}
\dot{x}_{t+1}
=\operatorname{clip}\!\left(
\dot{x}_t+(a_t-1)F-G\cos(3x_t),-0.07,0.07
\right),
\qquad
x_{t+1}
=\operatorname{clip}\!\left(x_t+\dot{x}_{t+1},-1.2,0.6\right),
\end{align*}
with the velocity reset to zero at the left boundary. For each task, the engine-force coefficient and gravity coefficient are independently sampled as $F\sim\mathrm{Unif}(0.0015,0.003)$ and $G\sim\mathrm{Unif}(0.0015,0.003)$ and are then held fixed within the task. Each task also independently resamples the goal position, initial-state distribution, step cost, goal reward, and RBF basis. Upon reaching the goal, the car receives the goal reward and is reset to a newly sampled initial state, producing a continuing trajectory. We train the same one-layer looped linear-attention transformer using Algorithm~\ref{alg: training transformer} on the MountainCar task family.

\paragraph{Task distribution.} For each task, we draw the goal position $x_{\mathrm{goal}}\sim\mathrm{Unif}(0.4,0.55)$, the reward for a non-goal transition $r_{\mathrm{step}}\sim\mathrm{Unif}(-0.1,-0.02)$, and the goal reward $r_{\mathrm{goal}}\sim\mathrm{Unif}(1,2)$. To define the initial-state distribution, we draw a center $c_x\sim\mathrm{Unif}(-0.9,x_{\mathrm{goal}}-0.1)$ and a half-width $h_x\sim\mathrm{Unif}(0.1,0.5)$, then sample
\begin{align*}
    x_0&\sim\mathrm{Unif}\!\left(\max\{c_x-h_x,-1.2\},\min\{c_x+h_x,x_{\mathrm{goal}}-0.02\}\right),\\
    \dot{x}_0&\sim\mathrm{Unif}(-0.02,0.02).
\end{align*}
These task parameters are held fixed throughout the trajectory. A transition with $x_{t+1}\geq x_{\mathrm{goal}}$ receives $r_{\mathrm{goal}}$ instead of $r_{\mathrm{step}}$, followed by a reset drawn from the same initial-state distribution; reaching the goal does not end the training task.

\paragraph{Feature construction.} We normalize the state to $\mathbf{z}(s)=((x+1.2)/1.8,(\dot{x}+0.07)/0.14)\in[0,1]^2$. For every task and action $a$, we independently draw $\sigma_a\sim\mathrm{Unif}(0.15,0.35)$ and perturb each center of a $4\times3$ grid by independent coordinate-wise noise $\boldsymbol{\xi}_{a,jk}\sim\mathrm{Unif}(-0.08,0.08)^2$:
\begin{align*}
    \mathbf{c}_{a,jk}&=\left(\frac{j+1/2}{4},\frac{k+1/2}{3}\right)+\boldsymbol{\xi}_{a,jk},
    \qquad j=0,\ldots,3,\quad k=0,\ldots,2,\\
    [\mathbf{q}_a(s)]_{jk}&=\exp\!\left(-\frac{\|\mathbf{z}(s)-\mathbf{c}_{a,jk}\|_2^2}{2\sigma_a^2}\right),
    \qquad \boldsymbol{\psi}_a(s)=\frac{\mathbf{q}_a(s)}{\|\mathbf{q}_a(s)\|_2}.
\end{align*}
The $12$ entries are stacked in a fixed order. The basis is fixed within a task and resampled independently for each action and each new task.

\paragraph{Training details.} We use discount $\gamma=0.99$, SARSA step size $\alpha=0.2$, exploration $\epsilon=0.1$, window length $n=20$, and a horizon of $1{,}000$ environment transitions per task. At each task's start, the entries of $\mathbf{w}^{(0)}\in\mathbb{R}^{36}$ are sampled independently from $\mathrm{Unif}(-1,1)$. The teacher parameter is updated after each window and determines the behavior policy for the next window. The final window, which reaches the horizon, is not used for training, giving $49$ optimizer updates per task. The transformer has one layer and one linear-attention head, with the block initialization described in Section~\ref{subsec: training}. We use Adam with initial learning rate $10^{-3}$ and exponential decay by a
factor of 0.99 every 10 MDPs. Training runs for $K=10{,}000$ tasks.

\paragraph{Results.} After training, all trainable blocks outside the effective parameter $(\mathbf{P}_{12},\bar{\mathbf{V}}_{21})$ are exactly zero, consistent with Proposition~\ref{prop:invariant-set}, so we report in Table~\ref{tab: mountaincar cosine} the cosine similarity between each learned block and its analytical SARSA counterpart in \eqref{eq:QKV}, at training checkpoints of $K=100$, $500$, $1{,}000$, $5{,}000$, and $10{,}000$ tasks. Figure~\ref{fig: mountaincar parameters} shows the full learned matrices $\mathbf{P}$ and $\mathbf{V}$, and Figure~\ref{fig: mountaincar loss} shows the training loss. The learned parameter matrices recover the predicted block structure in this non-tabular, continuous-state setting as well.

\begin{table}[H]
\centering
\caption{Cosine similarity between the learned effective blocks and the analytical SARSA construction \eqref{eq:QKV} on the MountainCar task family, as a function of the number of training tasks $K$.}
\label{tab: mountaincar cosine}
\begin{tabular}{lcc}
\toprule
Training tasks $K$ & $\cos(\mathbf{P}_{12})$ & $\cos(\bar{\mathbf{V}}_{21})$ \\
\midrule
$100$      & $0.064$ & $0.050$ \\
$500$      & $0.334$ & $0.301$ \\
$1{,}000$  & $0.549$ & $0.508$ \\
$5{,}000$  & $0.897$ & $0.901$ \\
$10{,}000$ & $1.000$ & $0.997$ \\
\bottomrule
\end{tabular}
\end{table}

\begin{figure}[H]
    \centering
    \includegraphics[width=\linewidth]{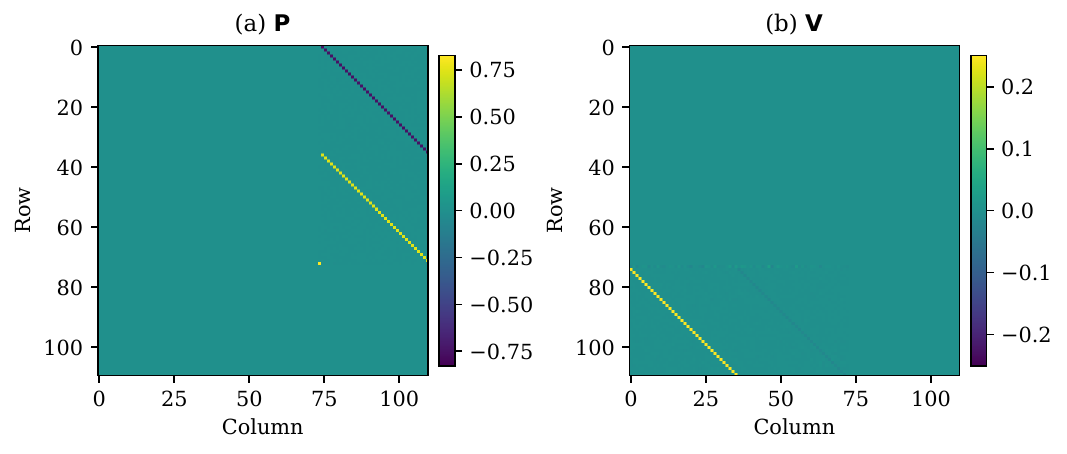}
    \caption{Full parameter matrices $\mathbf{P}$ (left) and $\mathbf{V}$ (right) learned after $K=10{,}000$ MountainCar training tasks.}
    \label{fig: mountaincar parameters}
\end{figure}

\begin{figure}[htbp]
    \centering
    \includegraphics[width=0.78\linewidth]{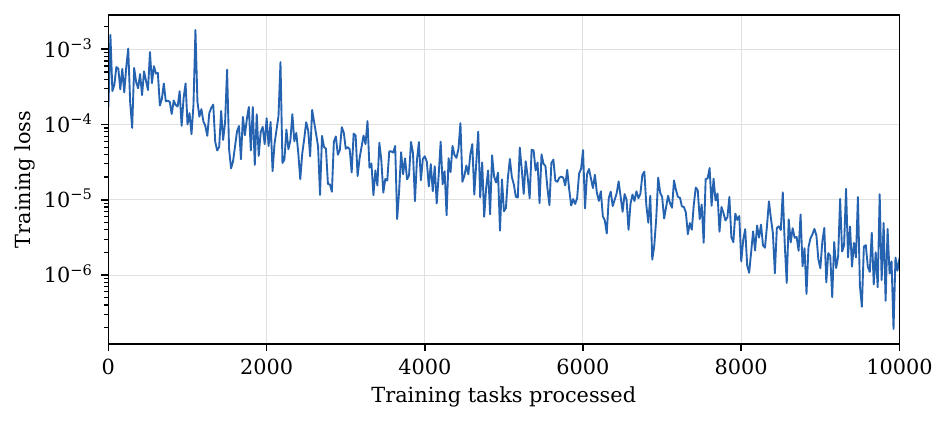}
    \caption{Training loss on the MountainCar task family over $K=10{,}000$ tasks. The vertical axis is logarithmic.}
    \label{fig: mountaincar loss}
\end{figure}
\color{black}
\subsection{Multi-layer, multi-head, and nonlinear attention variants}\label{app: nonlinear attention}

Our construction extends directly to multi-head and multi-layer linear attention: the outputs of a multi-head layer are summed, so zeroing the value maps of the extra heads recovers the single-head implementation exactly, and each attention block acts residually via \eqref{def:hLin}, so setting $\mathbf{V}=\mathbf{0}$ in the remaining layers of an $L$-layer stack reduces them to the identity. Nonlinear activations change the picture: under ReLU attention, \emph{two} heads recover linear attention exactly, since $\mathrm{ReLU}(x)-\mathrm{ReLU}(-x)=x$, whereas a softmax head cannot represent the update exactly and can only approximate a linear head using additional heads.

To test whether these representational considerations are predictive of trainability, we repeat the main SARSA experiment from Section~\ref{subsec: training} with nonlinear-attention variants: single- and two-head ReLU and softmax attention, obtained by applying the corresponding activation to the attention scores $\mathbf{H}^\top\mathbf{P}\mathbf{H}$ in \eqref{def:hLin} (with the outputs of multiple heads summed), keeping all other settings fixed. Table~\ref{tab: nonlinear attention} reports the mean discounted return on $100$ held-out MDPs after $100$ closed-loop in-context updates, following the evaluation protocol of Section~\ref{subsec: evaluation}, for models trained with $K=100$, $500$, $1{,}000$, and $5{,}000$ tasks, alongside the analytical SARSA teacher for reference.

\begin{table}[H]
\centering
\caption{Mean discounted return on $100$ held-out MDPs after $100$ closed-loop in-context updates, for linear-, ReLU-, and softmax-attention transformers trained with $K$ tasks.}
\label{tab: nonlinear attention}
\begin{tabular}{lccccc}
\toprule
 & linear & \multicolumn{2}{c}{ReLU} & \multicolumn{2}{c}{softmax} \\
Training tasks $K$ & $1$ head & $1$ head & $2$ heads & $1$ head & $2$ heads \\
\midrule
$100$      & $0.290$ & $0.112$ & $0.055$ & $0.076$ & $0.104$ \\
$500$      & $0.387$ & $0.191$ & $0.200$ & $0.125$ & $0.114$ \\
$1{,}000$  & $0.370$ & $0.210$ & $0.273$ & $0.121$ & $0.139$ \\
$5{,}000$  & $0.404$ & $0.219$ & $0.402$ & $0.110$ & $0.136$ \\
\midrule
analytical SARSA teacher & \multicolumn{5}{c}{$0.403$} \\
\bottomrule
\end{tabular}
\end{table}

The linear-attention transformer closes the gap to the teacher as the training budget grows and matches it at $K=5{,}000$. ReLU attention shows exactly the head-count dependence predicted by the reduction $\mathrm{ReLU}(x)-\mathrm{ReLU}(-x)=x$: a single ReLU head, which cannot realize this identity, saturates well below the teacher, while with two heads
gradient training finds a ReLU implementation of the update and reaches
the teacher after a slower start. Softmax attention, in contrast, plateaus early with either one or two heads. This is consistent with the representational discussion above: a softmax head cannot represent the SARSA update exactly, and approximating it requires more attention heads. These results also show that supplying the teacher update as the training target does not by itself guarantee that training will recover it: whether imitation succeeds depends on the architecture's ability to represent the update.

\subsection{Out-of-distribution deployment: cliff walk}\label{app: cliffwalk}

If the transformer has learned the update rule itself, rather than a pattern tied to the training distribution, then the same frozen weights should work on structured MDPs it has never seen, provided the feature dimension matches. To test this, we deploy the \emph{same} trained SARSA transformer from Section~\ref{sec: experiments} (trained on random $9$-state MDPs, and never on any grid world) on a single fixed $3\times 3$ cliff-walk task.

\paragraph{Setup.} The four actions move up, down, left, or right; a move that would leave the grid leaves the state unchanged, and with slip probability $0.1$ the executed action is replaced by a uniformly random one. The start is $(1,1)$, the goal is $(3,3)$, and the cliff is the center cell $(2,2)$, so the agent must detour around it. Entering the cliff incurs a penalty of $-1$ and resets the agent to the start; reaching the goal pays $+1$ and also resets to the start; every other transition costs $-0.125$. As during training, the transformer never sees the grid coordinates: states and actions enter the prompt only through a random feature map $\boldsymbol{\phi}(s,a)\in\mathbb{R}^{36}$ with i.i.d.\ $\mathrm{Unif}(-1,1)$ entries. Starting from $\mathbf{w}^{(0)}=\mathbf{0}$, each in-context update follows the closed-loop protocol of Section~\ref{subsec: evaluation}: the current parameter $\mathbf{w}^{(t)}$ defines an $\epsilon$-greedy behavior policy ($\epsilon=0.1$) that generates the next $n=20$ environment steps, those steps form the prompt, and the transformer's output becomes $\mathbf{w}^{(t+1)}$. Test prompts are undiscounted ($\gamma=1$), and the analytical SARSA teacher (step size $\alpha=0.2$) is run in its own closed loop on the same random stream. After $100$ in-context updates we freeze $\mathbf{w}^{(100)}$ and run its greedy ($\epsilon=0$) policy from the start for $32$ episodes, recording the number of steps taken to first reach the goal. For reference, ``optimal'' is the expected hitting time of the best policy under the same slip dynamics (computed by value iteration) and ``random'' is the uniform policy.

\begin{table}[H]
\centering
\caption{Average number of steps to reach the goal on the $3\times 3$ cliff-walk task, for the final greedy policy after $100$ in-context updates.}
\label{tab: cliffwalk}
\begin{tabular}{lcccc}
\toprule
Model & transformer & analytical SARSA & optimal & random \\
\midrule
Steps to goal & $4.56$ & $4.56$ & $4.50$ & $52.84$ \\
\bottomrule
\end{tabular}
\end{table}

\paragraph{Results.} As reported in Table~\ref{tab: cliffwalk}, the frozen transformer reaches the goal in every episode and matches the algorithm it implements to two decimals, with the residual $0.06$ steps above the optimal $4.50$ due to slip. Its final greedy policy is a shortest path around the cliff, $(1,1)\to(2,1)\to(3,1)\to(3,2)\to(3,3)$. The pretrained model thus behaves like a learner rather than a memorized map: the same frozen weights, trained only on random MDPs, carry out in-context policy improvement on a structured task far outside the training distribution.

\subsection{Prompt ablation for actor-critic}\label{app: raw prompt}

The actor-critic prompt \eqref{eq: ac prompt} supplies the score-function terms $\gamma^i\mathbf{g}_{\boldsymbol{\lambda}}(s_i,a_i)$ pre-computed. To test a less engineered prompt, we repeat the main actor-critic experiment with an identical setup, except that the score-function row is replaced by the raw policy feature $\boldsymbol{\phi}_\pi(s_i,a_i)$ of the executed action; the training target remains the exact actor-critic update \eqref{eq: actor-critic batch}. Table~\ref{tab: raw prompt} reports the mean discounted return on $100$ held-out MDPs after $3{,}000$ closed-loop in-context updates, for models trained with $K$ tasks.

\begin{table}[H]
\centering
\caption{Mean discounted return on $100$ held-out MDPs after $3{,}000$ closed-loop in-context updates, for the actor-critic transformer trained with the original prompt \eqref{eq: ac prompt} and with the raw prompt in which the score-function row is replaced by $\boldsymbol{\phi}_\pi(s_i,a_i)$.}
\label{tab: raw prompt}
\begin{tabular}{lccc}
\toprule
Training tasks $K$ & original prompt & raw prompt ($\boldsymbol{\phi}_\pi$ row) & teacher \\
\midrule
$100$       & $0.07$ & $0.04$ & $0.39$ \\
$500$       & $0.22$ & $0.12$ & $0.39$ \\
$1{,}000$   & $0.32$ & $0.17$ & $0.39$ \\
$5{,}000$   & $0.39$ & $0.31$ & $0.39$ \\
$10{,}000$  & $0.39$ & $0.31$ & $0.39$ \\
$15{,}000$  & $0.39$ & $0.32$ & $0.39$ \\
\bottomrule
\end{tabular}
\end{table}

Training remains stable and the return improves monotonically with $K$, but while the original prompt recovers the teacher exactly by $K=5{,}000$, the raw-prompt model saturates at ${\approx}0.31$ and the gap persists through $K=15{,}000$. This is expected from the definition $\mathbf{g}_{\boldsymbol{\lambda}}(s,a)=\boldsymbol{\phi}_\pi(s,a)-\sum_b\boldsymbol{\phi}_\pi(s,b)\,\pi_{\boldsymbol{\lambda}}(b\mid s)$: the baseline term depends on the policy features of \emph{all} actions available at each step, while the raw prompt contains only the executed action's feature, so without additional embedding dimensions or layers the model cannot infer the score function needed for the policy gradient. Consistently, as shown in Figure~\ref{fig: raw prompt parameters}, the trained weights show that the model still implements the critic update exactly, while the actor readout is learned as a strongly shrunk surrogate, which keeps closed-loop deployment stable but limits policy improvement. 

\begin{figure}[H]
    \centering
    \includegraphics[width=0.9\linewidth]{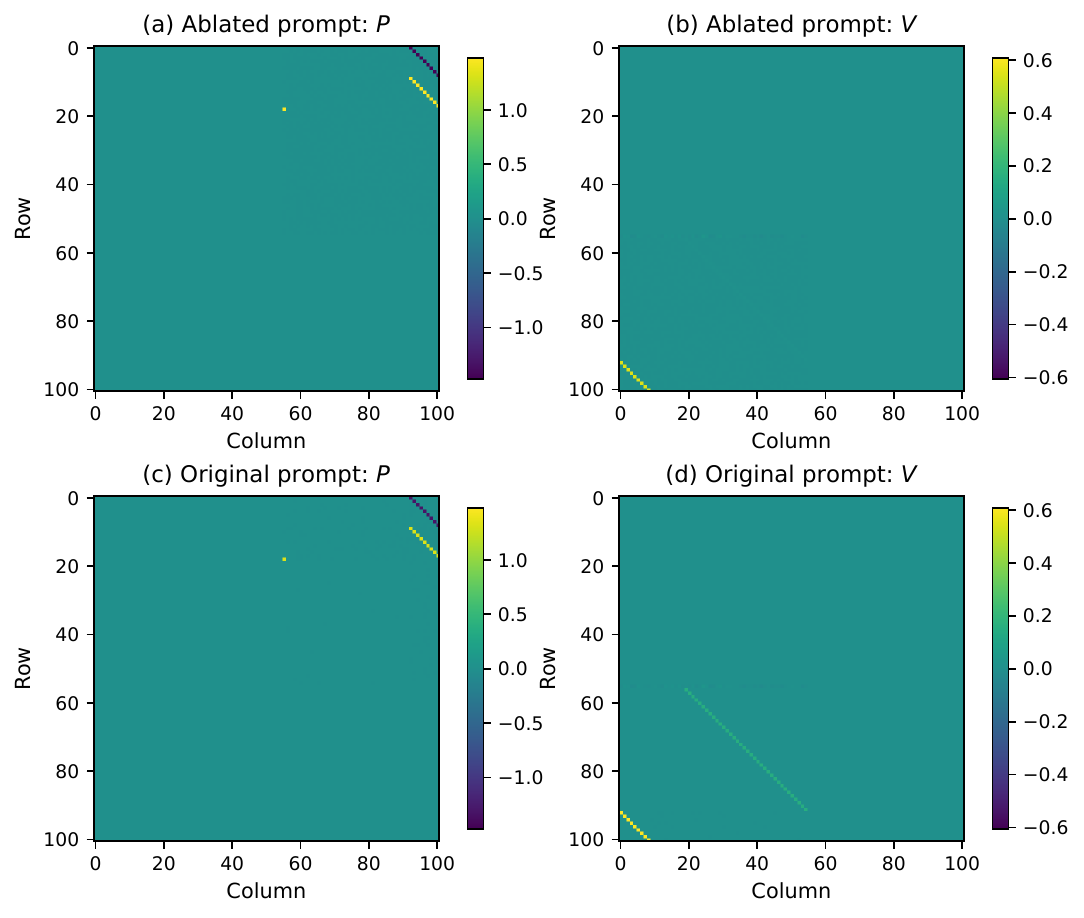}
    \caption{Final learned $\mathbf{P}$ (left) and $\mathbf{V}$ (right) after $K=15{,}000$ training tasks, using the ablated prompt with raw policy features (top) and the original actor-critic prompt (bottom).}
    \label{fig: raw prompt parameters}
\end{figure}

\subsection{Runtime}
The trainings and evaluations of the transformers used in our experiment were conducted on a MacBook Pro with the following specifications:

\begin{itemize}
    \item Apple M5 chip
    \item 16 GB Unified memory
\end{itemize}

The total training time is approximately 28 minutes, and the total evaluation time is approximately 5 minutes.


\end{document}